\documentclass[11pt,twoside]{article}

\usepackage[top=1in, bottom=1in, left=1in, right=1in]{geometry}
\usepackage{amsmath, amsthm, amssymb, graphicx, url, bm}

\RequirePackage{placeins}
\RequirePackage{amsmath}
\RequirePackage{amssymb}
\usepackage[round]{natbib}
\RequirePackage{graphicx}
\RequirePackage{url}
\PassOptionsToPackage{x11names}{xcolor}
\RequirePackage{xcolor}
\PassOptionsToPackage{algo2e,ruled}{algorithm2e}
\RequirePackage{algorithm2e}
\usepackage[utf8]{inputenc}
\usepackage{amsfonts}
\usepackage[normalem]{ulem}
\usepackage{xspace}
\usepackage{hyperref, prettyref}
\usepackage{xcolor}
\usepackage{enumitem}
\usepackage[at]{easylist}

\usepackage[T1]{fontenc}    

\usepackage{url}            
\usepackage{booktabs}       

\usepackage{nicefrac}       
\usepackage{microtype}      

\usepackage{algorithm, algorithmic}

\usepackage{mathtools}

\usepackage{thmtools, thm-restate}
\usepackage{enumitem}

\usepackage{graphicx,wrapfig}
\usepackage{natbib}

\usepackage[capitalize,noabbrev]{cleveref}
\RequirePackage{hyperref}
\RequirePackage{nameref}
\hypersetup{colorlinks,
            linkcolor=blue,
            citecolor=blue,
            urlcolor=magenta,
            linktocpage,
            plainpages=false}

\newcommand{\mathify}[1]{\ensuremath{#1}\xspace}
\newcommand{\fs}{\mathify{\mathcal{X}}} 

\newcommand{\R}{\mathify{\mathbb{R}}} 
\newcommand{\eps}{\mathify{\epsilon}}

\newcommand{\Var}{\operatorname{Var}} 
\newcommand{\Covar}{\operatorname{Cov}} 

\usepackage{amsfonts}
\usepackage{dsfont}
\usepackage{color}
\usepackage{thmtools}

\newcommand{\wh}{\widehat h}

\newcommand{\hs}{h^\star}

\newcommand{\bh}{\mathcal{E}}
\newcommand{\regret}{{{Reg}}}

\newcommand{\cA}{\mathcal{A}}
\newcommand{\cS}{\mathcal{S}}
\newcommand{\ind}{\mathds{1}}

\newcommand{\cB}{\mathcal{B}}
\newcommand{\cX}{\mathcal{X}}

\newcommand{\popl}{\mathcal{L}}

\newcommand{\hyps}{\mathcal{H}}
\newcommand{\cD}{\mathcal{D}}
\newcommand{\cC}{\mathcal{C}}
\newcommand{\cF}{\mathcal{F}}
\newcommand{\Rset}{\mathbb{R}}

\newcommand{\full}[1]{\ifnum \FULL=1{#1}\fi}
\newcommand{\short}[1]{\ifnum \FULL=0{#1}\fi}
\DeclareMathOperator*{\E}{\mathbb{E}}
\newcommand{\Esub}{{\E}}
\DeclareMathOperator*{\PP}{\mathbb{P}}
\newcommand{\PPsub}{{\PP}}

\newcommand{\cg}[1]{{\color{red}CG: #1}}

\def\marrow{\marginpar[\hfill$\longrightarrow$]{$\longleftarrow$}}
\def\cg#1{\textsc{\color{magenta} (Claudio  }\marrow\textsf{\color{magenta} #1})}

\newcommand{\err}{\text{err}}
\newcommand{\errvar}{\err_{\text{var}}}
\newcommand{\errcov}{\err_{\text{cov}}}
\newcommand{\errcovmessy}{\err'_{\text{cov}}}
\newcommand{\errcovone}{\err_{\text{cov,1}}}
\newcommand{\errcovtwo}{\err_{\text{cov,2}}}

\newcommand{\Plag}{P_1}
\newcommand{\Pzeros}{P_{00}}
\newcommand{\Pmixed}{P_{10}}
\newcommand{\Pones}{P_{11}}
\newcommand{\expressionI}{\textsc{(I)}}
\newcommand{\expressionII}{\textsc{(II)}}
\newcommand{\expressionIII}{\textsc{(III)}}

\newcommand{\vecs}{\vec{s}}

\newcommand{\vecyS}{\bm{Y}(S)}

\newcommand{\pseudodim}{\text{PDim}}
\newcommand{\complexitymeasurehere}{\min\{ \pseudodim(\hyps), k \}}
\newcommand{\bias}{\text{Bias}}

\newcommand{\mupper}{M_{\textsc{upper}}}
\newcommand{\muppers}{m_{\textsc{upper}}(S)}

\newcommand{\mlen}{\frac{np}2-2k+1}
\newcommand{\mlens}{\frac{np}2-2k+1}

\newtheorem{theorem}{Theorem}
\newtheorem{lemma}[theorem]{Lemma} 
 
\newtheorem{remark}[theorem]{Remark}

\makeatletter
\let\@fnsymbol\@arabic
\makeatother
\newcommand\blfootnote[1]{%
  \begingroup
  \renewcommand\thefootnote{}\footnote{#1}%
  \addtocounter{footnote}{-1}%
  \endgroup
}

\title{\bf{\LARGE{Statistical Learning from Attribution Sets}}}
\usepackage{times}
\author{Lorne Applebaum\footnotemark[1] \and Robert Busa-Fekete\footnotemark[1] \and August Y. Chen\footnotemark[2] \and Claudio Gentile\footnotemark[1] \and Tomer Koren\footnotemark[1]\, \footnotemark[3] \and Aryan Mokhtari\footnotemark[1]\, \footnotemark[4]}

\begin{document}

\maketitle
\blfootnote{Alphabetical ordering. Emails: \{lapplebaum@google.com, busarobi@google.com, ayc74@cornell.edu, cgentile@google.com, tkoren@google.com, amokhtari@google.com\} }
\footnotetext[1]{Google Research}
\footnotetext[2]{Cornell University, Ithaca, USA}
\footnotetext[3]{Tel Aviv University, Tel Aviv, Israel}
\footnotetext[4]{UT Austin, Austin, USA}

\begin{abstract}%
We address the problem of training conversion prediction models in advertising domains under privacy constraints, where direct links between ad clicks and conversions are unavailable. Motivated by privacy-preserving browser APIs and the deprecation of third-party cookies, we study a setting where the learner observes a sequence of clicks and a sequence of conversions, but can only link a conversion to a set of candidate clicks (an attribution set) rather than a unique source. We formalize this as learning from attribution sets generated by an oblivious adversary equipped with a prior distribution over the candidates. 
Despite the lack of explicit labels, we construct an unbiased estimator of the population loss from these coarse signals via a novel approach.
Leveraging this estimator, we show that Empirical Risk Minimization achieves generalization guarantees that scale with the informativeness of the prior and is also robust against estimation errors in the prior, despite complex dependencies among attribution sets.
Simple empirical evaluations on standard datasets suggest our unbiased approach significantly outperforms common industry heuristics, particularly in regimes where attribution sets are large or overlapping.
\end{abstract}

\smallskip

\section{Introduction}
Web advertising---one of the largest real-world applications of machine learning---has undergone a significant shift in recent years. To power automated bidding, advertisers (or their AdTech partners) train models to predict the probability of a {\em conversion} (e.g., a product purchase, a sign-up, an app installation, etc.) following an ad click. These predictions are essential for calculating bid prices in real-time online auctions running at the publisher side. Because the initial click occurs on a publisher's site while the conversion happens on the advertiser’s site, generating training labels requires linking these two distinct events. This process essentially involves tracking user behavior across different web domains~(see, e.g., \citealp{w19}).

While third-party cookies and link decoration have historically made tracking straightforward, a shift toward user privacy has transformed the landscape. Acknowledging the conflict between essential web advertising and the demand for privacy, major browsers have introduced specialized APIs to measure performance without compromising user anonymity. This transition is highlighted by the deprecation of third-party cookies in browsers like Apple's Safari \citep{w19} and Mozilla's Firefox~\citep{cc22}. 
These APIs restrict AdTechs to collecting cross-site data exclusively in some obfuscated form. This creates a challenge for publishers who require precise per-interaction predictions to run effective auctions. Under these privacy constraints, the publisher can see the list of individual ad interactions (clicks) but only receives approximate information about the resulting conversions from the advertiser’s side, rather than direct links between specific clicks and sales.
In particular, the publisher learns some coarse information about the conversion and click, such as the ad campaign they belonged to and the approximate time of the conversion.
Based on this coarse information, we can usually identify a collection of clicks that could have produced the conversion (i.e., clicks from the same ad campaign in a reasonable time interval given the conversion time), but it is not possible to determine exactly which click was responsible.
Our goal is to learn conversion prediction models from these weak conversion signals.

\vspace{-0.1in}
\subsection{Our contributions}
\vspace{-0.05in}
We formalize our problem as a novel setting of statistical learning from attribution sets: collections of clicks generated by an oblivious adversary according to a known prior (Section \ref{sec:setup}). Because the direct association between clicks and conversions is unobserved, we seek to learn this relationship using only these coarse signals. 

We provide three main theoretical contributions. 
Surprisingly, we first show that it is possible to construct an unbiased estimator of the population loss by decomposing the expected loss into moments that can be estimated from the attribution sets (Theorem \ref{prop:unbiasedestimatoroneji}). The core innovation lies in decoupling features from labels by conditioning on the adversary’s actions, allowing us to leverage a combinatorial argument to map inaccessible population moments to observable indicators.
Second, by minimizing our unbiased surrogate, we establish that Empirical Risk Minimization (ERM) attains strong generalization guarantees despite the statistical dependencies induced by the attribution process. 
Specifically, Theorem \ref{thm:allattributionsetssamplecomplexity} demonstrates that the sample complexity of our method scales with the standard capacity of the hypothesis class, inflated by a factor of $1/\|\pi\|_2^2$, where $\pi$ is the adversary’s prior distribution governing the possible locations of the true conversion within an attribution set. 
This is indeed expected as $\|\pi\|_2^2$ serves as a fundamental measure of the statistical difficulty of the task: more concentrated priors heighten the signal-to-noise ratio, yielding more favorable convergence rates.
A limitation of Theorem \ref{thm:allattributionsetssamplecomplexity} is that the learner requires exact knowledge of $\pi$. 
However, as our third theoretical contribution, we show in Theorem \ref{thm:robustallattrsets} that even if we have an estimate $\widehat\pi$ of $\pi$, our method is robust to the estimation error.

Finally, to verify our theoretical guarantees, we conduct preliminary experiments on standard datasets, like MNIST, CIFAR-10, and Higgs, showing that our unbiased approach substantially outperforms common industry heuristics--such as random or maximum-prior attribution--particularly when attribution sets are large or overlapping (Section \ref{sa:mainbody_experiments}).

\subsection{Related literature}
Conversion Rate (CVR) prediction remains a foundational challenge in online advertising, generating a vast body of literature. Central to this field is the {\em attribution problem}—the assignment of credit to specific user interactions for subsequent conversions. Established attribution heuristics (such as ``last touch,'' ``first touch,'' or ``linear attribution'') dictate the mechanisms for label generation and training, which in turn drive automated bidding and traffic allocation strategies. Relevant works include  \citep{borgs2007dynamics,cai2017real,zhu2017optimized,jin2018real,wang2017display,yang2019bid,singal2019shapley,liu2021neural,chen2022asymptoticallyunbiasedestimationdelayed,fan2025two,chen2025singleviewmultiattributionlearning}. While the above list is very far from doing justice, it is fair to say that many of these investigations are mostly experimental in nature.

More theoretically oriented is the related bulk of research on (stochastic) online/bandit algorithms with delayed feedback, with early investigations including \citep{pmlr-v28-joulani13,vernade2017stochasticbanditmodelsdelayed,lbsdf20}. These works predominantly address streaming data problems (online prediction), where models are continuously fine-tuned as feedback arrives. Crucially, these frameworks generally assume no privacy-induced label obfuscation; they postulate that a click will eventually yield an observable signal unless it is censored by ``freshness'' constraints. Typically, this involves setting an observation window $w_0$ (e.g., 48 hours) where a click at time $t_0$ is frozen until $t_0+w_0$. If a conversion occurs within this window, the click is labeled positive; otherwise, it is treated as a negative sample. Consequently, the primary technical challenge in these streaming settings is optimizing the trade-off between the cost of adaptivity (where larger $w_0$ delays updates) and the bias introduced by censoring (where smaller $w_0$ mislabels valid but delayed conversions). Bias correction mechanisms are often based on importance sampling (see, e.g., \cite{chen2022asymptoticallyunbiasedestimationdelayed}, and references therein).

In contrast, our work addresses an emerging landscape defined by privacy preservation. We operate in a setting where, even if deterministically linking a click to a conversion is technically feasible, the association is deliberately obfuscated by anti-tracking APIs mediating between publisher and advertiser data. Furthermore, we depart from the streaming paradigm to focus on a (more practical) {\em batch} learning setting. 
Since the complete, albeit obfuscated, dataset is available at the outset, concerns regarding data freshness and update latency are not directly relevant to our approach. 

Our work is also related to weak supervision paradigms, specifically Multiple Instance Learning (MIL) (e.g., \cite{maron1997framework,dietterich1997solving,ilse2018attention,tian2021weakly,lv2023unbiased,javed2022additive,jk24}) and Learning from Label Proportions (LLP). Early references on LLP include \citet{quadrianto2008estimating, patrini2014almost}, more recent ones are \citet{saket2021learnability,saket2022algorithms,scott2020learning,zhang2022learning,easyllp,brahmbhatt2023pac,l+24,b+25,applebaum2026optimallearninglabelproportions}. In these paradigms, attribution sets are referred to as {\em bags}. 
Both frameworks focus on learnability at the bag and instance levels. In MIL, a bag is labeled positive if it contains at least one positive instance and negative otherwise. In LLP, the learner observes the proportion of positive labels within each bag.
Crucially, the observability structure in these settings is significantly more informative than ours. In both MIL and LLP, every bag conveys a signal to the learner; MIL explicitly includes negative bags (containing zero positive labels), whereas our setting typically only generates attribution sets for positive outcomes (conversions). Furthermore, general statistical analyses of LLP (e.g., \cite{easyllp,l+24,b+25,applebaum2026optimallearninglabelproportions}) predominantly assume non-overlapping bags---an assumption that need not hold in the context of API-mediated attribution, where user interaction windows are often wide and overlapping.

\section{Preliminaries and Notation}\label{sec:setup}
We move from an idealized physical process to a distilled model that removes temporal dependencies.

\subsection{The click-conversion process}
\label{sec:process}

Consider a stylized advertising setup, illustrated in Figure \ref{f:2} (Left). This involves two parties: a {\em publisher}, who observes a stream of click events (interactions with a website by users), and an {\em advertiser}, who observes a stream of conversion events (e.g., purchases).
In order to decide which slot will be assigned to the competing advertisers, the publisher typically runs an auction, which is powered by a {\em conversion prediction} model. This is a model that takes as input click features and returns the estimated probability that the click will (eventually) lead to a conversion.

Due to privacy constraints (such as those enforced by anti-tracking APIs) and random time delays between clicks and conversions, the publisher cannot deterministically link a specific conversion at time $T_Y$ to its originating click at time $T_X$. Instead, for every observed conversion, the system provides an {\em attribution set}: a window of candidate clicks that {\em could} have caused the conversion. 

We are facing here a classical {\em attribution problem} for CVR (aka conversion rate) prediction. Yet, unlike the voluminous literature on the subject (e.g., \cite{borgs2007dynamics,cai2017real,zhu2017optimized,jin2018real,wang2017display,yang2019bid,liu2021neural,chen2022asymptoticallyunbiasedestimationdelayed,fan2025two,chen2025singleviewmultiattributionlearning}, and references therein), we are dealing with the more practical scenario of {\em batch} attribution, whereby an offline dataset of clicks and conversions has been recorded and made available to the publisher. The prediction model is built via this set of observations.

To analyze this setting rigorously, we start by viewing this ecosystem as a {\em click-conversion process} described by two pairs of random variables $\big\langle (X, D_X), (Y, D_Y) \big\rangle$, where:
$X \in \cX$ describes the click event features; $Y \in \{0,1\}$ is the corresponding binary label (conversion yes/no); $D_X$ and $D_Y$ are the {\em event delay} variables (time until the next event). In particular, $D_X$ is the time until the next click and $D_Y$ is the time until the next label.

Data is generated by fixing a total number of events $n$ and drawing $n$ i.i.d. pairs $(X_i, D_{X,i})$ and $(Y_i, D_{Y,i})$ from the joint distributions. The observed timestamps are cumulative sums of these delays: $T_{X,i} = \sum_{j=1}^i D_{X,j}$ and $T_{Y,i} = \sum_{j=1}^i D_{Y,j}$.
The main assumption we make is that the timing variables ($D_X, D_Y$) are independent of the event variables ($X, Y$). This independence allows us to separate the temporal dynamics from the feature-label relationship.
Our ultimate goal is to train a conversion prediction model using only these coarse, aggregate signals (the attribution sets), without ever observing direct links between individual clicks ($X$) and labels ($Y$).

\subsection{Mathematical formalization}
\label{sec:formalization}
We can now distill the process above into a learning framework on sequences, as depicted in Figure~\ref{f:2} (Right).
There exists a hidden (possibly randomized) bijection $b: [n] \to [n]$ that links the click $(X_i, T_{X,i})$ to its corresponding binary label and event delay $(Y_{b(i)}, T_{Y,b(i)})$.
This bijection defines a hidden dataset $S = \big\langle (X_1, Y_{b(1)}), \ldots, (X_n, Y_{b(n)}) \big\rangle$. Based on our independence assumption, the pairs in $S$ are i.i.d. draws from a distribution $\cD$ over $\cX \times \{0,1\}$, unknown to the learner.
For the remainder of the paper, we simplify notation by re-indexing such that $b$ is the identity, denoting the $i$-th pair simply as $(X_i, Y_i)$.\\[-3mm]

\noindent\textbf{The observation model.} The learner does not see the labels $Y$. Instead, the learner observes:
\begin{itemize}
\item A sequence of feature vectors $X_1, \ldots, X_n$ drawn i.i.d. from the marginal distribution over $\cX$;
\item A collection of {\em Attribution Sets} $\cA = \{A_1, \ldots, A_M\}$, each attribution set representing the candidate clicks that could have caused the conversion. While the sequence length $n$ is fixed, the number of observed attribution sets $M = \sum_{i=1}^n Y_i$ is a random variable equal to the number of positive labels/conversions.
\end{itemize}

\noindent\textbf{The adversary (attribution mechanism).} We model the generation of these sets (or windows of candidate clicks) via an oblivious adversary. Let $i_j(S) \in [n]$ be the index of the $j$-th positive label ($Y_{i_j}=1$) in $S$. For each conversion index $i_j(S)$, the adversary generates an attribution set $A_j \subseteq [n]$ consisting of $k$ {\em consecutive}\footnote{The consecutiveness is not a strict requirement here and is only assumed to simplify the subsequent notation.} indices that includes $i_j(S)$.
Crucially, the position of the true conversion within the set is governed by a {\em prior distribution} $\pi$ over $[k]$. Specifically, the adversary constructs the window $A_j$ such that the true index $i_j(S)$ appears at the $r$-th position of $A_j$ with probability $\pi[r]$:\footnote
{
If there are boundary effects, e.g., if $i_j(S)=1$, the adversary constructs $A_j$ s.t. $\mathbb{P}(A_j[r] = X_{i_j(S)}) \propto \pi[r]$ for valid $r$. This will anyhow not matter, as our algorithm will only consider $j\,:\,$  $k \le j \le M-k$, with no boundary effects.
}
\begin{equation*}
\mathbb{P}(A_j[r] = X_{i_j(S)}) = \pi[r] \quad \text{for } r \in \{1, \dots, k\}.
\end{equation*}
This prior $\pi$ captures domain knowledge, such as the ``last-touch'' heuristic (where $\pi[k]$ is large, encoding the belief that more frequently the last element is the cause of the conversion) 
or time-decay models. The generation of attribution sets is independent of the feature values $X$ (obliviousness), and the sets may overlap.
To streamline notation, we assume a constant set size $k$ and a fixed prior $\pi$ known to the learner. As we discuss later, our results extend to estimated priors. And they also extend to variable set sizes and variable priors, as briefly discussed in Remark \ref{rem:extension} (Appendix \ref{subsec:robustnessproofs}).\\[-2mm]

\noindent\textbf{Learning goal.}
Recall we want to learn a model to predict $Y$ from $X$ through such data. When learning a model $p:\cX \rightarrow [0,1]$, we operate within a defined hypothesis space $\hyps$. Each hypothesis $h \in \hyps$ represents a deterministic mapping from the input space $\cX$ to $[0, 1]$, where the output $h(x)$ estimates the probability that $Y = 1$ given $X = x$.
We measure the discrepancy between predictions and labels via a loss function $\ell~\colon [0,1]~\times~\{ 0,1\}~\to~\Rset^+$. We assume for simplicity that the loss is {\em bounded} (e.g., the square loss $\ell(h(x),y) = (y-h(x))^2$).

Given distribution $\cD$, hypothesis class $\hyps$, and loss function $\ell$, the {\em population loss} (or {\em statistical risk}) of a hypothesis $h \in \hyps$ is defined as:
$\popl(h) = {\E}_{(x,y) \sim \cD}[\ell(h(x), y)]$. 
We aim to minimize the {\em excess risk}, also referred to as the {\em regret}, $\regret(h)$, which quantifies the performance gap between $h$ and the {\em best-in-class} hypothesis $h^*_{\hyps}$:
\(
\regret(h) = \popl(h) - \popl(h^*_{\hyps}),
\)
where $h^*_{\hyps} = \arg\min_{h \in \hyps} \popl(h)$. We work in the general non-realizable (agnostic) setting where the true optimal mapping may not lie within $\hyps$ (i.e., $h^* \notin \hyps$).

We denote by $\mu$ the joint distribution over the two sources of randomness in our setting: (i) The generation of the dataset $S$ (drawn from $\cD^n$);
(ii) The adversary's generation of attribution sets $A_1, \ldots, A_M$ (drawn i.i.d. via $\pi$, conditioned on $S$).
Our goal is to find an estimator $\wh \in \hyps$ such that the population loss $\popl(\wh)$ is minimized with high probability over $\mu$. Specifically, we aim to design algorithms for $\wh$ and quantify $\regret(\wh)$ in terms of prior $\pi$ (which encodes the degree of label obfuscation), the sample size $n$, and the general properties of the loss function $\ell$ and hypothesis space $\hyps$.\\[-4mm]

\noindent\textbf{Further notation.}
Let $A_j[i] \in \cX$ be the $i$-th feature vector in $A_j$. From here on out, when expectations are not explicitly specified, they are w.r.t. $(X,Y) \sim \cD$. We let $\cD_X, \cD_Y$ denote the $X$ and $Y$-marginals of $\cD$ respectively, and $\cD_{X|Y=1}$, $\cD_{X|Y=0}$ denote conditional laws of $X$ given $Y=1$ and $Y=0$ respectively. We let $\E_1[\cdot]$ denote the conditional expectation $\E[\,\cdot\, |\, Y=1]$ and $\E_0[\cdot]$ denote $\E[\,\cdot\, |\, Y=0]$, and define $\PP_1$ and $\PP_0$ analogously. Finally, recall for $n$ i.i.d. Rademacher variables $\sigma_i \in \{-1,+1\}$, the quantity 
\[
R_n(\hyps) = \frac{1}{n}\,{\E}_{X_1,\ldots,X_n} \Big[{\E}_{\sigma_1,\ldots \sigma_n} \Big[\sup_{h \in \hyps}\,\Big|\sum_{i=1}^n \sigma_i h(X_i) \Big| \,\,|\,X_1,\ldots,X_n \Big] \Big] 
\]
is the (average) Rademacher Complexity of function class $\hyps$.

\begin{figure*}
    \vspace{-0.3in}
    \hspace{-0.2in}
    \includegraphics[width=0.62\textwidth]{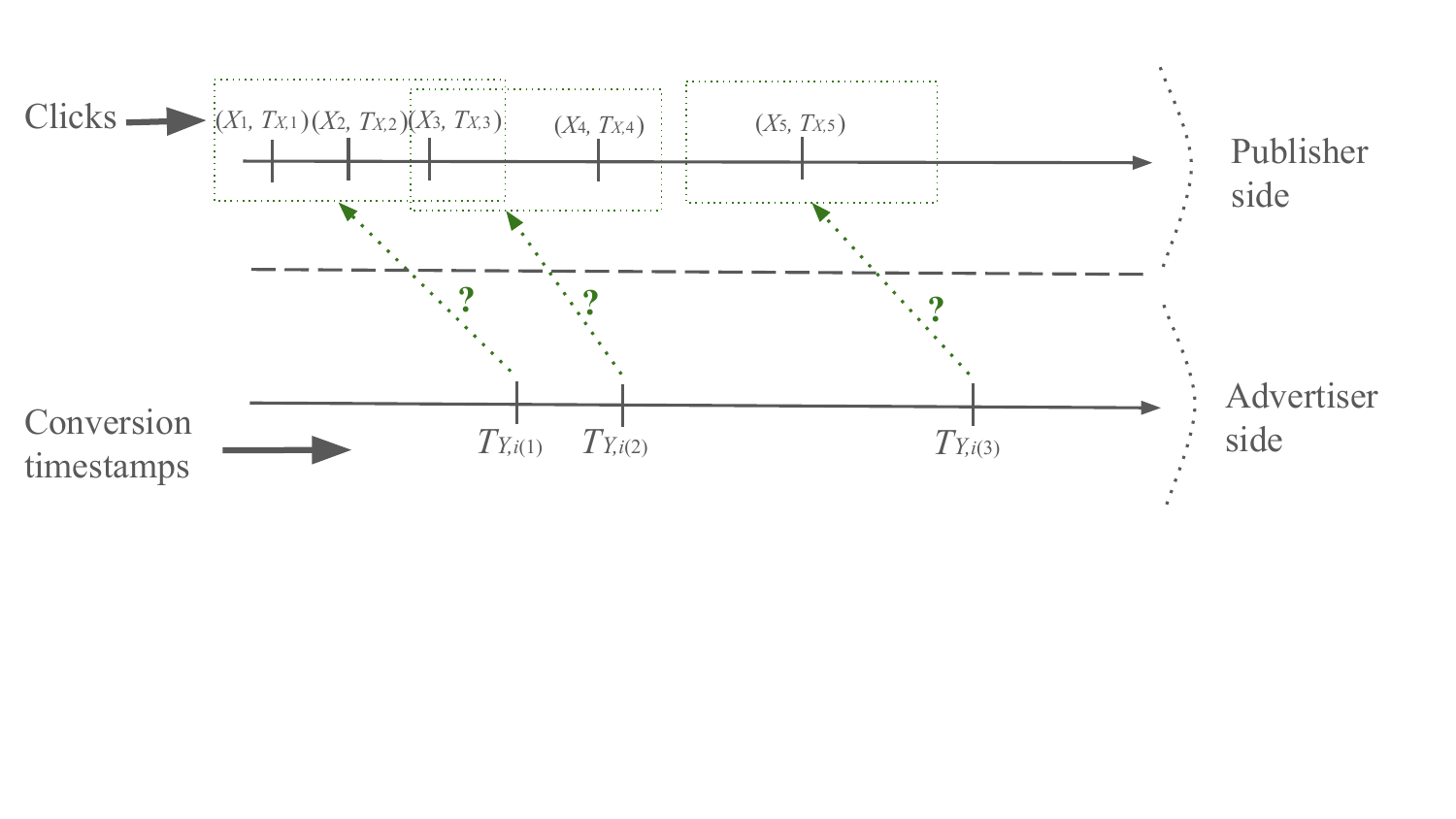}
    \hspace{-0.2in}
    \includegraphics[width=0.62\textwidth]{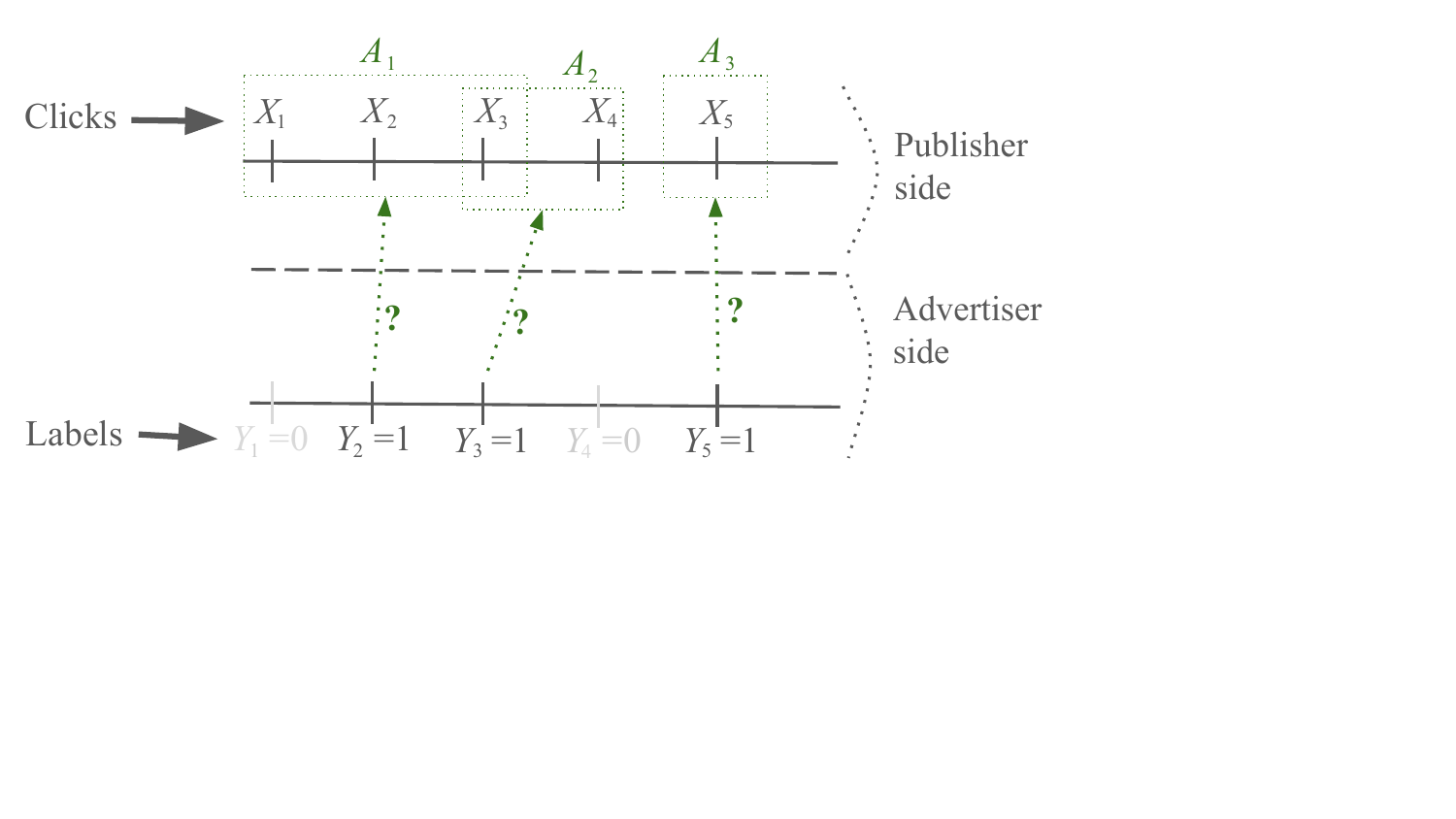}
    \vspace{-1.1in}
    \caption{
    {\bf Left:} The physical process. Publisher clicks $(X_i, T_{X,i})$ generate advertiser conversion timestamps $T_{Y, i(j)}$ with unknown delays. Attribution sets capture this uncertainty; for example, the conversion at $T_{Y,i(2)}$ is attributed to $\{X_3, X_4\}$. While $X_4$ is the more likely cause due to temporal proximity, the sets reflect all candidates defined by the window. Note that the attribution sets may overlap.
    {\bf Right:} The simplified sequence model used for analysis. Random variables $X_1,\ldots, X_5$ with positive labels at indices $2, 3, 5$ generate the observed attribution sets $A_1 = \{1,2,3\}$, $A_2 = \{3, 4\}$, and $A_3 =\{5\}$.
    \label{f:2}}
    \vspace{-0.15in}
\end{figure*}

\section{An Unbiased Estimator for $\popl(h)$}\label{s:unbiased}
Consider a hypothesis $h: \fs \to [0,1]$ and a loss function $\ell(h(x), y)$. 
In stark contrast to standard statistical learning settings, we have only weak partial label information in the form of the attribution sets.
Lacking any explicit labels, constructing an unbiased estimator for $\popl(h)$ is far from obvious. 

To construct an unbiased estimator, our first step is the following decomposition. As the labels $y$ are binary ($y \in \{0,1\}$), we can decompose the loss into a base term and a label-dependent term:
\vspace{-0.15cm}
\begin{align}\label{eq:key_decomposition}
\ell(h(x),y) = \overbrace{\ell(h(x),0)}^{f_1(h(x))} + y\overbrace{\Bigl(\ell(h(x),1)-\ell(h(x),0)\Bigl)}^{f_2(h(x))}\,.
\end{align}
Consequently, for suitable functions $f_1, f_2\,:\,[0,1] \rightarrow \R$, any binary loss can be expressed in the affine form:
\(
\ell(h(x),y) = f_1(h(x)) + y f_2(h(x)).
\)
For instance, the square loss is obtained by $f_1(h) = h^2$, and $f_2(h) = 1-2h$, where in both cases $h \in [0,1]$. 
Thus, estimating the population risk reduces to estimating ${\E}_{(X,Y) \sim \cD}\big[ f_1(h(X)) \big]$ and ${\E}_{(X,Y) \sim \cD}\big[ Y f_2(h(X)) \big]$.

The fundamental challenge in our setting is the latent nature of the labels $Y$, which precludes standard techniques to estimate these expectations such as importance sampling--the joint density is never observed directly.
Remarkably, we show that the combinatorial structure of the attribution sets--governed by the known prior $\pi$--renders the population loss identifiable. Specifically, we leverage a combinatorial argument to derive an exact mapping between the inaccessible population moment ${\E}_{(X,Y) \sim \cD}\big[Y f_2(h(X)) \big]$ and an expectation over the observable attribution signals: ${\E}_{\mu}\big[f_2(h(A_j[i]))\, \ind\{j \leq M-k\}\big]$. This leads to our first main result:

\begin{theorem}\label{prop:unbiasedestimatoroneji}
Let 
\(
\ell(h(x),y) = f_1(h(x)) + y f_2(h(x))
\)
be an arbitrary loss function for binary labels $y 
\in \{0,1\}$, and $\cD$ be a distribution over $\cX\times \{0,1\}$ such that $p = \PP(Y=1) \in (0,1)$. Let $M = \sum_{i=1}^n Y_i$ be a random variable denoting the number of conversions (1s) among the labels in the stream $S$.
Consider any $j$ with $k \le j \le n$,  any $i$ with $1 \le i \le k$, and any $h\,:\cX \rightarrow [0,1]$.
Let
\begin{align}
\widehat \ell(h,j,i) &= \frac{f_2(h(A_j[i]))}{\beta_1(j, i)} + \frac{\E\bigl[f_1 ( h(X) )\bigl]}{B_{n, p, j+k}}  
- \frac{\beta_0(j, i)\, \E\bigl[f_2 ( h(X) )\bigl]}{\beta_1(j, i)\, B_{n, p, j+k}}\,,\label{eq:singlejiestimator}
\end{align}
where $B_{n, p, k'} := \sum_{i'=k'}^n \binom{n}{i'} p^{i'}(1-p)^{n-i'}$ is the Binomial tail, and where
\begin{align*}
\beta_1(j,i) &:= \frac{\pi[i]\, B_{n, p, j+k}}{p} + \Big( B_{n-1, p, j+k-1} - \frac{1 - p\, B_{n-1, p, j+k-1}}{1-p} \Big) \big(1-\pi[i] \big)\,, \\
\beta_0(j,i) &:= \frac{1 - p\, B_{n-1, p, j+k-1}}{1-p} \big(1-\pi[i]\big)\,.
\end{align*} 
Then, we have 
\begin{align*}
{\E}_{\mu}\big[ \widehat \ell(h, j, i)\, \ind\{ j \le M-k\} \big] = \popl(h)\,.
\end{align*}
\end{theorem}
A proof sketch follows; a full proof is in Appendix \ref{sa:proofs1}. Next in Section \ref{sec:refined}, we leverage these unbiased estimators from Theorem \ref{prop:unbiasedestimatoroneji} across different $j$ and $i$ to build an unbiased estimator of the population loss, and study its statistical properties under ERM. \\ [-2mm]

\begin{proof}{\bfseries (of Theorem~\ref{prop:unbiasedestimatoroneji}; sketch)}
Consider the decomposition (\ref{eq:key_decomposition}). To obtain an unbiased estimator of $\popl(h)$ having access to features without explicit labels, we need to leverage the coarse attribution set signals to estimate the label-conditional moment $\E\big[ Y f_2(h(X)) \big]$. $\E\big[ f_1(h(X)) \big]$ can be estimated just from features. Surprisingly, we show that $\E\big[ Y f_2(h(X)) \big]$ can be cast in terms of $A_j[i]$, the $i$-th element of the $j$-th attribution set, which we do have access to:
\begin{equation} \label{eq:pf_sketch_step1}
\E\bigl[ Y f_2(h(X)) \bigr] = \frac1{\beta_1(j, i)}{\E}_{\mu}\Bigl[f_2(h(A_j[i])) \cdot \ind\{j \le M-k\} \Bigl] - \frac{\beta_0(j, i)}{\beta_1(j, i)} \E\bigl[ f_2(h(X)) \bigr]\,.
\end{equation}
Note the $B_{n,p,j+k}$ terms in Theorem \ref{prop:unbiasedestimatoroneji} arise naturally as $\E_{\mu}\big[ \ind\{j \le M-k\} \big]=B_{n,p,j+k}$. Combining (\ref{eq:pf_sketch_step1}) with the decomposition (\ref{eq:key_decomposition}) now proves Theorem \ref{prop:unbiasedestimatoroneji}. 

We will now explain the steps to establish the result in (\ref{eq:pf_sketch_step1}). Let $\eta(x) = \PP(Y=1|X=x)$. Noting $\E\big[Y f_2(h(X)) \big] = \E\big[ f_2(h(X)) \E[Y|X] \big] = \int \PP(X=x) f_2(h(x)) \eta(x) dx$, (\ref{eq:pf_sketch_step1}) follows from integrating the following result:
\begin{equation} \label{eq:pf_sketch_step2}
{\E}_{\mu}\Bigl[\ind \{ A_j[i]=x\} \cdot \ind\{ j \le M-k \} \Bigl] = \PP(X=x) \Big( \eta(x)\, \beta_1(j,i) + \beta_0(j,i) \Big)\,.
\end{equation}
Our aim is now to prove (\ref{eq:pf_sketch_step2}). To do so, we first simplify the left hand side of (\ref{eq:pf_sketch_step2}). 
By definition of the adversary's action, $A_j[i]$, the $i$-th element of $A_j$, is $X_{i_j(S) + i - r}$ with probability $\pi[r]$. Thus,
\vspace{-0.1cm}
\begin{equation}\label{eq:pf_sketch_summation}
{\E}_\mu\Big[ \ind\{ A_j[i]=x \} \cdot \ind\{j \le M-k\} \Big] 
= 
\sum_{r=1}^k \pi[r] \E_{S \sim \cD^n} \Big[ \ind\{j \le M-k\} \cdot \ind\big\{ X_{i_j(S) + i - r}=x\big\} \Big]\,.
\vspace{-0.1cm}
\end{equation}
The summation splits into two cases: $r=i$ and $r \neq i$. We now characterize $\E_{S \sim \cD^n} \big[ \ind\{ X_{i_j(S) + i - r}=x\} \cdot \ind\{j \le M-k\} \big]$ for $r \neq i$, the argument for $r=i$ is similar (and omitted).

A-priori, this is challenging, as the two events in the expectation are tightly coupled. We simplify this expectation with the following key observation: conditioned on any realization of the labels---in particular the event $\{j \le M-k\}$---the law of $X_{i_j(S)+i-r}$ can be readily understood. 
Specifically, conditioned on $\{Y_{i_j(S)+i-r}=1\,,\, j \le M-k\}$ we have $X_{i_j(S)+i-r} \sim \cD_{X|Y=1}$, and conditioned on $\{ Y_{i_j(S)+i-r}=0\,,\, j \le M-k\}$ we have $X_{i_j(S)+i-r} \sim \cD_{X|Y=0}$: these results are in Lemma \ref{lem:conditionallawoneidx}. 

Now, the only remaining piece to compute the expectation of $\ind\{X_{i_j(S) + i - r}=x\} \cdot \ind\{j \le M-k\}$ is computing the probabilities of the events $\{Y_{i_j(S)+i-r}=1\,,\, j \le M-k\}$ and $\{ Y_{i_j(S)+i-r}=0\,,\, j \le M-k\}$. 
Interestingly, by leveraging a combinatorial argument, we can show that the probability of these two events can be simplified to $p B_{n-1, p, j+k-1}$ and $1 - p B_{n-1, p, j+k-1}$, respectively; this is proven in Lemma \ref{lem:problabel1beforeij}. 
Given these simplifications, and by leveraging the Bayes' rule, we can show that $\E_{S \sim \cD^n} \big[ \ind\{ X_{i_j(S) + i - r}=x\} \cdot \ind\{j \le M-k\} \big]$ equals the expression
\begin{align}
\PP(X=x)\, \eta(x)\, B_{n-1, p, j+k-1} + \frac{\PP(X=x)\,\bigl(1 - \eta(x) \bigr)}{1-p} \cdot \bigl(1 - p\, B_{n-1, p, j+k-1} \bigl)\,.\notag
\vspace{-0.05in}
\end{align}
Given this expression and the definitions of $\beta_1(j, i), \beta_0(j, i)$, the result in (\ref{eq:pf_sketch_step2}) follows. 
\end{proof}

\section{From an Unbiased Estimator to an ERM Algorithm}\label{sec:refined}
We now leverage the unbiased estimator from Theorem \ref{prop:unbiasedestimatoroneji} to create a sample-efficient unbiased estimator that uses a sizeable fraction of the data. Specifically, Theorem \ref{prop:unbiasedestimatoroneji} implies that for any hypothesis $h$, the quantity $\widehat \ell(h, j, i)\, \ind\{j \le M-k\}$ is an unbiased estimator of $\popl(h)$, provided we have exact knowledge of the conversion rate $p$ (involved in the expression for $\beta_0(j,i)$ and $\beta_1(j,i)$) and the two expectations $\E[f_1 ( h(X) )]$ and $\E[f_2 ( h(X) )]$. 

However, the cost for removing this prior knowledge would in fact be minor (and leading to a negligibly biased estimator). This is because these three quantities can be straightforwardly estimated at a higher resolution than the one allowed by the signals we receive from the adversary. 
To see this, note we can always split $\{X_1,\ldots,X_n\}$ into two equal-size subsets $\{X_1,\ldots,X_{n/2}\}$ and $\{X_{n/2+1},\ldots,X_{n}\}$, estimate $\E[f_1 ( h(X) ) ]$ and $\E[f_2 ( h(X) )]$ via $\{X_1,\ldots,X_{n/2}\}$, estimate $p$ as the fraction of conversions up to time $n/2$, and then build the estimator $\widehat \ell(h,j,i)$ on variables in the second half $\{X_{n/2+1},\ldots,X_{n}\}$, where true expectations are replaced by the estimates constructed on the first half.
Now, any uniform guarantee over $h \in \hyps$ in estimating $\E[f_1 ( h(X) )]$ and $\E[f_2 ( h(X) )]$ will be at a rate $1/\sqrt{n}$, and similarly for $p$.
On the other hand, as we shall see below in Theorem \ref{thm:allattributionsetssamplecomplexity}, the amount of information the adversary releases to the learner can only afford rates at best $1/\sqrt{n}$. 

Consequently, with little loss of generality, we assume that $p, \E[f_1 (h(X)) ], \E[ f_2 (h(X)) ]$ are known to the learner. Furthermore, for simplicity, we work with a {\em bounded} loss function: for all $h \in \hyps$, and all $(x,y) \in \cX\times \{0,1\}$ we have
$\ell(h(x),y) = f_1(h(x)) + y f_2(h(x))$, with\footnote
{
The boundedness involving $f_1(\cdot)$ will not play any role here. Since we assumed prior knowledge of $\E[f_1(h(X)]$ in Section \ref{s:unbiased}, it suffices to have $\E[f_1(h(X))] < \infty$.
} 
$|f_1(h(x))| \leq F_1$, and $|f_2(h(x))| \leq F_2$, for some $F_1, F_2 > 0$.

\subsection{The ERM Algorithm}
We now describe the ERM algorithm. Write $S=\big\langle(X_1,Y_1), \ldots, (X_n,Y_n)\big\rangle$, and denote the family of attribution sets by $\cA = \{A_1,\ldots, A_M\}$, with $M = \sum_{i=1}^n Y_i$. Define $\mupper := \min\{ \frac{np}2 - k, M-k \}$. Let $\Sigma = \sum_{i=1}^k \pi[i]^2 = \|\pi\|_2^2$; the quantity $\frac1{\Sigma}$ is the ``effective'' set size determined by prior sparsity.\footnote
{
We have $1 \ge \Sigma \ge \frac1k$, where the lower bound is from Cauchy-Schwarz; equality is obtained in the upper and lower bound when $\pi$ is a singleton and uniform, respectively.
} 
Recalling the definition of $\widehat\ell(h,j,i)$ from (\ref{eq:singlejiestimator}), we then consider the estimator
\begin{align}\label{eq:allsetsestimator}
\widehat \ell_M (h, S, \cA) = \frac{1}{\frac{np}2 - 2k + 1}\,\sum_{j=k}^{\mupper} \widehat\ell(h, j)\text{~~~~~ where~~~~~ }\widehat\ell(h,j) = \frac1{\Sigma} \sum_{i=1}^k \pi[i]^2\, \widehat\ell(h, j, i)\,.
\end{align}
Note $\widehat \ell_M (h, S, \cA) = \frac{1}{\frac{np}2 - 2k + 1}\,\sum_{j=k}^{\frac{np}2-k} \widehat\ell(h, j) \cdot \ind\{j \le M-k\}$. 
Thus by Theorem \ref{prop:unbiasedestimatoroneji}, $\widehat\ell_M(h, S, \cA)$ is unbiased.
Finally, define the ERM estimator
\begin{align}
\wh = \wh (S,\cA) = \arg\min_{h \in \hyps}\,  \widehat \ell_M(h, S, \cA)\,.\label{eq:erm_estimator_def}
\end{align}

We next state the following sample complexity guarantee for $\wh$. This is the main result of this paper. 
\begin{theorem}\label{thm:allattributionsetssamplecomplexity}
Let $\ell(h(x),y) = f_1(h(x)) + y f_2(h(x))$ be a bounded and Lipschitz loss function:
for some $F_1, F_2, L > 0$, 
$|f_1(h(x))| \leq F_1$, 
$|f_2(h(x))| \leq F_2$
for all $h \in \hyps$ and $x \in \cX$, and $f_2(\cdot)$ is $L$-Lipschitz. 
Let $\wh \in \hyps$ be the hypothesis returned by the ERM estimator (\ref{eq:erm_estimator_def}). 
Suppose $p,\,\delta \in (0,1/2]$, $k \leq \frac{np}{8}$, and $np = \Omega\left( \log\left(\frac{1}{\delta\,p} \max_{i \in [k]}\frac{1}{\pi[i]} \right) \right)$.
Then with $\mu$-probability at least $1-\delta$,
\[
\regret(\widehat h) = \widetilde O\left( \frac{L\, R_n(\hyps)}{\Sigma} + \frac{F_2}{\Sigma}\,\sqrt{\frac{\log 1/\delta}{n}} + \frac{F_2}{\Sigma}\,\sqrt{\frac{p\, \complexitymeasurehere}{n}}\right)\,,
\]
where $\pseudodim(\hyps)$ denotes the pseudo-dimension of $\hyps$. Here $\widetilde O(\cdot)$ hides 
a logarithmic dependence on $np, L, \Sigma, F_2$ but \textit{excluding} $1/\delta$.
\end{theorem}

Proving Theorem \ref{thm:allattributionsetssamplecomplexity} poses several challenges. The attribution sets may overlap, thus the estimators $\widehat \ell(h, j)$ are \textit{not} independent across different $j$. Moreover, the unbiased estimators $\widehat \ell(h, j, i)$ are \textit{also not} independent across $i$, as the $A_j[i]$ are not independent; for example, for any $1 \le j \le M$, it is known that there is at least one conversion among the labels corresponding to $A_j[1], \ldots, A_j[k]$. 

One way partially around the independence issue is to force independence across different $j$ by skipping data; instead of using all attribution sets, only use a largest subsequence of well-separated attribution sets, e.g.,  $A_k, A_{3k}, A_{5k} ,\ldots $. The drawback of this approach is that we are only using $\frac{M}{2k}$-many sets instead of the available $M$, in contrast to the ERM estimator (\ref{eq:erm_estimator_def}) that we consider. This would inevitably lead to suboptimal sample complexity guarantees. We instead eschew an approach based on independence of attribution sets, and as such we are able to use a sizeable fraction of 
them, as claimed above. The full proof of Theorem \ref{thm:allattributionsetssamplecomplexity} is in Appendix \ref{sa:proofs3}. A proof sketch follows.\\[-4mm]

\begin{proof}{\bfseries (of Theorem~\ref{thm:allattributionsetssamplecomplexity}; sketch)}
We sidestep the potentially complicated dependency structure by rewriting $\widehat \ell(h, S, \cA)$ as a function of the $X_i$, and directly studying the sensitivity of $\widehat \ell(h, S, \cA)$. We then control $\regret(\widehat h)$ by splitting into two separate uniform convergence guarantees: 
\begin{enumerate}
\item Convergence of ${\E}_{\mu|S}[\widehat \ell(h, S, \cA)]$ to its expectation $\popl(h)$, where probability is w.r.t. $S \sim \cD^n$. Here we use a Rademacher complexity analysis that views $f(S) = {\E}_{\mu|S}[\widehat \ell(h, S, \cA)]$ as a function of i.i.d. random variables $(X_1,Y_1),\ldots, (X_n,Y_n)$; we control the sensitivity of $f(S)$ to each individual pair $(X_i,Y_i)$. This yields the $\frac{L\, R_n(\hyps)}{\Sigma}$ term in the regret guarantee. 
\item Convergence of $\widehat \ell(h, S, \cA)$ to its expectation ${\E}_{\mu|S}[\widehat \ell(h, S, \cA)]$ in the conditional space where $S$ is frozen, where probability is w.r.t. the adversary. This instead relies on a covering argument that views $\widehat \ell(h, S, \cA)$ solely as a function of the attribution sets (since $S$ is frozen), and applies with high probability over the generation of $S$. This generates the term $\frac{F_2}{\Sigma}\,\sqrt{\frac{p\, \complexitymeasurehere}{n}}$.
\end{enumerate}
Finally, the middle term $\frac{F_2}{\Sigma}\,\sqrt{\frac{\log 1/\delta}{n}}$ is a confidence term that is common to both analyses.
\end{proof}
%

\begin{remark}\label{rem:smallpiindices}
We note that the condition $np = \Omega\left( \log\left(\frac{1}{\delta\,p} \max_{i \in [k]}\frac{1}{\pi[i]} \right) \right)$ in Theorem \ref{thm:allattributionsetssamplecomplexity} forces $\pi[i]$ not to be exponentially small in $np$. This assumption is only made above for technical convenience, and can be relaxed to $np \ge \log\Big( \frac{\sqrt{2k}}{\delta p} \Big)$. Indeed, we can always construct $\widehat \ell(h, j)$ by restricting to the $i$ such that $\pi[i] \ge \frac1{\delta p} e^{-np}$. One then follows the exact same proof of Theorem \ref{thm:allattributionsetssamplecomplexity} in Appendix \ref{sa:proofs3}, only considering these $i$. Note that, as per Theorem \ref{prop:unbiasedestimatoroneji}, unbiasedness is retained. The difference is that $\Sigma$ now is replaced by the slightly smaller quantity $\Sigma' = \sum_{i\,:\,\pi[i] \ge \frac1{\delta p} e^{-np}} \pi[i]^2$. Yet, when $np \ge \log\Big( \frac{\sqrt{2k}}{\delta p} \Big)$, 
since $\Sigma \ge \frac1k$, we have
\(
\sum_{i\,:\,\pi[i] < \frac1{\delta p} e^{-np}} \pi[i]^2 \le \frac{k e^{-2np}}{\delta^2 p^2} \le \frac1{2k} \le \frac{\Sigma}2~,
\)
implying $\Sigma' \ge \frac{\Sigma}2$.
Thus the regret rate only changes by a constant factor.
\end{remark}

\begin{remark}
It is worth stressing that the weights $w_i = \pi[i]^2/\Sigma$, in the definition of $\widehat\ell(h,j)$ in (\ref{eq:allsetsestimator}) have been selected just for the sake of minimizing the variance of $\widehat \ell (h, j)$. 
Note the random variable in $\widehat \ell (h, j, i)$ is $f_2(h(A_j[i]))/\beta_1(j,i)$, and $\beta_1(j,i) \approx \pi[i]/p$ (Lemma \ref{lem:beta01approxsamplecomplexity}). Ignoring pairwise correlations, we approximate the variance of $\widehat \ell (h, j)=\sum_{i=1}^k w_i \widehat \ell (h, j, i)$ by $\sum_i w_i^2/\pi[i]^2$, which is minimized at $w_i\propto \pi[i]^2$, as per our choice in (\ref{eq:allsetsestimator}). 
Moreover, one can show that within each attribution set, such pairwise correlations are {\em non-positive}. The proof of this claim is lengthy but otherwise similar to the proof of Theorem \ref{prop:unbiasedestimatoroneji}. Yet, since this is not needed to prove the other results, we did not include it in the paper for brevity. 
In general, it seems difficult to find the optimal weights to construct the {\em global} estimator $\widehat \ell_M (h, S, \cA)$ from the $\widehat\ell (h, j)$ in (\ref{eq:allsetsestimator}), due to possible overlaps among attribution sets.  
\end{remark}

We now instantiate Theorem \ref{thm:allattributionsetssamplecomplexity} on several concrete examples. 
\begin{enumerate}
\item \textbf{Uniform vs decaying priors.} For a uniform prior $\pi[\cdot]$, our regret bound is of the form 
$$
k \cdot \frac{L R_n(\hyps) + F_2\sqrt{\log 1/\delta} + F_2 \sqrt{p\,\min\{ \pseudodim(\hyps), k \}}}{\sqrt{n}}~.
$$
However, in practice the last-touch heuristic (last element in attribution set triggers conversion) is often approximately true. A heavily decaying prior, e.g., with polynomial or exponential decay, is a more accurate model. For these examples, $\Sigma$ is much larger, being $\Omega(1)$. Our regret in Theorem \ref{thm:allattributionsetssamplecomplexity} adapts to such settings gracefully.
\item \textbf{$\hyps$ is a VC-class.} That is, $h(x) \in \{0,1\}$ for all $x \in \cX$ and $h \in \hyps$, and its VC-dimension $VCdim(\hyps)=d<\infty$. Our regret bound is of the form 
$$
\frac{L\sqrt{d}+ F_2\sqrt{\log 1/\delta} + F_2\sqrt{p\,\min\{d,k\}}}{\Sigma \sqrt{n}}~.
$$ 
This follows from the fact that for VC-classes, $R_n(\hyps) = \widetilde O\big(\sqrt{VCdim(\hyps)/n}\big)$.
\item \textbf{$|\hyps|$ is a finite class.} Our regret bound is of the form 
$$ 
\frac{L \sqrt{\log |\hyps|} + F_2\sqrt{\log 1/\delta} + F_2 \sqrt{p\log |\hyps|}}{\Sigma \sqrt{n}}~.
$$ 
This follows from Massart's Finite Lemma \citep{ma00} and the bound $\pseudodim(\hyps) \leq \log |\hyps|$.
\item \textbf{When $pk = O(1)$.} Here the attribution sets have minimal overlap with high probability. Our regret bound is of the form 
$$ 
\frac{L R_n(\hyps)}{\Sigma} + \frac{F_2}{\Sigma}\sqrt{\frac{\log 1/\delta}{n}}~.
$$
\end{enumerate}

\subsection{Robustness to errors in the prior}
Here we discuss how to extend Theorem \ref{thm:allattributionsetssamplecomplexity} to when the learner knows an estimate $\widehat{\pi}$ of the distribution $\pi$ with small error. This is a realistic situation where, e.g., $\pi$ is well-approximated by a parametric form, and we can estimate its parameters, for example with a small source of labeled data in the clear. 

The algorithm is similar but implemented in terms of $\widehat\pi$. 
We define $\widehat\beta_1(j,i), \widehat\beta_0(j,i)$ the same way as $\beta_1(j, i), \beta_0(j, i)$, but using $\widehat\pi$ instead of $\pi$, and then define $\widehat \ell(h, j, i), \widehat\ell(h,j)$ in terms of $\widehat\beta_1(j,i), \widehat\beta_0(j,i)$ as before. The explicit definition is provided in Appendix \ref{subsec:robustnessproofs}. We now let
\begin{align*}
\widehat\ell = \widehat\ell(h, S, \cA) = \frac1{\frac{np}2 - 2k + 1} \sum_{j=k}^{\mupper} \widehat\ell(h, j)\,,\qquad \wh = \wh (S,\cA) = \arg\min_{h \in \hyps}\, \widehat \ell(h, S, \cA)\,.
\end{align*}
We will establish in Appendix \ref{subsec:robustnessproofs} that we can obtain a regret $\regret(\widehat h)$ that is the sum of a ``\bias'' term, plus a term analogous to the regret from Theorem \ref{thm:allattributionsetssamplecomplexity}. 
The Bias term decreases in the approximation error between $\pi$ and $\widehat\pi$, and does not feature explicit $k$-dependence when the squared $L_2$ distance $\|\pi -\widehat \pi\|_2^2 \le \frac{\Sigma}8$ (recall $\Sigma = ||\pi||_2^2$). 
When $\|\pi -\widehat \pi\|_2^2 > \frac{\Sigma}8$, the explicit $k$-dependence in $\regret(\widehat h)$ below is unavoidable using our current analysis; see Remark \ref{rem:tightnessofrobustnessproof} in Appendix \ref{sa:proofs3}.
\begin{theorem}\label{thm:robustallattrsets}
Using the same notation as defined above, under the same conditions on the loss function $\ell$ and on $k$ as in Theorem \ref{thm:allattributionsetssamplecomplexity}, we have the following. If $np = \Omega\Big( \log\big( \frac{\sqrt{2k}}{\delta p}\big)\Big)$, then with $\mu$-probability at least $1-\delta$, the ERM estimator $\widehat h$ satisfies
\begin{align*}
\regret(\widehat h) = 
\begin{cases}
\widetilde O\left( \frac{L\, R_n(\hyps)}{\Sigma} + \frac{F_2}{\Sigma}\,\sqrt{\frac{\log \frac{1}{\delta}}{n}} + \frac{F_2}{\Sigma}\,\sqrt{\frac{p\, \complexitymeasurehere}{n}}\right) + \bias & {\mbox{if\,\, }} \Sigma \ge 8 \| \pi - \widehat\pi \|_2^2\,, \\ 
\widetilde O\left( L k\, R_n(\hyps) + F_2 k\,\sqrt{\frac{\log \frac{1}{\delta}}{n}} + F_2 k\,\sqrt{\frac{p\, \complexitymeasurehere}{n}}\right) + \bias & {\mbox{if\,\, }} \Sigma < 8\| \pi - \widehat\pi \|_2^2\,,
\end{cases}
\end{align*}
where $\pseudodim(\hyps)$ denotes the pseudo-dimension of $\hyps$, and where
\begin{align*}
\bias := \begin{cases} O\Big( pF_2 \Big( \frac{\|\pi - \widehat \pi\|_1}{\Sigma} + \frac{\|\pi - \widehat \pi\|_2}{\Sigma^{3/2}} \Big) \Big) & {\mbox{if\,\, }} \Sigma \ge 8 \| \pi - \widehat\pi \|_2^2\,, \\
O\Big( \frac{p k F_2 \|\pi - \widehat \pi\|_2}{\Sigma^{1/2}} \Big) & {\mbox{if\,\, }}  \Sigma < 8\| \pi - \widehat\pi \|_2^2\,. \end{cases}
\end{align*}
Again, $\widetilde O(\cdot)$ hides 
a logarithmic dependence on $np, L, \Sigma, F_2$ but \textit{excluding} $1/\delta$.
\end{theorem}

\begin{remark}
As an application of Theorem \ref{thm:robustallattrsets}, it is instructive to consider the case where the prior $\pi$ is unknown, but can be estimated by using a smaller set of labeled data. Given such a labeled dataset, an alternative method would be to simply perform ERM directly on that data. We now show that, in regimes of practical interest, our results in Theorem \ref{thm:robustallattrsets} yields a superior rate over ERM just operating on the smaller set of data available in the clear. Most notably, Theorem \ref{thm:robustallattrsets} avoids any dependence on the complexity of $\hyps$ evaluated on this reduced sample size, unlike ERM. 

For concreteness, suppose $\hyps$ is a VC-class of VC-dimension $d$ and that, on top of attribution sets, the labels of $\cS$, a set of examples $(x,y)$ covered by $s$ attribution sets of size $k$ are also known. Hence, overall, $|\cS| \leq sk$. 
Assume, as plausible in practice, that $sk \ll n$. 
Then ERM on $\cS$ gives a rate at best of the form $\sqrt{d/(sk)}$. Meanwhile, if we estimate $\pi$ by empirical probabilities over $\cS$, {\em Bias} from Theorem \ref{thm:robustallattrsets} is of the form $pk^2/\sqrt{s}$ (this is when $\pi$ is uniform; for a non-uniform $\pi$, this can only improve).
Theorem \ref{thm:robustallattrsets} thus yields a rate of the form
$k\sqrt{d/n} + pk^2/\sqrt{s}$, improving on ERM operating solely on $\cS$ when $d$ is large relative to $pk^2$ (e.g., $d \geq k(pk^2)^2$), and $n$ is large relative to $s$ (e.g., $n \geq k^3 s$), a setting which is arguably closer to practice.
\end{remark}

\vspace{-0.25in}
\section{Experiments}\label{sa:mainbody_experiments}
\vspace{-0.05in}
We conduct preliminary experiments to validate the estimator constructed from Theorem \ref{thm:allattributionsetssamplecomplexity} vs. simple baselines that correspond to industry heuristics (e.g., \cite{ktena2019addressing}). Our experiments are performed on MNIST \citep{lecun2010mnist}, CIFAR-10 \citep{Krizhevsky09Cifar}, and Higgs \citep{Baldi2014Higgs}, 
each modified in a way compatible with our model. More details are in Section \ref{sa:experiments}. \\[-4mm]

\noindent\textbf{Modifying the datasets.} 
We binarize each dataset: 1-vs-rest for MNIST, Animal-vs-Machine for CIFAR-10, while Higgs is natively binary. We then shuffle the data. For each positive label, we generate an attribution set by drawing an interval of $k$ adjacent indices that contain the positive label, where the position of the window is drawn according to the prior $\pi$. The algorithms observe only the unlabeled data and the resulting attribution sets. We consider a uniform prior $\pi=\big(\frac1k, \ldots, \frac1k\big)$, and an exponential prior $\pi \propto \big(2^{-k}, \ldots, 2^{-2}, 2^{-1}\big)$ motivated by last-touch attribution heuristics. \\[-4mm]

\noindent\textbf{Algorithms.} We implement three algorithms. For each algorithm the base loss $\ell(\widehat y,y)$ is log loss with prediction $\widehat y$ clipped to the interval $[0.01,0.99]$ for boundedness and numerical stability.
\begin{enumerate}
\vspace{-0.07in}
\item Our algorithm ({\sc unbiased}): minimizes the loss estimator in (\ref{eq:allsetsestimator}), where $p$ in Theorem \ref{prop:unbiasedestimatoroneji} is estimated by the fraction of positive labels in the training set, and $\E[f_1(h(X)]$, $\E[f_2(h(X)]$ are estimated on each mini-batch (see training details below) by empirical averages. While this estimation introduces a slight bias, our experiments confirm that this effect is negligible. 
\vspace{-0.07in}
\item {\sc random} baseline: Both this algorithm and the following baseline operate directly on the base loss $\ell(\widehat y,y)$ (with the same clipping for $\widehat y$) on fully supervised but {\em hallucinated} labels. 
{\sc random} assigns label 1 to a single position $i$ per attribution set, drawing this position according to $\pi$, and label $0$ to all remaining points in the attribution set. The data points in the training set that do not fall into any attribution set are assigned label $0$. Overlapping attribution sets may produce duplicate instances with potentially conflicting labels.\footnote{A natural alternative would be to generate fractional labels, but we did not explore this solution here.} 
\vspace{-0.07in}
\item {\sc max prior} baseline: analogous to {\sc random}, but the positive label is placed deterministically at the position $i$ where $\pi[i]$ is maximized.
\end{enumerate}

\vspace{-0.1in}
\noindent\textbf{Training and evaluation.} For each dataset, we train standard neural architectures known to perform reasonably well: a 3-hidden-layer MLP for MNIST, a 2-layer CNN for CIFAR-10, and a fully connected network for Higgs, training with the Adam optimizer for each algorithm \citep{kingma2015adam}. 
For {\sc random} and {\sc max prior}, we take a minibatch of 128 training examples with the hallucinated labels. For {\sc unbiased}, we estimate the loss from (\ref{eq:allsetsestimator}) by subsampling a minibatch of 128 attribution sets and another 128 training examples (without labels) directly from the dataset in order to estimate the expectation components ${\E}[f_1(h(X))]$, ${\E}[f_2(h(X))]$.  
For all three algorithms, we use bag sizes $k = 2^i$ for $0 \le i \le 8$ (or $0 \le i \le 7$ for Higgs), learning rates range in 10 log-spaced values from $10^{-6}$ to $10^{-2}$, and we use 200 training epochs for MNIST and 100 for CIFAR-10 and Higgs. 
Each experiment (a given dataset, algorithm, attribution set size, and learning rate) is repeated 10 times with randomized data shuffling and model initialization. Performance is measured on the test set with labels in the clear, averaged across repetitions. For each dataset, algorithm, and attribution set size, we report the best average over learning rates. \\[-4mm]

\noindent\textbf{Results.} We report test set accuracy (Figure \ref{f:3}), and test set log loss (Figure \ref{f:4} in Appendix \ref{sa:experiments}); for MNIST, we also show test set F1-measure in Figure \ref{f:4} due to label imbalance.
As expected, the performance of all algorithms degrades as $k$ increases. For large enough $k$, performance becomes trivial, for instance with {\sc random} and {\sc max prior} on CIFAR-10 when $k \geq 4$. On CIFAR-10, the trivial accuracy performance of 60\% is obtained by always predicting ``1''; for MNIST, always predicting ``0'' achieves 88.65\% accuracy. In such cases, we report the trivial performance level (with 0 variance) instead of the actual performance in Figure \ref{f:3}. Several observations can be made:
\begin{itemize}
\vspace{-0.07in}
\item {\sc unbiased} vs. {\sc random} and {\sc max prior}. We can see the clear advantage offered by our theory as opposed to the baselines; the performance gap is striking in all cases.
\vspace{-0.07in}
\item Uniform vs. Exponential prior. All algorithms perform better with exponential prior than with uniform, with {\sc max prior} being comparatively better than {\sc random}. Note for the exponential prior, $\Sigma \rightarrow \frac13$ as $k \rightarrow \infty$; our Theorem \ref{thm:allattributionsetssamplecomplexity} predicts {\sc unbiased} will degrade gracefully as $k$ increases, consistent with our results in Figure \ref{f:3}.
\vspace{-0.07in}
\item Small $k$. At $k=1$, the two baselines both reduce to full supervision with true training labels (no label hallucination); this is not the case for {\sc unbiased}. Thus the baselines' performance at $k=1$ ($2^0$ in Figures \ref{f:3} and \ref{f:4}) reflects fully supervised training with the base loss.
\vspace{-0.07in}
\item Overlapping ($pk > 1$) vs. non-overlapping ($pk < 1$) regime.
The difference here can only be appreciated on MNIST, where $p \approx 0.1$. {\sc unbiased} remains largely unaffected by the attribution set overlap, but both baselines are affected significantly, especially {\sc random}.
\end{itemize}

\begin{figure*}[t!]
\vspace{-0.20in}
    \centering
\hspace{0.1in}\includegraphics[width=0.246\textwidth]{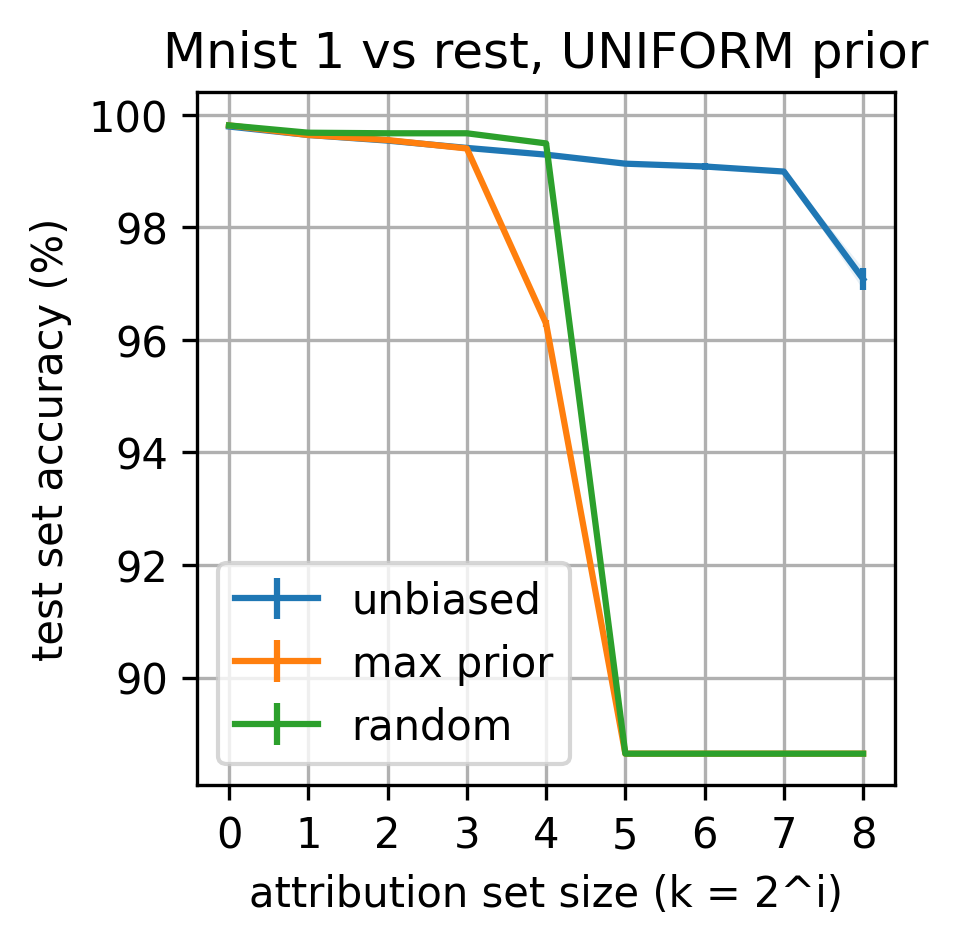}
\hspace{0.1in}\includegraphics[width=0.30\textwidth]{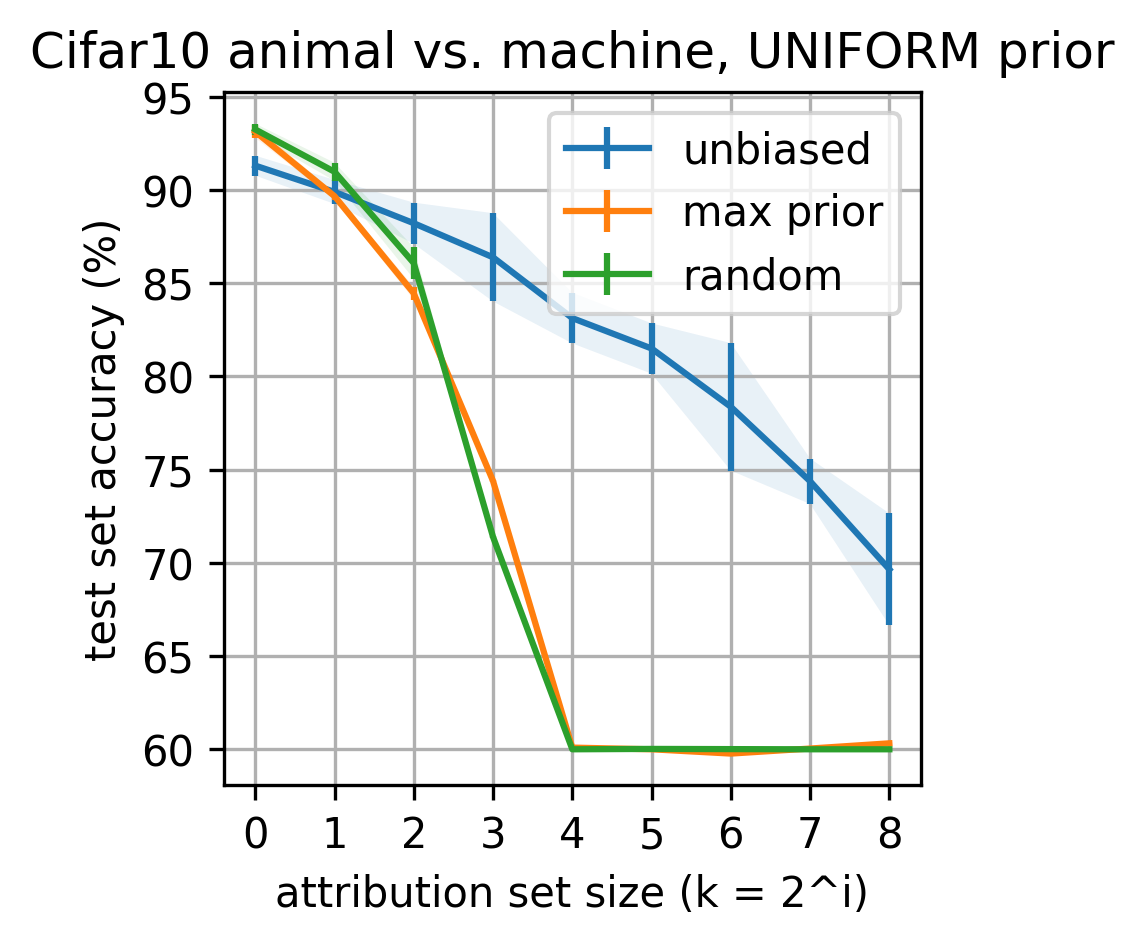}
\hspace{-0.15in}
\includegraphics[width=0.235\textwidth]{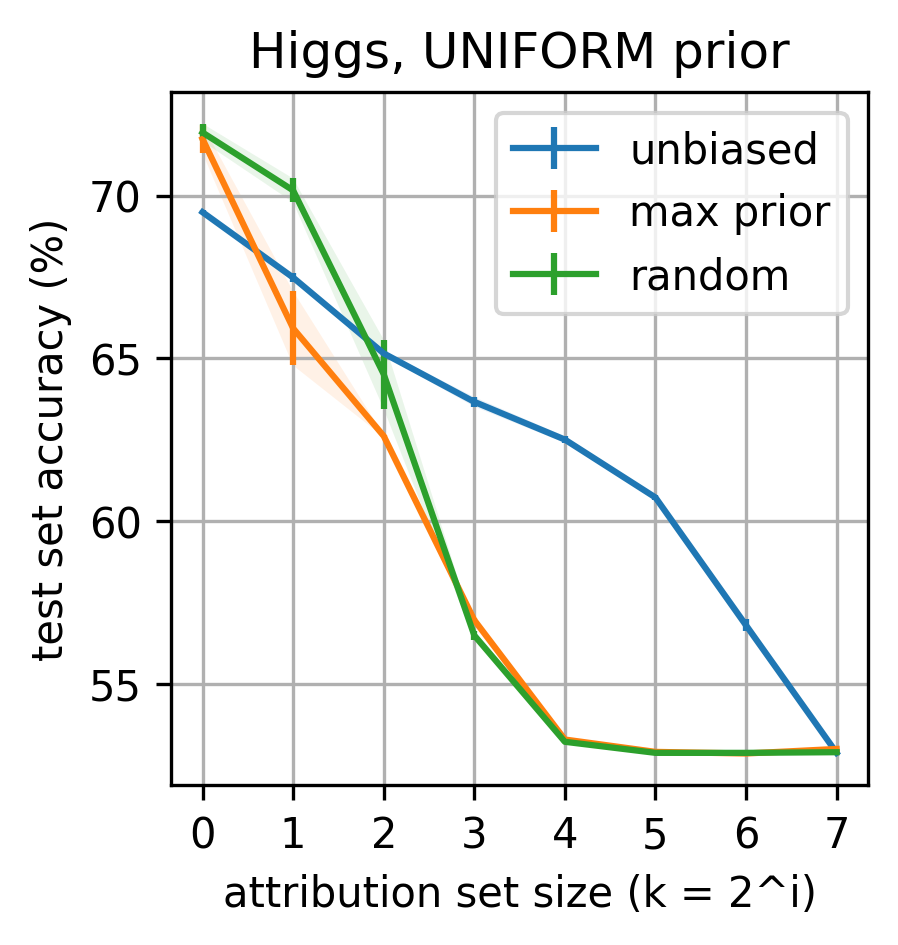}\\[-2mm]
\hspace{0.1in}
\includegraphics[width=0.24\textwidth]{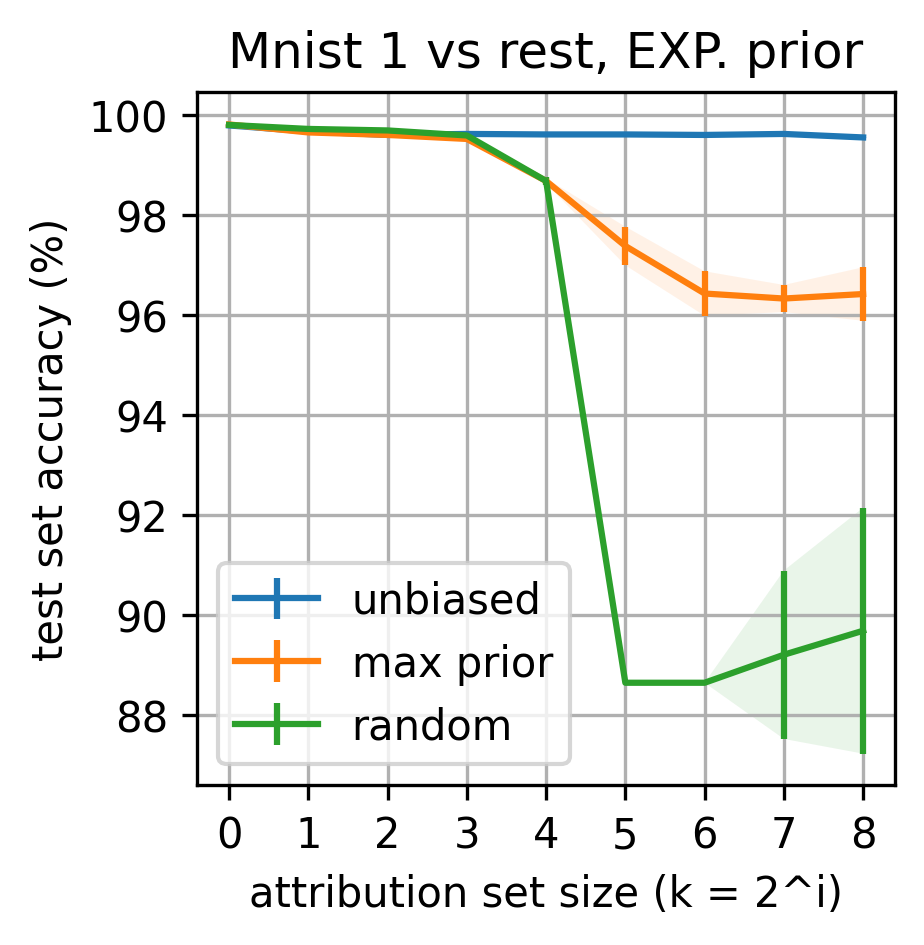}
\hspace{0.15in}
\includegraphics[width=0.274\textwidth]{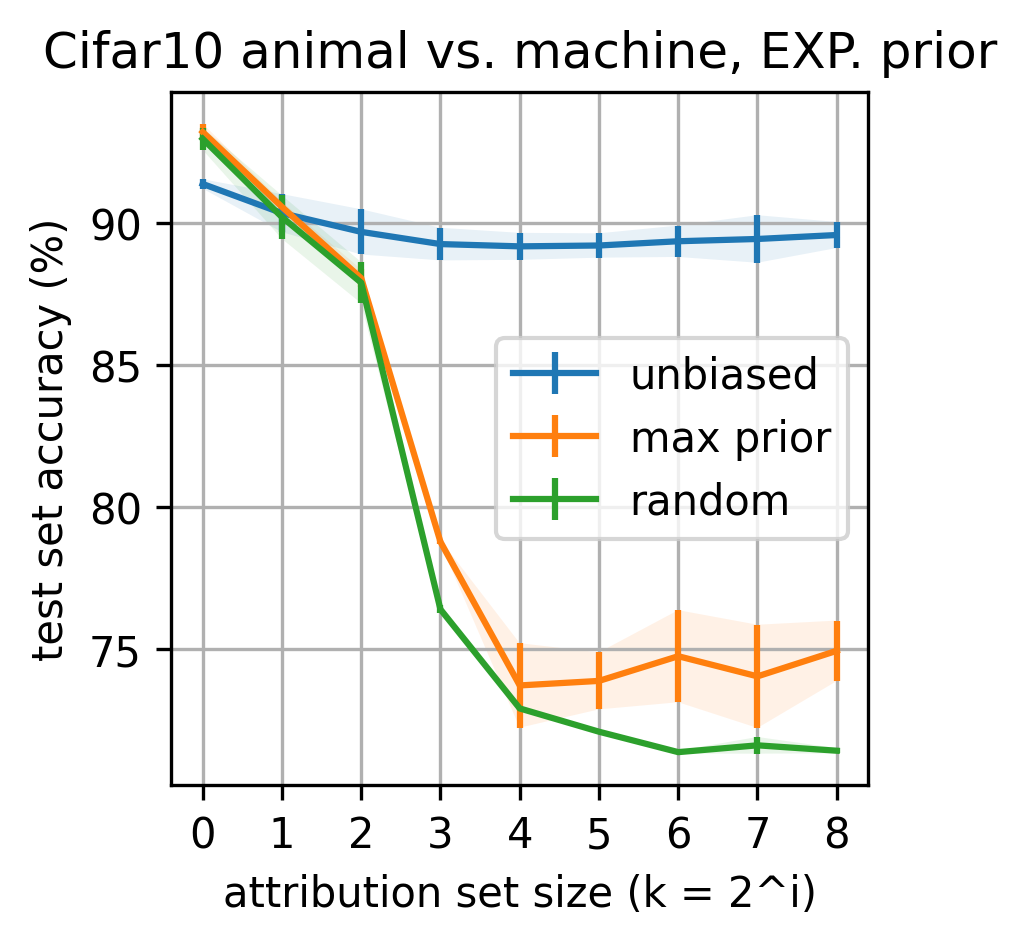}
\hspace{-0.05in}
\includegraphics[width=0.25\textwidth]{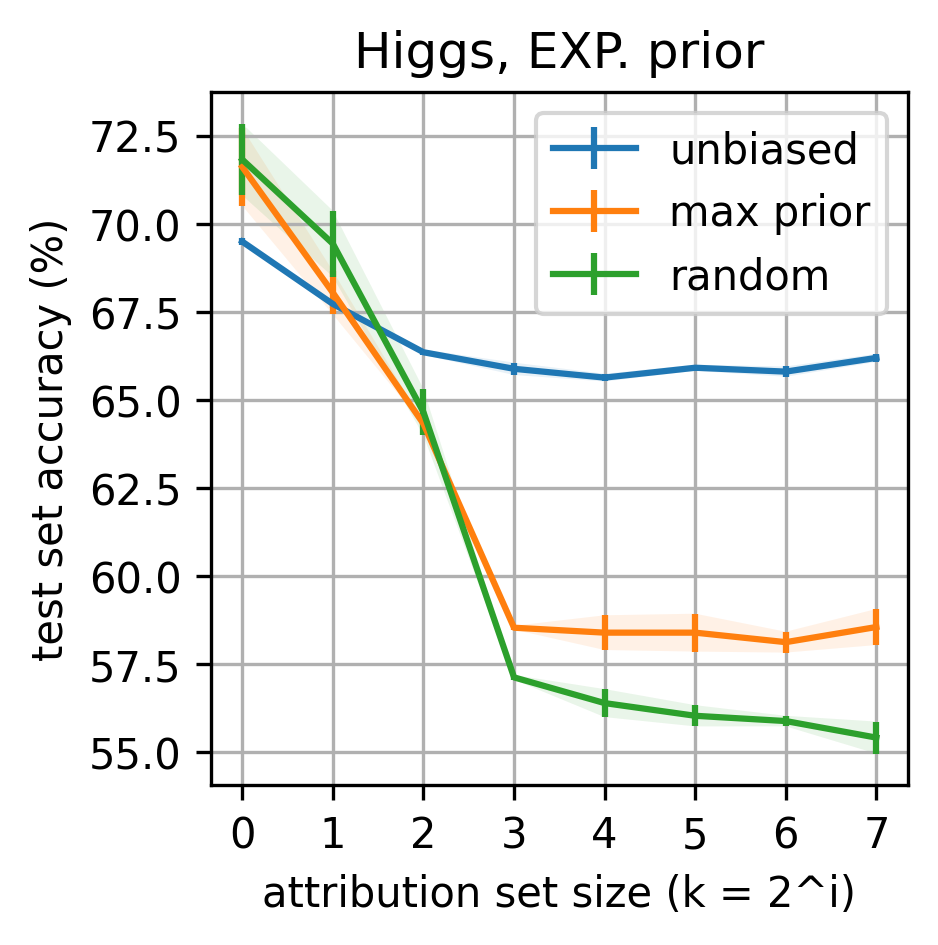}
\vspace{-0.2in}
    \caption{Experiment on MNIST (1-vs-rest), CIFAR-10 (animal vs. machine) and Higgs datasets. We plot test set accuracy vs. attribution set size $k=2^0,2^1,\ldots, 2^8$ (or $k=2^0,2^1,\ldots, 2^7$ for Higgs), averaged over 10 repetitions. On the top row is the uniform prior, on the bottom is the exponential prior. For MNIST, the trivial accuracy performance is $88.65\%$, for CIFAR-10 it is $60\%$, for Higgs it is $52.87\%$.
    Standard deviations are also depicted.}
    \label{f:3}
    \vspace{-0.2in}
\end{figure*}

\vspace{-0.2in}
\section{Conclusions and Future Work}\label{s:conclusions}
\vspace{-0.05in}
We introduced a formal framework for statistical learning from attribution sets, addressing the growing challenge of tracking-prevention conversion prediction, where publishers observe clicks but only receive coarse signals about conversions. 
We derived an unbiased risk estimator, established generalization bounds that scale with the ``effective'' set size determined by prior sparsity, and proved robustness to prior estimation errors. Given the availability of unbiased (or approximately unbiased) loss estimators, these analyses can be readily adapted to stochastic gradient descent-like algorithms, since unbiasedness of loss estimates translate to unbiasedness of loss gradient estimates.
Our preliminary experiments suggest that our method significantly outperforms simple industry heuristics on readily available datasets, particularly when attribution sets are large and/or overlapping.
Future work includes the following.
\begin{itemize}
\vspace{-0.05in}
\item We have assumed an oblivious adversary and a known (or partially known) prior $\pi$ over positions within each attribution set; this structure is critical for identifiability. The estimators in Theorem \ref{thm:allattributionsetssamplecomplexity} and Theorem \ref{thm:robustallattrsets} hinge on expressing the inaccessible population moment $\mathbb{E}[Y f_2(h(X))]$ as an expectation over observable 
quantities, with coefficients that depend explicitly on $\pi$. 
One may wonder if learning against a non-oblivious adversary is at all possible. In a non-oblivious regime, the above decoupling fails. Since an attribution set merely indicates the presence of at least one positive label, multiple labelings--and thus multiple population risks--remain consistent with the same observations. This lack of identifiability suggests that consistent learning is 
impossible without additional structure. We provide this as an intuition rather than a formal impossibility result. Characterizing the minimal assumptions required for learning under adaptive adversaries is a compelling direction for future work.
\vspace{-0.05in}
\item Our agnostic bounds extend to the realizable setting. However, capturing the fast convergence rates expected under benign losses (like square loss) requires a localized Rademacher complexity analysis, which we leave to future work. Proving matching regret lower bounds to establish the tightness of our results also remains an open problem.
\vspace{-0.05in}
\item We plan to extend out framework to more complex attribution logics, like multi-touch attribution (where multiple clicks contribute to a single conversion), and to develop methods to jointly learn the prior distribution and the conversion model from the data stream itself.

\end{itemize}

\vspace{-0.15in}

\section{Acknowledgments}
We thank the anonymous reviewers for their useful comments that helped improve the presentation of this paper. Also, we thank Haim Kaplan for early discussions related to this paper.
This work was done while AC was a student researcher at Google Research, NY.

\appendix


\newpage
\section{Proofs for Section \ref{s:unbiased}}\label{sa:proofs1}
This section contains the proofs that apply to the individual attribution set.

\subsection{Proof of Theorem \ref{prop:unbiasedestimatoroneji}}

First we need the following lemma which lets us control the probabilities that labels prior to and after $i_j(S)$ are 1.

\begin{lemma}\label{lem:problabel1beforeij}
For any $t \ge 0$ with $t+1 \le j$, we have 
\[ 
\PP_{S \sim \cD^{ n}}\Bigl(Y_{i_j(S) \pm t} = 1 \,, j \le M-k \Bigl) = p\, B_{n-1, p, j+k-1}\,,
\]
where $B_{n, p, j}$ is defined in Theorem \ref{prop:unbiasedestimatoroneji}.
\end{lemma}
\begin{proof}
Note the event $\{ Y_{i_j(S) \pm t} = 1 \} \cap \{ j \le M-k \}$ can be rewritten as a disjoint union over all $m$ such that $j+k \le m \le n$, of the intersections $\{ Y_{i_j(S) \pm t} = 1 \} \cap \{ M=m \}$. 
Now we claim that for fixed $m$, the number of binary strings satisfying $\{ Y_{i_j(S) \pm t} = 1 \} \cap \{ M=m \}$ is $\binom{n-1}{m-1}$. 
\begin{itemize}
\item For $Y_{i_j(S) - t}$: Note $\{ Y_{i_j(S) \pm t} = 1 \} \cap \{ M=m \}$ can be written as the disjoint union over all $l$ such that $\max\{ t+1,j \} \le l \le n$ of the intersection $\{ i_j(S)=l \} \cap \{ Y_{l - t}=1 \} \cap \{ M=m \}$. The reason $l \ge j$ is because $i_j(S)$ is the $j$-th occurrence of a sample in the stream $S$ with label 1, thus it must occur at the $j$-th position or after. Moreover we need $l \ge t+1$ for the event $\{ Y_{i_j(S)-t}=1 \}$ to make sense. This event occurs if and only if $Y_l=Y_{l-t}=1$, if there are exactly $j-2$ ones in $[l-1] - \{l - t\}$, and if there are exactly $m-j$ ones in $\{l+1, \ldots, n\}$. 
These ones can be chosen in $\binom{l-2}{j-2} \binom{n-l}{m-j}$ ways, where we adopt the convention that $\binom{\cdot}{-1}=0$ (which matches the combinatorial interpretation). Since $j \ge t+1$, it follows that the number of desired strings in this case is 
\begin{align*}
\sum_{l = j}^n \binom{l-2}{j-2}\binom{n-l}{m-j} = \sum_{k=0}^{n-j} \binom{j-2+k}{j-2} \binom{n-j-k}{m-j}=\binom{n-1}{m-1}\,,
\end{align*}
where we apply Vandermonde's identity in the last step.

\item For $Y_{i_j(S)+t}$: Similar to the proof of Lemma \ref{lem:problabel1beforeij}, note the event $\{ Y_{i_j(S) + t} = 1 \} \cap \{ M=m \}$ can be rewritten as a disjoint union over all $l$ such that $j \le l \le \min\{n-t, n-m+k\}$ of the intersection $\{ i_j(S)=l \} \cap \{ Y_{l+t}=1 \} \cap \{ M=m \}$. Here $l \le \min\{n-t, n-m+j\}$ because we need $i_j(S)+t \in [n]$ and as there are $m-j$ 1s in $\{l+1, \ldots, n\}$. 
These ones can be chosen in $\binom{l-1}{j-1} \binom{n-l-1}{m-j-1}$ ways. Note $n-m+j \le n-t$ as $m \ge j+k \ge j+t$. Thus, the total number of desired strings in this case is
\begin{align*}
\sum_{l=j}^{n-m+j}\binom{l-1}{j-1} \binom{n-l-1}{m-j-1} 
&= \sum_{k=0}^{n-m} \binom{k+j-1}{j-1} \binom{n-j-1-k}{m-j-1} = \binom{n-1}{m-1}\,,
\end{align*}
where we again apply Vandermonde's identity in the last step.
\end{itemize}
Each such binary string occurs with probability $p^m (1-p)^{n-m}$. Hence
\begin{align*}
\PP_{S \sim \cD^{ n}}\Bigl(Y_{i_j(S) \pm t} = 1 \cap j \le M-k \Bigl) &= \sum_{m=j+k}^n \binom{n-1}{m-1} p^m (1-p)^{n-m} \\
&= p \sum_{m'=j+k-1}^{n-1} \binom{n-1}{m'} p^{m'} (1-p)^{n-1-m'} \\
&= p B_{n-1, p, j+k-1}\,.
\end{align*}
Also observe that this probability is independent of $t$ for $j \ge t+1$, as can be seen by symmetry among the indices $1, 2, \ldots, i_j(S) - 1$. 
\end{proof}

We next control the law of $X_{i_j(S)\pm t}$ conditioned on $Y_{i_j(S)\pm t}$.
\begin{lemma}\label{lem:conditionallawoneidx}
Consider any $j$ and $t \ge 0$ such that $j \ge t+1$. Then:
\begin{itemize}
    \item The law of $X_{i_j(S)\pm t}$ conditioned on the event $\{Y_{i_j(S)\pm t} = 1\,,\, j \le M-k\}$ is that of a random variable distributed according to $X|Y=1$.
    \item The law of $X_{i_j(S) \pm t}$ conditioned on the event $\{Y_{i_j(S) \pm t} = 0\,,\, j \le M-k\}$ is that of a random variable distributed according to $X|Y=0$.
\end{itemize}
\end{lemma}
\begin{proof}
The high-level idea is that the event $\{j \le M-k\}$ is only determined by the labels, and given the value of the labels (thus $Y_{i_j(S) \pm t}=y$), the law of $X_{i_j(S) \pm t}$ is determined as $X|Y=y$ since the data points $(X_i, Y_i)$ are i.i.d.. We will prove the $\{Y_{i_j(S)-t} = 1\}$ case, the proofs for the other cases being analogous.

To make this formal, we let $\vecyS$ denote the labels $Y_1,\ldots, Y_n$ (we write $\vecyS$ to make it clear that this random variable depends on $S$). We also let $\vecs$ denote a fixed binary string in $\{0, 1\}^n$, and let $i_j(\vecs)$ denote the index of the $j$-th 1 in the binary string $\vecs$. For a binary string $\vecs \in \{0,1\}^n$, we let $|\vecs|$ denote the number of its 1s. Therefore for $\vecs$ such that its ($l-t$)-th entry is 1, there are $|\vecs|-1$ other 1s in the string $\vecs$, and $n - |\vecs|$ 0s. We have 
\begin{align*}
\PP&_{S \sim \cD^{ n}} \Bigl( X_{i_j(S)-t}=x\, |\, Y_{i_j(S)-t} = 1\,,\, j \le M-k \Bigl) \\
&= 
\frac{\PP_{S \sim \cD^{ n}} \Bigl( X_{i_j(S)-t}=x\,,\, Y_{i_j(S)-t} = 1\,,\, j \le M-k \Bigl)}{\PP_{S \sim \cD^n} \Bigl( Y_{i_j(S)-t}=1\,,\, j \le M-k \Bigl)}\\
&= \frac{ 1 }{\PP_{S \sim \cD^n} \Bigl(Y_{i_j(S)-t}=1\,,\, j \le M-k \Bigl)} \Biggl( \sum_{\vecs \in \{0,1\}^n} \PP_{S \sim \cD^{ n}} \Bigl( X_{i_j(S)-t}=x\, \Big|\,  \vecyS = \vecs\,,\, Y_{i_j(S)-t} = 1\,,\, j \le M-k \Bigl) \\
&\qquad \qquad \qquad \qquad \qquad \qquad \qquad \qquad \qquad \qquad \qquad  \cdot \PP_{S \sim \cD^n}\Bigl( \vecyS = \vecs\,,\,  Y_{i_j(S)-t} = 1\,,\,j \le M-k \Bigr) \Biggr)\\
&= 
\frac{1}{\PP_{S \sim \cD^n} \Bigl( Y_{i_j(S)-t}=1\,,\, j \le M-k \Bigl)} \Biggl( \sum_{l=j}^n \sum_{~\vecs\,:\,  i_j(\vecs)=l} \PP_{S \sim \cD^{ n}}\Bigl( X_{l-t}=x\, \Big|\, \vecyS = \vecs\,,\,  Y_{l-t} = 1\,,\, j \le M-k \Bigl)\,  \\
&\qquad \qquad \qquad \qquad \qquad \qquad \qquad \qquad \qquad \qquad \qquad \cdot \PP_{S \sim \cD^n}\Bigl( \vecyS = \vecs\,,\, Y_{l-t} = 1\,,\, j \le M-k \Bigr) \Biggr)\,.
\end{align*}
The above uses the fact that we must have $i_j(\vecs) \ge j$, so the only possible indices $l$ for $i_j(\vecs)$ are $j, j+1, \ldots, n-1, n$. 
Notice if $|\vecs| \ge j+k$ then $\{\vecyS = \vecs\,,\, Y_{l-t} = 1\,,\, j \le M-k \} = \{\vecyS=\vecs\,,\, Y_{l-t}=1\}$. Else $\PP_{S \sim \cD^n}\Bigl( \vecyS = \vecs\,,\, Y_{l-t} = 1\,,\, j \le M-k \Bigr)=0$. 
Thus 
\begin{align*}
&\sum_{~\vecs\,:\,  i_j(\vecs)=l} \PP_{S \sim \cD^{ n}}\Bigl( X_{l-t}=x\, \Big|\, \vecyS = \vecs\,,\,  Y_{l-t} = 1\,,\, j \le M-k \Bigl) \PP_{S \sim \cD^n}\Bigl( \vecyS = \vecs\,,\, Y_{l-t} = 1\,,\, j \le M-k \Bigr) \\
&\qquad = \sum_{~\vecs\,:\,  i_j(\vecs)=l, |\vecs| \ge j+k} \PP_{S \sim \cD^{ n}} \Bigl( X_{l-t}=x \, \Big|\, \vecyS = \vecs \,,~ Y_{l-t} = 1  \Bigl) \PP_{S \sim \cD^n}\Bigl(\vecyS = \vecs \,,~ Y_{l-t} = 1 \Bigl)
\end{align*}

Then since $S$ is i.i.d., we have
\begin{align*}
\PP_{S \sim \cD^{ n}} \Bigl( & X_{l-t}=x \, \Big|\, \vecyS = \vecs \,,~ Y_{l-t} = 1  \Bigl) \\
&= 
\frac{ \PP_{S \sim \cD^{ n}}\Bigl(\vecyS = \vecs \,,~ Y_{l-t} = 1 \, \Big|\, X_{l-t} = x\Bigl)\, \PP_{S \sim \cD^{ n}}\Bigl(X_{l-t}=x \Bigl) }{\PP_{S \sim \cD^{ n}}\Bigl(\vecyS = \vecs \,,~ Y_{l-t} = 1  \Bigl)} \\
&= 
\frac{ \PP(Y=1\,|\,X=x)\cdot p^{|\vecs|-1}\cdot (1-p)^{n-|\vecs|}\cdot \PP(X=x)}{p \cdot p^{|\vecs|-1} \cdot (1-p)^{n-|\vecs|}} \\
&= 
\PPsub_1(X=x)\,,
\end{align*}
where the last step follows from Bayes' Rule.

Combining this with the earlier display, and using our earlier observations, thus gives
\begin{align*}
\E&_{S \sim \cD^{ n}}  \Bigl[ \ind\{ X_{i_j(S)-t}=x\}\, |\, y_{i_j(S)-t} = 1 \Bigl] \\
&= \PP_{S \sim \cD^{ n}}\Bigl( X_{i_j(S)-t}=x\, |\, Y_{i_j(S)-t} = 1 \Bigl) \\
&= 
\frac{1}{\PP_{S \sim \cD^n} \Bigl( Y_{i_j(S)-t}=1\,,\, j \le M-k \Bigl)} \Biggl( \sum_{l=j}^n \sum_{~\vecs\,:\, i_j(\vecs)=l, |\vecs| \ge j+k} \PP_{S \sim \cD^{ n}}\Bigl( X_{l-t}=x\, \Big|\, \vecyS=\vecs \,,~ Y_{l-t} = 1  \Bigl)   \\
&\qquad \qquad \qquad \qquad \qquad \qquad \qquad \qquad \qquad \qquad \qquad \qquad \qquad \cdot  \PP_{S \sim \cD^n}\Bigl( \vecyS = \vecs \,,~ Y_{l-t} = 1  \Bigr) \Biggl)\\
&= 
\frac1{\PP_{S \sim \cD^n} \Bigl( Y_{i_j(S)-t}=1\,,\,  j\le M-k \Bigl)} \cdot \PPsub_1(X=x) \cdot \sum_{l=j}^n \sum_{~\vecs\,:\, i_j(\vecs)=l, |\vecs| \ge j+k}  \PP_{S \sim \cD^n}\Bigl( \vecyS = \vecs \,,~ Y_{l-t} = 1  \Bigr)  \\
&= 
\frac1{\PP_{S \sim \cD^n} \Bigl( Y_{i_j(S)-t}=1\,,\,  j\le M-k \Bigl)} \cdot \PPsub_1(X=x) \cdot \sum_{l=j}^n \sum_{~\vecs\,:\, i_j(\vecs)=l}  \PP_{S \sim \cD^n}\Bigl( \vecyS = \vecs \,,\, Y_{l-t} = 1\,,\, j \le M-k  \Bigr)  \\
&= 
\PPsub_1(X=x) \cdot \frac{\PP_{S \sim \cD^n} \Bigl( Y_{i_j(S)-t}=1\,,\, j \le M-k \Bigl)}{\PP_{S \sim \cD^n} \Bigl(Y_{i_j(S)-t}=1\,,\, j \le M-k \Bigl)}\\
&= 
\PPsub_1(X=x)\,.
\end{align*}
This proves that $X_{i_j(S)-t}$ is distributed according to $X|Y=1$. The proof for $X_{i_j(S)+t}$ and the $\{Y_{i_j(S)\pm t} = 0\}$ case is analogous.
\end{proof}

Using Lemma \ref{lem:problabel1beforeij}, Lemma \ref{lem:conditionallawoneidx}, and conditioning on $Y_{i_j(S) \pm t}$, we obtain a formula for the law of $X_{i_j(S) \pm t}$.

\begin{lemma}\label{lem:previjSlaw}
For any $t \ge 0$ with $t+1 \le j$, we have 
\begin{align*}
\PP_{S \sim \cD^{ n}} \Bigl( X_{i_j(S) \pm t}=x \,, j\le M-k \Bigl) 
&= 
\PPsub_1(X=x) \cdot p\, B_{n-1, p, j+k-1} \\
&\qquad + \PPsub_0(X=x) \cdot \Bigl(1 - p\, B_{n-1, p, j+k-1} \Bigl)\,.
\end{align*}
\end{lemma}

\begin{proof}
We can rewrite
\begin{align*}
&\PP_{S \sim \cD^{ n}} \Bigl(X_{i_j(S)\pm t}=x \,, j \le M-k \Bigl)  \\
&\qquad =  
\PP_{S \sim \cD^{ n}} \Bigl( X_{i_j(S)\pm t}=x \,, j \le M-k \,|\, Y_{i_j(S)\pm t} = 1 \,, j \le M-k \Bigl) \PP_{S \sim \cD^{ n}} \Bigl(Y_{i_j(S)\pm t} = 1 \,, j \le M-k \Bigl) \\
&\qquad \qquad+ \PP_{S \sim \cD^{ n}} \Bigl( X_{i_j(S)\pm t}=x \,, j \le M-k \,|\, Y_{i_j(S)\pm t} = 0 \,, j \le M-k \Bigl) \PP_{S \sim \cD^{ n}}\Bigl(Y_{i_j(S)\pm t} = 0 \,, j \le M-k \Bigl).
\end{align*}
The result now follows by noting that $\PP(A \,, E\, |\, B \,, E) = \PP(A\, |\, B \,, E)$ for measurable events $A,B,E$, and then combining with Lemma \ref{lem:problabel1beforeij} and Lemma \ref{lem:conditionallawoneidx}.
\end{proof}

We now use the above to study the law of the $i$-th element of $A_j$, and prove Theorem \ref{prop:unbiasedestimatoroneji}.\\

\begin{proof} (of Theorem \ref{prop:unbiasedestimatoroneji})
First note that by the assumed interval structure, we have for all $1 \le r \le k$ that with probability $\pi[r]$,
\begin{align*}
A_j[i] = \begin{cases}
X_{i_j(S)+i-r} &\mbox{for\quad } 1 \le i_j(S) + i - r \le n \\
X_1 &\mbox{for\quad } i_j(S)+i-r < 1 \\
X_n &\mbox{for\quad } i_j(S)+i-r > n\,.
\end{cases}
\end{align*}
For any integer $i'$, we let $i'_{\text{truncate}} = 1$ if $i'<1$, $n$ if $i'>n$, and $i'$ if $1 \le i' \le n$.
We thus obtain 
\[ 
{\E}_{\pi}\Bigl[\ind \{ A_j[i]=x\} \, |\, S \Bigl] = \sum_{r=1}^k \pi[r]\, \ind\Bigl\{ X_{(i_j(S)+i-r)_{\text{truncate}}} = x\Bigl\}\,.
\]
Note that when $k \le j \le M-k \le n-k$, as $i-r \in \{-(k-1),\ldots, k-1\}$, we have $1 \le i_j(S) + i - r \in \{1, \ldots, n\}$.
Since the event $k \le j \le M-k$ is independent of $\pi$, only depending on $S$, we obtain
\begin{align*}
{\E}_{\pi}\Bigl[\ind \{ A_j[i]=x\} \cdot \ind\{ j \le M-k \} \, |\, S \Bigl] &= {\E}_{\pi}\Bigl[\ind \{ A_j[i]=x\} \, |\, S \Bigl]\, \cdot\, \ind\{ j \le M-k \} \\
&= \Big( \sum_{r=1}^k \pi[r]\, \ind\Bigl\{ X_{(i_j(S)+i-r)_{\text{truncate}}} = x\Bigl\} \Big) \cdot \ind\{ j \le M-k \} \\
&= \sum_{r=1}^k \pi[r] \ind\Bigl\{ X_{(i_j(S)+i-r)_{\text{truncate}}} = x\Bigl\}\, \cdot\, \ind\{ j \le M-k \} \\
&= \sum_{r=1}^k \pi[r] \ind\bigl\{ X_{i_j(S)+i-r} = x\bigl\}\, \cdot\, \ind\{ j \le M-k \}\,.
\end{align*}
Here the last equality follows as when $j \le M-k$, $(i_j(S)+i-r)_{\text{truncate}} = i_j(S)+i-r$.

Since $\pi$ does not depend on $S$, unfreezing over $S$ now gives 
\begin{align*}
{\E}_{\mu}\Bigl[\ind \{ A_j[i]=x\} \cdot \ind\{ j \le M-k \} \Bigl] 
&= 
\pi[i]\, \E_{S \sim \cD^n} \Big[ \ind\bigl\{ X_{i_j(S)} = x\bigl\}\, \cdot\, \ind\{ j \le M-k \} \Big] \\
&\qquad +  \sum_{r \neq i} \pi[r] \E_{S \sim \cD^{ n}} \Bigl[ \ind\{ X_{i_j(S)+i-r}=x\} \, \cdot\, \ind\{ j \le M-k \}  \Bigl]\,.
\end{align*}
The sum is now split into two parts: $r=i$, $r \neq i$. The high-level idea to understand the law $\PP_{S \sim \cD^n} \Big( X_{i_j(S) + i - r} = x \,, j \le M-k \Big)$ in each of these two cases.

\ \\
\textbf{Case 1, $r=i$ terms:} We rewrite
\begin{align*}
\E_{S \sim \cD^n} \Big[ \ind\bigl\{ X_{i_j(S)} = x\bigl\}\, \cdot\, \ind\{ j \le M-k \} \Big] &= \PP_{S \sim \cD^n} \Big( X_{i_j(S)} = x \,, j \le M-k \Big) \\
&= \PP_{S \sim \cD^n} \Big( X_{i_j(S)} = x\, \big|\, j \le M-k \Big) \PP_{S \sim \cD^n}\Big( j \le M-k \Big).
\end{align*}
Since $Y_{i_j(S)}=1$ is guaranteed to hold, and the event $\{j \le M-k\}$ is measurable based on only the labels $Y_1, \ldots, Y_n$, it follows that
\begin{align*}
\PP_{S \sim \cD^n} \Big( X_{i_j(S)} = x\, \big|\, j \le M-k \Big) = \PPsub_{1}(X=x)\,.
\end{align*}
Thus the contribution of this term is 
$$
\PPsub_1(X=x)\, \PP_{S \sim \cD^n}\Big( j \le M-k \Big) = \PPsub_1(X=x) \cdot B_{n, p, j+k}~.
$$

\ \\
\textbf{Case 2, $r \neq i$ terms:} Fix any $r \neq i$. By Lemma \ref{lem:previjSlaw}, as $j \ge k \ge t+1$, we obtain
\begin{align*}
&\E_{S \sim \cD^{n}} \Bigl[ \ind\{ X_{i_j(S)+i-r}=x\} \cdot \ind\{j \le M-k\} \Bigl] \\
&\qquad= \PP_{S \sim \cD^n} \Big( X_{i_j(S) + i - r} = x \,, j \le M-k \Big) \\
&\qquad= \PPsub_1(X=x) \cdot p B_{n-1, p, j+k-1 } + \PPsub_0(X=x)
\Bigl(1-p B_{n-1, p, j+k-1}\Bigl)\,.
\end{align*}
By Bayes' Rule, we have 
\[ 
\PPsub_1(X=x) = \frac{\PP(Y=1\, |\, X=x) \PP(X=x)}{\PP(Y=1)} = \frac{\eta(x)\, \PP(X=x)}{p}\,,
\]
where $p = \PP(Y=1)$, $\eta(x) = \PP(Y=1\, |\, X=x)$ and, analogously,
\[ 
\PPsub_0(X=x) = \frac{\bigl(1 - \eta(x) \bigr)\, \PP(X=x)}{1-p}\,.
\]
Thus
\begin{align*}
&\E_{S \sim \cD^{n}} \Bigl[ \ind\{ X_{i_j(S)+i-r}=x\} \cdot \ind\{j \le M-k\} \Bigl]\\
&\qquad =
\frac{\eta(x)}{p}\, \PP(X=x)\, p B_{n-1, p, j+k-1 } \\
&\qquad \qquad \qquad + \Bigl( \frac1{1-p}\, \PP(X=x) - \frac{\eta(x)}{1-p}\, \PP(X=x) \Bigl)\Bigl(1-p B_{n-1, p, j+k-1}\Bigl) \\
&\qquad = 
\Biggl(\eta(x) \cdot \Bigl( B_{n-1, p, j+k-1} - \frac{1 - p B_{n-1, p, j+k-1}}{1-p} \Bigr) + \frac{1 - p B_{n-1, p, j+k-1}}{1-p}\Biggr) \PP(X=x)\,.
\end{align*}
Note this expression is independent of $i$.

\ \\
\textbf{Putting it all together:} Combining our work in the above cases yields
\begin{align}
&{\E}_{\mu}\Bigl[\ind \{ A_j[i]=x\} \cdot \ind\{ j \le M-k \} \Bigl] \notag \\
&\qquad = 
\pi[i] B_{n, p, j+k} \cdot \PPsub_1(X=x) \notag \\
&\qquad\qquad +
\sum_{r \neq i} \pi[r] \Biggl(\eta(x) \cdot \Bigl( B_{n-1, p, j+k-1} - \frac{1 - p B_{n-1, p, j+k-1}}{1-p} \Bigr) \notag \\
&\qquad \qquad \qquad \qquad \qquad \qquad + \frac{1 - p B_{n-1, p, j+k-1}}{1-p}\Biggr) \PP(X=x) \notag \\
&\qquad = 
\PP(X=x) \Bigl( \eta(x)\, \beta_1(j,i) + \beta_0(j,i) \Bigl)\,.\label{eq:unbiased_pf_mixture}
\end{align}
From here, for an arbitrary (measurable) function $g\,:\cX \rightarrow \R$, we obtain
\begin{align*}
{\E}_{\mu}\Bigl[g(A_j[i]) \cdot \ind\{j \le M-k\} \Bigl] 
&= 
\int g(x)\, \PP (X=x )\, \Bigl( \eta(x)\, \beta_1(j, i) + \beta_0(j, i) \Bigr) dx \\
&= 
\beta_1(j, i)\, \int \PP(X=x)\, g(x)\, \eta(x)\, dx + \beta_0(j, i)\, \E\Bigl[ g(X) \Bigr]\,.
\end{align*}
Now observe that
\begin{align*}
\E\Bigl[ Y g(X) \Bigr] 
&= \E\Bigl[ \E\Bigl[ Y g(X)\, |\, X \Bigr] \Bigr] \\
&= \E \Bigl[ g(X) \E \Bigl[ Y\, |\, X \Bigr] \Bigr]\\
&= \E\Bigl[ g(X) \eta(X) \Bigr] = \int \PP(X=x)\, g(x)\, \eta(x)\, dx\,. 
\end{align*}
Thus
\begin{equation}\label{eq:temporary}
{\E}_{\mu}\Bigl[g(A_j[i]) \cdot \ind\{j \le M-k\} \Bigl] = \beta_1(j, i)\, \E\Bigl[ Y g(X) \Bigr] + \beta_0(j, i)\, \E\Bigl[ g(X) \Bigr]\,,
\end{equation}
yielding
\[ 
\E\Bigl[Y g(X)\Bigl] = \frac{1}{\beta_1(j,i)}{\E}_{\mu}\Bigl[g(A_j[i]) \cdot \ind\{j \le M-k\} \Bigl]  - \frac{\beta_0(j,i)}{\beta_1(j,i)}   \E\Bigl[g(X)\Bigl]\,.
\]
Hence letting $g = f_2 \circ h$, we obtain
\begin{align*}
\E\Bigl[\ell(h(X), Y)\Bigl] &= \E\Bigl[f_1 ( h(X) )\Bigl] + \E\Bigl[Y f_2 ( h(X) ) \Bigl] \\
&= \E\Bigl[f_1 ( h(X) )\Bigl] - \frac{\beta_0(j, i)}{\beta_1(j, i)} \E\Bigl[f_2 ( h(X) )\Bigl] + \frac{1}{\beta_1(j, i)} {\E}_{\mu}\Bigl[g(A_j[i]) \cdot \ind\{j \le M-k\} \Bigl]\,,
\end{align*}
as desired.
\end{proof}


\subsection{Ancillary Results}\label{sa:P}

Recall $B_{n, p, j+k}$ and $B_{n-1, p, j+k-1}$ from the definitions of $\beta_0(j,i)$ and $\beta_1(j,i)$ in Theorem \ref{prop:unbiasedestimatoroneji}. Standard Chernoff bounds guarantee that in regimes of interest, both quantities are very close to 1.
\begin{lemma}\label{lem:pnpjboundvsp}
For all $n \geq 2$, $p \in (0,1]$, all $j$ such that $1\leq j\leq 
np/2-k$, with $k \leq np/2-1$ the quantities 
\[
B_{n, p, j+k}\,,
\qquad B_{n-1, p, j+k-1}
\]
are at least
\[
1 - e^{ - \Omega(np)}\,,
\]
the big-$\Omega$ notation to be interpreted ``as $n$ grows large''.
\end{lemma}
\begin{proof}
Simply observe that 
\begin{align*}
B_{n, p, j+k}
&= 
\PP(X_{n,p} \geq j+k)\\
B_{n-1, p, j+k-1}
&= 
\PP(X_{n-1,p} \geq j+k-1)\,,
\end{align*}
where $X_{n,p}$ is a binomial random variable with parameters $n$ and $p$. Then, if $j+k \leq np/2$,
\begin{align*}
\PP(X_{n,p} \geq j+k) 
&=
1-\PP(X_{n,p} < j+k)\\
&\geq
1-\PP(X_{n,p} \leq j+k)\\
&\geq
1-\PP(X_{n,p} \leq np/2)\\
&\geq
1-e^{-np/8}\,.
\end{align*}
The same argument, with the same conditions on $j,k,n,p$, holds for $X_{n-1,p}$. This concludes the proof.
\end{proof}

Before proceeding, we use a consequence of the above Lemma \ref{lem:pnpjboundvsp} to approximate $\beta_1(j, i)$ and $\beta_0(j, i)$ occurring in the statement of Theorem \ref{prop:unbiasedestimatoroneji}.

\begin{lemma}\label{lem:beta01approxsamplecomplexity}
For all $n \geq 2$, $p \in (0,1/2]$, all $j$ such that $1\leq j\leq 
np/2-k$ with $k \leq np/2-1$, we have 
\begin{align*}
\Big|\beta_1(j, i) - \frac{\pi[i]}p \Big| &= O\left(\frac{e^{-\Omega(np)}}{p}\right)\,,\qquad\Big|\beta_0(j, i) - \big(1-\pi[i]\big) \Big| = O(p e^{-\Omega(np)})\,,
\end{align*}
where $\Omega(\cdot), O(\cdot)$ hides a universal constant. Thus when
$$
np = \Omega\Biggl( \log\left( \max_{i \in [k]}\frac{1}{\pi[i]} \right)  \Biggr)\,,
$$
for a suitable large enough constant hidden in the big-omega notation, we have
\[ 
\beta_1(j, i) \ge \frac{\pi[i]}{2p}\,,\,0 \le \beta_0(j, i) \le 1\,.
\]
\end{lemma}
\begin{proof}
Let $\delta_1 = 1 - B_{n, p, j+k} \ge 0$ and $\delta_2 = 1 - B_{n-1, p, j+k-1} \ge 0$. By Lemma \ref{lem:pnpjboundvsp}, we have $0 \le \delta_1, \delta_2 \le e^{-\Omega(np)}$. Therefore, we can rewrite
\begin{align*}
\beta_1(j, i) &= \frac{\pi[i]}{p}(1-\delta_1)+\Big(1-\delta_2 - \frac{1-p(1-\delta_2)}{1-p}\Big)(1-\pi[i]) \\
&= \frac{\pi[i]}p - \frac{\pi[i] \delta_1}p - \delta_2\Big(1+\frac{p}{1-p}\Big)(1-\pi[i])\,.
\end{align*}
Since $p \le \frac12$, using that $0 \le \delta_1, \delta_2 \le e^{-\Omega(np)}$ the upper bound on $\Big|\beta_1(j, i) - \frac{\pi[i]}p \Big|$ follows. Similarly we can rewrite
\begin{align*}
\beta_0(j, i) &= \frac{1-p(1-\delta_2)}{1-p}(1-\pi[i]) = 1-\pi[i] + \frac{p\delta_2}{1-p}(1-\pi[i])\,,
\end{align*}
and we use $p \le \frac12$ and $0 \le \delta_2 \le e^{-\Omega(np)}$. The final conclusion on the bounds $\beta_1(j, i) \ge \frac{\pi[i]}{2p}$, $0 \le \beta_0(j, i) \le 2$ is evident given the second condition on $np$.
\end{proof}

\section{Proofs for Section \ref{sec:refined}}\label{sa:proofs3}
This section contains the proofs that apply to the ERM algorithm defined in Section \ref{sec:refined}.

Introduce the shorthand 
\[
\widehat \ell_M(h) = \widehat \ell_M (h, S, \cA)\,.
\]
Then consider 
where $\widehat\ell(h, j) $ is as in  (\ref{eq:all_set_estimator}). 
We rewrite
\[
\widehat \ell_M (h) = \frac{1}{\big(\mlen\big) \,\Sigma}\,\sum_{j=k}^{\mupper} 
\Biggl( 
\sum_{i=1}^k r(j,i)\,f_2(h(A_j[i]))
\Biggl) + C\,,
\]
where we have introduced the short-hand notation $r(j,i) = \frac{\pi[i]^2}{\beta_1(j,i)}$, and
\begin{align*}
C &= 
- \frac{1}{\big(\mlen\big)\Sigma}\left(\sum_{j=k}^{\mupper}\frac{1}{B_{n, p, j+k}}\sum_{i=1}^k \frac{\pi[i]^2\,\beta_0(j,i)}{\beta_1(j,i) }\right)\,\E[f_2 (h(x))] \\
&\qquad \qquad + \frac{(\mupper-k+1)\,\E[f_1 (h(x))]}{(\frac{np}{2}-2k+1)\,B_{n, p, j+k}}\,.    
\end{align*}
We note that
$\widehat \ell_M (h, S, \cA)$ can always be reformulated in terms of the original variables $X_1,\ldots, X_n$, as specified next. 
Specifically, denote by $j(i) \in[k]$ the position of variable $X_i$ within attribution set $A_j$ if $X_i \in A_j$, and 0 otherwise.\footnote{This is clearly well-defined because a given $X_i$ is only in one position for a given attribution set.} Then we have
\begin{equation}\label{eq:all_set_estimator}
\widehat \ell_M (h) = \frac{1}{\big(\mlen\big)\,\Sigma}\, \sum_{i=1}^n f_2(h(X_i))\,\sum_{j\,:\,k \le j \le \mupper\,,\, X_i \in A_j} r(j,j(i)) 
+
C
\,.
\end{equation}
Note that in the above interpretation $|\left\{j\,:\,k \le j \le \mupper \,,\, X_i \in A_j\right\}|$ is the number of occurrences of variable $X_i$ in  estimator $\widehat \ell_M (h, S, \cA)$. 
%
We denote by $\muppers$ the realization of $\mupper$ determined by $S$.

The first observation is that for any $j$, $\ind\{X_i \in A_j\}$ equals 1 for exactly $k$ distinct $X_i$. Also, recall the $j(i)$ must all be distinct for different $i$. This allows us to prove the following upper bound which is independent of $k$, $M$ and $n$:
\begin{lemma}\label{lem:allattributionsetsupperboundobservation}
Consider $np = \Omega\Biggl( \log\left( \max_{i \in [k]}\frac{1}{\pi[i]} \right)  \Biggr)$. Then for any realization of $S$, any $j$ such that $k \le j \le \muppers$, and any $h$ we have
\begin{equation}
\Bigl|~ \sum_{i=1}^n f_2(h(X_i))~ r(j, j(i)) ~\ind\{X_i \in A_j\} ~\Bigr| \le 2pF_2\,. \label{eq:this_one} 
\end{equation}
\end{lemma}
\begin{proof}
Since $\ind\{X_i \in A_j\}$ equals 1 for exactly $k$ items $X_i$, irrespective of the realization of $\pi$, the sum 
$$
\sum_{i=1}^n f_2(h(X_i))\, r(j, j(i))\, \ind\{X_i \in A_j\}
$$ 
is only over $k$ distinct data items $X_i$. Also, note the $j(i)$ must all be distinct for different $i$. 
Furthermore Lemma \ref{lem:beta01approxsamplecomplexity} insures that for $j$ such that $k \le j \le \muppers \le \frac{np}2-k$, we have
\[ 
r(j, i) \le 2p \pi[i]\,.
\]
Hence, as the terms $r(j, j(i))~ \ind\{X_i \in A_j\}$ are all non-negative, this yields
\begin{align*}
\Bigl|~ \sum_{i=1}^n f_2(h(X_i))\, r(j, j(i))\,\ind\{X_i \in A_j\} ~\Bigr| &\le \sum_{i=1}^n \big| f_2(h(X_i)) \big| \cdot r(j, j(i))\, \ind\{X_i \in A_j\} \\
&\le F_2 \sum_{i'=1}^k r(j, j(i')) \\
&\leq 2pF_2 \sum_{i'=1}^k \pi[i'] =2pF_2\,,
\end{align*}
as desired.
\end{proof}

It is now convenient to set up some notation. By (\ref{eq:all_set_estimator}), we can write 
\[ 
\widehat \ell_M(h) = \sum_{i=1}^n f_2(h(x_i)) \, \frac{1}{\big(\mlen\big)\,\Sigma} \sum_{j\,:\,k \le j \le \mupper \,,\, X_i \in A_j} r(j,j(i)) +C\,. 
\]
Thus
\[ 
{\E}_{\mu|S}[ \widehat \ell_M(h) ] 
= 
\sum_{i=1}^n f_2(h(x_i)) \, {\E}_{\mu|S} \left[ \frac{1}{\big(\mlen\big)\,\Sigma} \sum_{j\,:\,k \le j \le \muppers \,,\, X_i \in A_j} r(j,j(i)) \right] + C\,.
\]
Moreover by Lemma \ref{prop:unbiasedestimatoroneji}, we have $\popl(h) = {\E}_{S}\Bigl[ {\E}_{\mu|S}[\widehat\ell] \Bigr]$. Note that ${\E}_{\mu|S}[ \widehat \ell_M(h) ]$ is a function of the i.i.d. random variables $\{ Z_i\}_{i=1}^n$ where each $Z_i = (X_{i}, Y_{i})$. Thus we can define the function of  $S = \langle z_1,\ldots,z_n \rangle = \langle (x_1,y_1)),\ldots, (x_n,y_n) \rangle$
\[ 
f(z_1, \ldots, z_n) = f(z_1, \ldots, z_n; h) := {\E}_{\mu|S}[ \widehat \ell_M(h) ]\,.
\]
We need to upper bound the sensitivity of $f(z_1, \ldots, z_n)$ to each $z_i$. 
In particular, we now establish the following.

\begin{lemma}\label{lem:allattrsetconditionalexpectationconcentration_sensitivity}
Suppose  $np = \Omega\left( \log\left( \max_{i \in [k]}\frac{1}{\pi[i]} \right)  \right)$. 
Consider any $h \in \hyps$. Define $\mathcal{B}$ as the set of $\langle z_1, \ldots, z_n\rangle$ such that $\frac{np}{2}  \le M$. Consider any $z = \langle z_1, \ldots, z_{i-1}, z_i, z_{i+1}, \ldots z_n\rangle$, $z' = \langle z_1, \ldots, z_{i-1}, z'_i, z_{i+1}, \ldots z_n\rangle$, where $z$ and $z'$ only differ in their $i$-th coordinate. Then for all $z, z' \in \mathcal{B}
$ we have 
\[ 
\Bigl| f(z) - f(z') \Bigr| \le \frac{12 p F_2}{(\mlens) \Sigma}\,. 
\]
In particular, 
if $k \le \frac{np}8$, then for all $z, z' \in \mathcal{B}
$ we have 
\[ 
\Bigl| f(z) - f(z') \Bigr| \le \frac{48 F_2}{n \Sigma}\,. 
\]
\end{lemma}

\begin{proof}
Recall 
\begin{align*}
f(Z_1, \ldots, Z_n) &= {\E}_{\mu | S}\Bigl[\widehat \ell_M(h)\Bigr] \\
&= 
\frac{1}{\big(\mlen\big)\,\Sigma} \sum_{i=1}^n f_2(h(X_i))\,{\E}_{\mu | S}\left[\sum_{j\,:\, k \le j \le \mupper \,,\, X_i \in A_j} r(j,j(i)) \right]   
+
C\,,
\end{align*}
is a function of i.i.d. random variables $Z_i = (X_i,Y_i)$, $i = 1, \ldots, n$.
Consider any $S$. For fixed $i \in [n]$ and $h$, we consider the sensitivity of $f$ as we change $Z_i = z_i = (x_i, y_i)$ to $Z'_i = z'_i = (x'_i,y'_i)$.

First, since $z,z' \in \mathcal{B}$, we have $\muppers = \frac{np}{2}-k$, so that $C$ turns to the constant
$$
C = 
- \frac{1}{\big(\mlen\big)\Sigma}\left(\sum_{j=k}^{\muppers}\frac{1}{B_{n, p, j+k}}\sum_{i=1}^k \frac{\pi[i]^2\,\beta_0(j,i)}{\beta_1(j,i) }\right)\,\E[f_2 (h(x))]  + \frac{\E[f_1 (h(x))]}{B_{n, p, j+k}}\,,
$$
thereby not contributing any sensitivity.

Consider any $j$ such that $k \le j \le \muppers$. Let us focus on random variable $j(i)$ in the conditional measure $\mu|S$ where $S$ is given. First, the number of distinct attribution sets that may contain $x_i$ is upper bounded by $2k-1$. This is because there are at most $2k-1$ conversions at distance $\leq k-1$ from $x_i$, corresponding to the labels $y_{i-(k-1)}, y_{i-(k-2)}, \ldots, y_{i-1}, y_i, y_{i+1},\ldots, y_{i+k-2}, y_{i+k-1}$ (if the indices are in bounds), which may or may not be one. Now, suppose $y_{i+b} = 1$ is the $j$th conversion, for some $0\leq b \leq k-1$. Then
\[
j(i) =
\begin{cases}
k-b &{\mbox{with prob. $\pi[k]$}}\\
k-b-1 &{\mbox{with prob. $\pi[k-1]$}}\\
\vdots     &\vdots\\
1 &{\mbox{with prob. $\pi[b+1]$}}\\
0      &{\mbox{with remaining prob.}}
\end{cases}
\]

so that
\[
{\E}_{\mu | S}\Bigl[r(j,j(i)) \Bigr] = \sum_{i'=b+1}^k \pi[i']\, r(j,i'-b)\,.
\]

Similarly, suppose $y_{i-b} = 1$ is the $j$th conversion, for some $1\leq b \leq k-1$. Then
\[
j(i) =
\begin{cases}
b+1 &{\mbox{with prob. $\pi[1]$}}\\
b+2 &{\mbox{with prob. $\pi[2]$}}\\
\vdots     &\vdots\\
k &{\mbox{with prob. $\pi[k-b]$}}\\
0      &{\mbox{with remaining prob.}}
\end{cases}
\]
with expectation
\[
{\E}_{\mu | S}\Bigl[r(j,j(i) \Bigr] = \sum_{i'=1}^{k-b} \pi[i']\, r(j,i'+b)\,.
\]
Note these computations only apply for $j$ such that $k \le j \le \muppers$.

Denote by
\[
0 \leq b^+(1) < b^+(2) < \ldots < b^+(r^+)\,, 
\qquad 
0 > -b^-(1) > -b^-(2) > \ldots > -b^-(r^-)
\]
the positive and negative offsets of the conversions at distance at most $k-1$ from $i$ within $S$, whose corresponding attribution set's index $j$ is such that $k \le j \le \muppers$. 
Hence both the $b^+(t)$ and the $b^-(t)$ are non-negative. Let the index $j$ of the conversion corresponding to each $b^+(t)$ be $j_{b^+(t)}$ for all $1 \le t \le r^+$, and similarly define $j_{b^-(t)}$ for all $1 \le t \le r^-$. Thus we have for all $1 \le t \le r^+$ that $k \le j_{b^+(t)} \le \muppers$, and similarly for all $1 \le t \le r^-$ that $k \le j_{b^-(t)} \le \muppers$. Hence we can write
\begin{align*}
&{\E}_{\mu | S}\left[\sum_{j\,:\, k \le j \le \muppers \,,\, x_i \in A_j} r(j,j(i))\right] \\
& \qquad = 
\sum_{i'=b^+(1)+1}^k \pi[i']\, r\Big(j_{b^+(1)},\,i'-b^+(1)\Big) + \ldots +
\sum_{i'=b^+(r^+)+1}^k \pi[i']\, r\Big(j_{b^+(r^+)},\,i'-b^+(r^+)\Big)\\
&\qquad \qquad+
\sum_{i'=1}^{k-b^-(1)} \pi[i']\, r\Big(j_{b^-(1)},\,i' + b^-(1)\Big) + \ldots +
\sum_{i'=1}^{k-b^-(r^-)} \pi[i']\, r\Big(j_{b^-(r^-)},\,i' + b^-(r^-)\Big)
\,.
\end{align*}

Note the change $x_i \rightarrow x'_i$ does not impact the structure of the attribution sets $A_j$.
Thus, the contribution from the change $x_i \rightarrow x'_i$ is upper bounded by
\[
\frac{2 F_2}{(\mlens)\,\Sigma}  {\E}_{\mu | S}\left[\sum_{j\,:\, k \le j \le \muppers \,,\, x_i \in A_j} r(j,j(i)) \right]\,. 
\]
Recall that for all $1 \le t \le r^+$ we have that $k \le j_{b^+(t)} \le \muppers$, and similarly for all $1 \le t \le r^-$ we have that $k \le j_{b^-(t)} \le \muppers$. By our condition on $n$ and as $k \le j_{b^+(t)}, j_{b^-(t)} \le 
\muppers \le \frac{np}2-k$, Lemma \ref{lem:beta01approxsamplecomplexity} gives $\beta_1(j_{b^+(t)},i' - b^+(t)) \geq \frac{\pi[i' - b^+(t)]}{2p}$ and hence $r(j_{b^+(t)},i' - b^+(t)) \leq 2p \cdot \pi[i' - b^+(t)]$. Similarly we have $r(j_{b^-(t)}, i' + b^-(t)) \leq 2p \cdot \pi[i' + b^-(t)]$. Therefore as the entries of $\pi$ are all non-negative, we have 
\begin{align}
&{\E}_{\mu | S}\left[\sum_{j\,:\, k \le j \le \muppers \,,\, x_i \in A_j} r(j,j(i))\right] \notag \\
& \qquad = \sum_{t=1}^{r^+} \sum_{i'=b^+(t)+1}^k \pi[i']\, r\Big(j_{b^+(t)},\,i'-b^+(t)\Big) +
\sum_{t=1}^{r^-} \sum_{i'=1}^{k-b^-(t)} \pi[i']\, r\Big(j_{b^-(t)},\, i' + b^-(t)\Big) \notag \\
& \qquad \leq 
2p \sum_{t=1}^{r^+} \sum_{i'=b^+(t)+1}^k \pi[i'] \, \pi[i'-b^+(t)] +
2p \sum_{t=1}^{r^-} \sum_{i'=1}^{k-b^-(t)} \pi[i']\, \pi[i'+b^-(t)] \notag \\
&\qquad \leq 2p \sum_{b=0}^{k-1} \sum_{i'=b+1}^k \pi[i'] \pi[i'-b] + 2p \sum_{b=1}^{k-1} \sum_{i'=1}^{k-b} \pi[i'] \pi[i'+b]\, \notag \\
&\qquad \leq
2p\,\sum_{i'=1}^{k} \pi[i']\,\sum_{b=1}^{k}  \pi[b] + 
2p\,\sum_{i'=1}^{k-1} \pi[i]\,\sum_{b=1}^{k}  \pi[b]\notag\\
&\qquad =
2p\,\sum_{i'=1}^{k} \pi[i'] + 
2p\,\sum_{i'=1}^{k-1} \pi[i']\notag\\
&\qquad \leq 
4p\,.\label{eq:this_one3}
\end{align}
On the other hand, the contribution from the change $y_i \rightarrow y'_i$ amounts to either adding (if $y_i = 0$ and $y'_i = 1$) or subtracting (if $y_i = 1$ and $y'_i = 0$) an attribution set $A_j$ (the one 
associated with
$y_i$). Therefore this contribution amounts to, respectively, adding or subtracting a term $j$ from the sum $\sum_{j\,:\, k \le j \le \muppers \,,\, x_i \in A_j}$ within ${\E}_{\mu | S}\left[\sum_{j\,:\, k \le j \le \muppers \,,\, x_i \in A_j} r(j,j(i)) \right]$ for each data index $i$. 
Moreover, this does not change the other attribution sets, which are generated independently for each conversion. Note if $y'_i = y_i$ then we do not need to take this (non-negative) change into account, so the following bound will hold irrespective of whether $y'_i = y_i$ or $y'_i \neq y_i$. 

The idea is that even for this new $j$, as $k \le j \le \muppers$, by Lemma \ref{lem:allattributionsetsupperboundobservation} we have
\begin{align*}
\Bigl|\, \sum_{i=1}^n f_2(h(x_i))\, r(j,j(i))\, \ind\{x_i \in A_j\}\, \Bigr|
&\le 
F_2 ~\Bigl|~ \sum_{i=1}^n r(j,j(i))~ \ind\{x_i \in A_j\} ~\Bigr| \\
&\le 
F_2 ~\Bigl| ~\sum_{b=1}^k r(j,b)~\Bigr|\\
&\leq 
2p\,F_2\,,
\end{align*}
allowing us to control the sensitivity.

To see this, denote this new $j$ by $j_0$ and suppose that $y_i = 0, y'_i =1$ (thus the term corresponding to $j_0$ is {\em added}). Since we are considering $j$ in the range $[k, \muppers]$, we have $k \le j_0 \le \muppers$. A very similar argument holds in the case where $y_i = 1, y'_i =0$ (in which case the term corresponding to $j_0$ is {\em subtracted}). The above derivations allow us to conclude that the sensitivity of $f$ when turning $z_i$ into $z'_i$ is upper bounded by 
\begin{align*}
&\Biggl|\,\frac{1}{(\mlens) \,\Sigma} \sum_{s=1\,s\neq i}^n f_2(h(x_s))\,\Biggl({\E}_{\mu | S}\left[\sum_{j\,:\, k \le j \le \muppers \,,\, x_s \in A_j} r(j,j(s)) \right]  + {\E}_{\mu | S}\Bigl[\ind\{x_s \in A_{j_0}\} r(j_0,j_0(s)) \Bigr]\Biggl)\\ 
&\qquad+
\frac{1}{(\mlens)\Sigma}\,f_2(h(x'_i))\,\Biggl({\E}_{\mu | S}\left[\sum_{j\,:\, k \le j \le \muppers \,,\, x_i \in A_j} r(j,j(i)) \right]  + {\E}_{\mu | S}\Bigl[\ind\{x'_i \in A_{j_0}\} r(j_0,j_0(i)) \Bigr]\Biggl)\\
&\qquad-
\frac{1}{(\mlens)\,\Sigma} \sum_{s=1\,s\neq i}^n f_2(h(x_s))\,{\E}_{\mu | S}\Biggl[\sum_{j\,:\, k \le j \le \muppers \,,\, x_s \in A_j} r(j,j(s)) \Biggr] \\
&\qquad -
\frac{1}{(\mlens)\,\Sigma}\,f_2(h(x_i))\,{\E}_{\mu | S}\left[\sum_{j\,:\, k \le j \le \muppers \,,\, x_i \in A_j} r(j,j(i)) \right] \,\Biggl|\\
&=
\Biggl|\,\frac{1}{(\mlens)\,\Sigma} \sum_{s=1\,s\neq i}^n f_2(h(x_s))\, {\E}_{\mu | S}\Bigl[\ind\{x_s \in A_{j_0}\} r(j_0,j_0(s)) \Bigr]\\ 
&\qquad+
\frac{1}{\big(\mlen\big)\,\Sigma}\,f_2(h(x'_i))\,\Biggl({\E}_{\mu | S}\left[\sum_{j\,:\, k \le j \le \muppers \,,\, x_i \in A_j} r(j,j(i)) \right]  + {\E}_{\mu | S}\Bigl[\ind\{x_i \in A_{j_0}\} r(j_0,j_0(i)) \Bigr]\Biggl)\\
&\qquad-
\frac{1}{(\mlens)\,\Sigma}\,f_2(h(x_i))\,{\E}_{\mu | S}\left[\sum_{j\,:\, k \le j \le \muppers \,,\, x_i \in A_j} r(j,j(i)) \right] \,\Biggl|\,.
\end{align*}
In turn the above is upper bounded by
\begin{align*}
&\underbrace{\frac{1}{(\mlens)\,\Sigma}\,\Biggl|\, \sum_{s=1\,s\neq i}^n f_2(h(x_s))\, {\E}_{\mu | S}\Bigl[\ind\{x_s \in A_{j_0}\} r(j_0,j_0(s)) \Bigr]\,\Biggl|}_{\expressionI}\\ 
&\qquad+
\underbrace{\frac{1}{(\mlens)\,\Sigma}\,\Biggl|\,  f_2(h(x'_i)) - f_2(h(x_i))\Bigl|\, \cdot\,\Biggl|\,{\E}_{\mu | S}\left[\sum_{j\,:\, k \le j \le \muppers \,,\, x_i \in A_j} r(j,j(i)) \right] \Biggl|}_{\expressionII} \\
&\qquad+
\underbrace{\frac{1}{(\mlens)\,\Sigma}\,\Biggl|f_2(h(x'_i)) {\E}_{\mu | S}\Bigl[\ind\{x'_i \in A_{j_0}\} r(j_0,j_0(i)) \Bigr] \Biggl|\,.}_{\expressionIII}    
\end{align*}

Now, note as the terms $\ind\{x_s \in A_{j_0}\} r(j_0,j_0(s))$ are non-negative, we have by the same rationale as in the proof of Lemma \ref{lem:allattributionsetsupperboundobservation} that
\begin{align*}
\expressionI &= \frac{1}{(\mlens)\,\Sigma}\,\Biggl|\, {\E}_{\mu | S}\left[ \sum_{s=1\,s\neq i}^n f_2(h(x_s))\, \ind\{x_s \in A_{j_0}\} r(j_0,j_0(s)) \right]\,\Biggl|\, \\
&\le \frac1{(\mlens)\, \Sigma} {\E}_{\mu | S} \Biggl[ \sum_{s=1\,s \neq i}^n \Bigl| f_2(h(x_s)) \Bigr| \cdot \ind\{x_s \in A_{j_0}\} r(j_0,j_0(s)) \Biggr] \\
&\le \frac{F_2}{(\mlens)\Sigma} {\E}_{\mu | S} \Biggl[ \sum_{s=1}^n  \ind\{x_s \in A_{j_0}\} r(j_0,j_0(s)) \Biggr]\\
&\leq \frac{2p\,F_2}{(\mlens)\,\Sigma}\,.
\end{align*}
In particular, this is because $\ind\{x_s \in A_{j_0}\}=1$ for at most $k$ indices $s$, and for these $s$, the indices $j_0(s)$ must all be distinct.

Moreover, because of (\ref{eq:this_one3}),
\[
\expressionII \leq \frac{8p\,F_2}{(\mlens)\,\Sigma}\,.
\]
As for $\expressionIII$, we leverage again
our condition on $n$ and that $k \le j_0 \le \muppers$. Combining with Lemma \ref{lem:beta01approxsamplecomplexity} this gives us the upper bound 
$r(j_0, j_0(i)) \le 2p \pi[i] \leq 2p$, so that
\[
\expressionIII \leq \frac{2p\,F_2}{(\mlens) \Sigma}\,.
\]
Putting these bounds together, this yields the sensitivity bound
\[
\frac{12p\,F_2}{(\mlens)\,\Sigma}\,,
\]
holding for all $S$ and $h \in \hyps$. This completes the proof.
\end{proof}

Now observe that the sensitivity bound from Lemma \ref{lem:allattrsetconditionalexpectationconcentration_sensitivity} depends on the condition $\frac{np}{2} \leq M$, which involves random variable $M$.
Thus we use the following variant of McDiarmid's inequality from \cite{combes2024extension}:
\begin{theorem}[Theorem 3 and Example 3 of \cite{combes2024extension}]\label{prop:mcdiarmids_whp}
Consider a generic metric space $\cX$. Let $g(z_1, \ldots, z_n)$ be a function of $n$ i.i.d. random variables $z_1, \ldots, z_n$ with $z_i \in \cX$, that satisfies the following property. For some subset $\mathcal{B} \subseteq \cX^n$ and any pair of $z = (z_1, \ldots, z_{i-1}, z_i, z_{i+1}, \ldots z_n)$ and $z' = (z_1, \ldots, z_{i-1}, z'_i, z_{i+1}, \ldots z_n)$ that only differ in their $i$-th coordinate with $z, z' \in \mathcal{B}$, we have $|g(z) - g(z')| \le c_i$. Let $\rho = \PP(\mathcal{B}^c)$ and $\overline{c} = \sum_{i=1}^n c_i$. Then
\[ 
\PP \Bigl( g(z) - \E\bigl[ g(z)\, \bigl|\, z \in \mathcal{B}\bigl] \ge t + \rho \overline{c} \Bigr) \le \rho + \exp\left( -\frac{2t^2}{\sum_{i=1}^n c_i^2} \right)\,. 
\]
\end{theorem}

Before moving forward, we need to introduce a simple lemma which is a direct consequence of the contraction lemma for Rademacher Complexity (e.g., \cite{lt11}).
\begin{lemma}\label{lem:multiplyconstantrademachercomplexity}
For a hypothesis space $\hyps$ and any fixed sample $z=\{z_1, \ldots, z_n\}$, for any $a_1, \ldots, a_n \in \R$ independent of $\sigma_i$ but potentially depending on the samples $z_i$, letting $\sigma_i$ be i.i.d. Rademacher random variables we have
\begin{align*}
{\E}_{\sigma} \left[\sup_{h \in \hyps}\,\Big|\sum_{i=1}^n \sigma_i a_i h(z_i) \Big| \,\,\Biggl|\,z \right] 
\le 
\Big( \max_i \big| a_i \big| \Big) {\E}_{\sigma} \left[\sup_{h \in \hyps}\,\Big|\sum_{i=1}^n \sigma_i  h(z_i)\Big| \,\,\Biggl|\,z \right]\,.
\end{align*}
\end{lemma}
\begin{proof}
For each $1 \le i \le n$, consider the function $\phi_i(x) = a_i x$. Thus $a_i h(z_i) = \phi_i\big(h(z_i)\big)$. Each $\phi_i$ is $\big|a_i\big|$ Lipschitz, and in particular $\max_i \big| a_i \big| $-Lipschitz. Thus the Lemma follows by the Contraction Lemma of Rademacher Complexity. 
Note the functions used in the Contraction Lemma may depend on the sample $z$, since the entire proof of the Contraction Lemma is done with the sample $z$ fixed.
\end{proof}

We are now in position to prove Theorem \ref{thm:allattributionsetssamplecomplexity}. \\

\begin{proof}[Proof of Theorem \ref{thm:allattributionsetssamplecomplexity}]
%
%
From standard arguments, and the fact that, from Theorem \ref{prop:unbiasedestimatoroneji},
\(
{\E}_\mu[\widehat \ell_M(h)] = \popl(h)
\)
we can write, for any $\epsilon \geq 0$,
\begin{align*}
\ind\Bigl\{\regret(\wh) \geq \epsilon\Bigl\}
&\leq
\ind\Biggl\{ 2\,\sup_{ h \in \hyps} \Bigl| \widehat \ell_M(h) - \popl(h) \Bigl| \geq \epsilon
\Biggl\} 
\\
&\leq 
\ind\Bigl\{\sup_{ h \in \hyps}\, \Bigl| {\E}_{\mu|S}[ \widehat \ell_M(h) ] - \popl(h) \Bigl| \geq \epsilon/4 \Bigl\} 
+  
\ind\Bigl\{\sup_{ h \in \hyps}\,\Bigl| \widehat \ell_M(h) - {\E}_{\mu|S}[ \widehat \ell_M(h) ] \Bigl| \geq \epsilon/4 \Bigl\}\,.
\end{align*}
We take expectations w.r.t. $\mu$ and consider in turn
\begin{equation}\label{e:term1}
{\E}_\mu\left[\ind\Bigl\{\sup_{ h \in \hyps}\, \Bigl| {\E}_{\mu|S}[ \widehat \ell_M(h) ] - \popl(h) \Bigl| \geq \epsilon/4 \Bigl\}\right] 
= 
{\PP}_\mu\left(\sup_{ h \in \hyps}\, \Bigl| {\E}_{\mu|S}[ \widehat \ell_M(h) ] - \popl(h) \Bigl| \geq \epsilon/4 \right)
\,,
\end{equation}
and
\begin{equation}\label{e:term2}
{\E}_\mu\left[\ind\Bigl\{\sup_{ h \in \hyps}\,\Bigl| \widehat \ell_M(h) - {\E}_{\mu|S}[ \widehat \ell_M(h) ] \Bigl| \geq \epsilon/4 \Bigl\}\right]
=
{\PP}_\mu\left(\sup_{ h \in \hyps}\,\Bigl| \widehat \ell_M(h) - {\E}_{\mu|S}[ \widehat \ell_M(h) ] \Bigl| \geq \epsilon/4\right)
\,.
\end{equation}
\paragraph{Controlling (\ref{e:term1}):} We deal first with (\ref{e:term1}). Recall that
\begin{align*}
f(z_1, \ldots, z_n) 
&= f(z_1, \ldots, z_n; h ) 
= 
{\E}_{\mu | S}\Bigl[\widehat \ell_M(h)\Bigr] \\
&= 
\frac{1}{(\mlen)\,\Sigma} \sum_{i=1}^n f_2(h(x_i))\,{\E}_{\mu | S}\left[\sum_{j\,:\, k \le j \le \muppers \,,\, x_i \in A_j} r(j,j(i)) \right]   
+
C\,.
\end{align*}
%
The function $f(Z_1, \ldots, Z_n)$
is a function of i.i.d. random variables $Z_i = (X_i,Y_i)$, $i = 1,\ldots,n$. Let $Z = (Z_1,\ldots, Z_n)$. We also know that, for each $h\in \hyps$, 
$$
{\E}_Z[f(Z; h)] = \popl(h)\,.
$$

Focus on the random variable
$
\gamma_i = {\E}_{\mu | S}\left[\sum_{j\,:\, k \le j \le \muppers \,,\, x_i \in A_j} r(j,j(i)) \right]\,.
$
A closer inspection reveals that this random variable does depend on the sample $Z=z$ only through the labels $y_1,\ldots, y_n$. In particular, $\gamma_i$ depends on index $i$, but does not directly depend on $x_i$ (note that $A_j$ is a set of {\em indices}, so the condition ``$x_i \in A_j$'' within the brackets should be interpreted as ``$i \in A_j$, for frozen value of $x_i$''), and certainly it does not depend on $h$. So, let us adopt the notation $y = \langle y_1,\ldots, y_n\rangle$, and
\[
\gamma_i(y) = {\E}_{\mu | S}\left[\sum_{j\,:\, k \le j \le \muppers \,,\, x_i \in A_j} r(j,j(i)) \right]\,,
\]
so that now
\[
f(z; h ) 
= 
\frac{1}{(\mlen)\,\Sigma} \sum_{i=1}^n \gamma_i(y)\,f_2(h(x_i))   
+
C\,.
\]
Let $z^{(i)}$ be the same as $z$, with the $i$-th item replaced by $z_i' = (x_i',y_i')$, and introduce independent Rademacher variables $\sigma = (\sigma_1,\ldots, \sigma_n)$. Define
\[
\Phi(z) = \sup_{h \in \hyps} \Bigl| \popl(h)- f(z;h)\Bigl| \,.
\]
We have 
\[
\Phi(z) - \Phi(z^{(i)}) 
\leq  
\sup_{h \in \hyps} \Bigl| f(z^{(i)};h) - f(z;h)\Bigl| 
\leq 
\frac{12 p F_2}{(\mlen) \Sigma} \le \frac{48 F_2}{n \Sigma}\,,
\]
the second inequality deriving from Lemma \ref{lem:allattrsetconditionalexpectationconcentration_sensitivity}. 

We apply Theorem \ref{prop:mcdiarmids_whp} with $g(z) = \Phi(z)$, $\cB = \Big\{ z\,:\, \frac{np}2 \le M \Big\}$, and $c_i = \frac{48 F_2}{n \Sigma}$ for all $1 \le i \le n$. By Chernoff's bound, $\PP(\cB^c) \le e^{-\Omega(np)}$. 
Theorem \ref{prop:mcdiarmids_whp} now gives 
\begin{align}
\PP \Bigl( \Phi(Z) - \E[ \Phi(Z)\,|\,Z \in \mathcal{B}] \ge t +  \frac{48F_2}{\Sigma} e^{-\Omega(np)} \Bigr) \le e^{-\Omega(np)} + \exp\Bigl( - \frac{2n\Sigma^2}{2304 F_2^2} \cdot t^2 \Bigr)\,.\label{eq:concentrationPhi}
\end{align}
We now control $\E[ \Phi(Z)\,|\,Z \in \mathcal{B}] - \E[\Phi(Z)]$. By the Law of Total Expectation,
\begin{align*}
&\big| \E[ \Phi(Z)\,|\,Z \in \mathcal{B}] - \E[\Phi(Z)] \big| \\
&\qquad = \big| \PP(\cB) \E[\Phi(Z)|Z \in \cB] + \PP(\cB^c) \E[\Phi(Z) | Z \in \cB^c] - \E[\Phi(Z) | Z \in \cB] \big| \\
&\qquad = \PP(\cB^c) \big| \E[\Phi(Z) | Z \in \cB^c] - \E[\Phi(Z)|Z \in \cB] \big| \\
&\qquad \le 2\PP(\cB^c) \sup_{z \in Z} |\Phi(z)|\,.
\end{align*}
By definition of $\Phi(z) = \sup_{h \in \hyps} \Bigl| \popl(h)- f(z;h)\Bigl|$, since the additive term $+C$ in the definition $\popl(h)$ and $f(z;h)$ cancel, we obtain
\begin{align*}
|\Phi(z)| &\le 2\sup_{h \in \hyps} \Bigg| {\E}_{\mu|S}\Bigg[ \frac{1}{(\mlen)\,\Sigma} \sum_{i=1}^n \gamma_i(y)\,f_2(h(x_i)) \Bigg] \Bigg| \\
&= 2\sup_{h \in \hyps} \Bigg|{\E}_{\mu|S} \left[ ~\sum_{j=k}^{\muppers} \frac{1}{\big(\mlen\big) \Sigma} ~ \sum_{i=1}^n f_2(h(x_i))~ r(j, j(i)) ~\ind\{x_i \in A_j\}~ \right]\Bigg|\,,
\end{align*}
where the last step follows from swapping the order of summation.
Thus Lemma \ref{lem:allattributionsetsupperboundobservation} gives $|\Phi(z)| \le \frac{4pF_2}{\Sigma}$.
Hence by Chernoff's bound we obtain from the above that
\begin{equation}\label{eq:allattrsetconditionalexpectationconcentration_eq2}
\big| \E[ \Phi(Z)\,|\,Z \in \mathcal{B}] - \E[\Phi(Z)] \big| \le \frac{8pF_2}{\Sigma} e^{-\Omega(np)}\,.    
\end{equation}
Combining (\ref{eq:concentrationPhi}) and (\ref{eq:allattrsetconditionalexpectationconcentration_eq2}) gives
\begin{align}
\PP \Bigl( \Phi(z) - \E[ \Phi(z) ] \ge t +  \frac{40 F_2}{\Sigma} e^{-\Omega(np)} \Bigr) \le e^{-\Omega(np)} + \exp\left( - \frac{n\Sigma^2}{1152 F_2^2} \cdot t^2\right)\,.\label{eq:bound_on_term1}
\end{align}
Now, let $Z' = (Z'_1, \ldots, Z'_n)$ be an independent sample, with $Z_i' = (X_i',Y_i')$, $i \in [n]$, and set $Y' = \langle Y'_1,\ldots,Y'_n\rangle$. We can write

%
\begin{align*}
\E[\Phi(Z)] 
&= 
{\E}_Z \left[\sup_{h \in \hyps} \Bigl| \popl(h)- f(Z;h)\Bigl|  \right]   \\
&=
{\E}_Z \left[\sup_{h \in \hyps} \Bigl|{\E}_{Z'} [f(Z';h)] - f(Z;h)\Bigl|  \right] \\
&\leq
{\E}_{Z,Z'} \left[\sup_{h \in \hyps} \Bigl| f(Z';h) - f(Z;h)\Bigl|  \right] \\
&=
\frac{1}{(\mlen)\,\Sigma}\,{\E}_{Z,Z'} \left[\sup_{h \in \hyps}\Big|\sum_{i=1}^n  \Bigl(\gamma_i(Y') f_2(h(X'_i)) - \gamma_i(Y) f_2(h(X_i)) \Bigl)\Big|\right]\\
&=
\frac{1}{(\mlen)\,\Sigma}\,{\E}_{Z,\, Z',\sigma} \left[\sup_{h \in \hyps}\Big|\sum_{i=1}^n  \sigma_i\Bigl(\gamma_i(Y') f_2(h(X'_i)) - \gamma_i(Y) f_2(h(X_i)) \Bigl)\Big|\right]\\
&\leq
\frac{2}{(\mlen)\,\Sigma}\,{\E}_{Z,\sigma} \left[\sup_{h \in \hyps}\Big|\sum_{i=1}^n  \sigma_i\gamma_i(Y) f_2(h(X_i))  \Big|\right]\\
&\leq
\frac{2}{(\mlen)\,\Sigma}\,{\E}_{Z} \Biggl[{\E}_{\sigma} \left[ \sup_{h \in \hyps}\Big|\sum_{i=1}^n  \sigma_i\gamma_i(Y) f_2(h(X_i))  \Big| \,\Bigg|\,Z \right] \Biggl]\,. 
\end{align*}
Now, observe that, by (\ref{eq:this_one3}) and under the conditions of Lemma \ref{lem:beta01approxsamplecomplexity}, 
\[
\gamma_i(y) \in [0,4p]\,,
\]
for all $y$, independent of $h$. Thus by the above bound on $\gamma_i(y)$ and Lemma \ref{lem:multiplyconstantrademachercomplexity}, this gives
\[
\E[\Phi(Z)] \leq \frac{8p}{(\mlen)\,\Sigma}\,{\E}_{Z} \Biggl[{\E}_{\sigma} \left[ \sup_{h \in \hyps}\Big|\sum_{i=1}^n  \sigma_i f_2(h(X_i))  \Big| \,\Bigg|\,Z \right] \Biggl]
=
\frac{8pn}{(\mlen)\,\Sigma}\,R_n(f_2 \circ \hyps)
\]
where
\[
R_n(\hyps) = \,{\E}_{X_1,\ldots,X_n} \Biggl[{\E}_{\sigma} \left[\sup_{h \in \hyps}\,\Big|\sum_{i=1}^n \sigma_i \frac{ h(X_i)}{n} \Big| \,\,\Biggl|\,X_1,\ldots,X_n \right] \Biggl] 
\]
is the (average) Rademacher Complexity of function class $\hyps$, and
\[
f_2\circ\hyps = \{ f_2\circ h\,|\, h \in \hyps\}\,.
\]
Since we assumed $k \leq \frac{np}{8}$, the above implies
\[
\E[\Phi(Z)] \leq \frac{32}{\Sigma} R_n(f_2 \circ \hyps)\,.
\]
Finally, the $L$-lipschitzness of $f_2(\cdot)$, along with Talagrand's contraction lemma (e.g., \cite{lt11}), gives
\[
\E[\Phi(Z)] \leq \frac{32L}{\Sigma}\, R_n(\hyps)\,.
\]
We plug back into (\ref{eq:bound_on_term1}) to obtain, as a consequence,
\begin{align}\label{e:bound2_on_term1}
\PP\Big( \Phi(Z) \ge  \underbrace{ \frac{32L\, R_n(\hyps)}{\Sigma} +   t +  \frac{40 F_2}{\Sigma} e^{-\Omega(np)}}_{\epsilon/4}\Big) 
\le 
e^{-\Omega(np)} + \exp\left( - \frac{n\Sigma^2}{1152 F_2^2} \cdot t^2 \right)\,,
\end{align}
the left-hand side being a version of (\ref{e:term1}) once we set 
$$
\epsilon/4 \geq \frac{32L\, R_n(\hyps)}{\Sigma} +   t +  \frac{40 F_2}{\Sigma} e^{-\Omega(np)}\,.
$$

\paragraph{Controlling (\ref{e:term2}):} We now turn to (\ref{e:term2}). 
We define the random variables 
\[ 
V_j = V(h,A_j) = \frac{1}{\big(\mlen\big) \Sigma} \Bigl(~ \sum_{i=1}^n f_2(h(X_i))~ r(j,j(i)) ~\ind\{X_i \in A_j\} \Bigr)\,.
\]
Observe that in the conditional space where $S$ is given, the adversary operates on each individual conversion independently, hence the random variables $\ind{\{X_{i} \in A_j\}}\,\pi[j(i)]$ are always independent of each other. Thus, under the measure $\mu | S$, the $f_2(h(X_i))$ are constants, and the variables $\pi[j(i)] ~\ind\{X_i \in A_j\}$ are independent for different $j$. 
Hence the $V_j$ are independent w.r.t. the measure $\mu | S$. We let $m = m(S)$ is the number of conversions (number of positive labels) in $S$, and let $\muppers = \min\{\frac{np}{2}-k,m-k\}$.

Note that by swapping the order of summation we can rewrite $\widehat \ell_M (h)$ as
\begin{align*}
\widehat \ell_M (h) = \sum_{j=k}^{\muppers} V_j + C\,,\qquad {\E}_{\mu|S}\big[ \widehat\ell_M (h) \big] = \sum_{j=k}^{\muppers} {\E}_{\mu|S}\big[V_j\big] + C\,.
\end{align*}
Furthermore, by Lemma \ref{lem:allattributionsetsupperboundobservation}, for $j$ such that $k \le j \le \min\{\frac{np}{2},m\}-k$, and $S \in \mathcal{B}$, with
\[
\mathcal{B} = 
\left\{S\,:\,\frac{np}{2} \leq M \leq \frac{3np}{2} \right\}
\]
we have 
\begin{equation}\label{e:bound_on_Z}
|V_j| \le \frac{2p F_2}{(\frac{np}2-2k+1) \Sigma} 
\leq
\frac{8 F_2}{n\Sigma} 
\,,
\end{equation}
where we used the condition $k \leq \frac{np}{8}$.
Consider, in the conditional space where $S \in \mathcal{B}$,
\[
{\PP}_{\mu|S}\left(\sup_{ h \in \hyps}\,\Bigl| \widehat \ell_M(h) - {\E}_{\mu|S}[ \widehat \ell_M(h) ] \Bigl| \geq \epsilon/4\right)\,.
\]
In this conditional space, denote by $A = \langle A_k,\ldots,A_{\muppers} \rangle$ the (random) attribution sets in command of the adversary, and define
\[
\Psi(A) = \sup_{h \in \hyps} \Bigl|  \widehat \ell_m(h) - {\E}_{\mu|S}[ \widehat \ell_m(h) ] \Bigl| \,.
\]
Again, define $A^{(j)}$ to be $A$ with the $j$-th attribution set $A_j$ changed to $A'_j$. From (\ref{e:bound_on_Z}) we get
\[
\Psi(A) - \Psi(A^{(j)})  
\leq 
\frac{16 F_2}{n\Sigma} \,.
\]
The standard McDiarmid's inequality yields, for any $t \geq 0$,
\begin{equation}\label{e:bound_on_term2}
{\PP}_{\mu|S} \left(\Psi(A) - {\E}_{\mu|S}[\Psi(A)] \geq t\right) 
\leq 
\exp \left(-\frac{2t^2\,n^2\,\Sigma^2}{256\,(m-k+1)\,F_2^2}\right)
\le
\exp \left(-\frac{t^2\,n\,\Sigma^2}{192\,p\,F_2^2}\right)\,,
\end{equation}
the last inequality deriving from $m \leq \frac{3np}{2}$ (which is implied by $S \in \mathcal{B}$), and $k \geq 1$.

Consider, for frozen $S \in \mathcal{B}$, the variables $\{V_j\}_{j=k}^{\muppers}$. Note that these variables are independent, 
but they need not have the same distribution.
Let $\{V'_j\}_{j=k}^{\muppers}$, with $V'_j = V(h,A'_j)$, $j = k,\ldots, \muppers$, be an independent sample conditioned on the same $S$, made up of independent random variables, where $V'_j$ has the same distribution as $V_j$. Denote for brevity by $A$ the collection of random variables $\{A_j\}$, and by $A'$ the collection $\{A'_j\}$.  

The standard symmetrization lemma still holds:
\begin{align*}
{\E}_{\mu|S}[\Psi(A)] 
=
{\E}_{A}[\Psi(A)] 
&= 
{\E}_{A}\left[ \sup_{h \in \hyps} \Bigg|\sum_{j=k}^{\muppers} V(h,A_j) - V(h,A'_j) \Bigg|\right] \\
&=
{\E}_{A,A',\sigma}\left[ \sup_{h \in \hyps} \Bigg|\sum_{j=k}^{\muppers} \sigma_j\Bigl(V(h,A_j) - V(h,A'_j) \Bigl) \Bigg|\right] \\
&\le
2\,{\E}_{A,\sigma}\left[ \sup_{h \in \hyps} \Bigg|\sum_{j=k}^{\muppers} \sigma_j V(h,A_j) \Bigg|\right] \,.
\end{align*}
Now, focus on the quantity
\begin{align*}
\sup_{h \in \hyps}\Bigg|\sum_{j=k}^{\muppers} \sigma_j V(h,A_j) \Bigg| 
=
\frac{1}{\big(\mlen\big) \Sigma}\,\sup_{h \in \hyps}\,\Bigg| \sum_{j=k}^{\muppers}\,\sigma_j\,\sum_{i=1}^n f_2(h(X_i))~  r(j,j(i)) ~\ind\{X_i \in A_j\} ~\Bigg|\,,
\end{align*}
which we bound via a covering argument.\footnote
{
A more refined chaining version of this covering argument can be leveraged here, which leads to replacing pseudo-dimension by a Dudley's integral. We decided not to take this route, as this would not add much to the value of the paper.
}

For frozen $X = \langle X_1,\ldots, X_n\rangle$, and index set $I \subseteq [n]$, with $|I| = k$, consider $\cC_{2,\epsilon}(X,I)$, a minimal $\epsilon$ cover of $\hyps$ w.r.t. the 2-norm on $X$ projected onto $I$, that is, w.r.t. the (pseudo-)metric
\begin{equation}\label{e:metric}
||h_1-h_2||_{2,X,I} = \sqrt{\frac{1}{k}\sum_{i \in I} |h_1(X_i) - h_2(X_i)|^2}\,.
\end{equation}
Denote by $|\cC_{2,\epsilon}(X,I)|$ the size of such a cover.
For given $h \in \hyps$, and $j$ such that $k \leq j \le \muppers$, let $h_j' = h'(h,A_j) \in \cC_{2,\epsilon}(X,A_j)$ be such that $||h-h_j'||_{2,X,A_j} \leq \epsilon$. Then we can write
\begin{align*}
\sup_{h \in \hyps}\Bigg|\sum_{j=k}^{\muppers} \sigma_j V(h,A_j) \Bigg|  
&\leq
\sup_{h \in \hyps}\Bigg|\sum_{j=k}^{\muppers} \sigma_j \Big(V(h,A_j) - V(h'_j,A_j)\Big)  \Bigg| +  \sup_{h \in \hyps}\Bigg|\sum_{j=k}^{\muppers} \sigma_j V(h'_j,A_j) \Bigg|\\
&\leq
\sup_{h \in \hyps}\sum_{j=k}^{\muppers} \Big| V(h,A_j) - V(h'_j,A_j)\Big| \\ 
&\qquad+  \max_{r\,:\,k \leq r \le \muppers}\,\sup_{h \in \cC_{2,\epsilon}(X,A_r) }\Bigg|\sum_{j=k}^{\muppers} \sigma_j V(h,A_j) \Bigg|\,.
\end{align*}
On the other hand
\begin{align*}
\left(\mlen\right) \Sigma\,\Big| V(h,A_j) - V(h'_j,A_j)\Big| 
&\leq
\sum_{i=1}^n \Big|f_2(h(X_i)) - f_2(h'_j(X_i))\Big|~  r(j,j(i)) ~\ind\{X_i \in A_j\} \\
&\leq
2pL\,\sum_{i=1}^n \Big|h(X_i) - h'_j(X_i)\Big|~  \pi[i] ~\ind\{X_i \in A_j\} \\
&\mbox{(by the $L$-Lischitzness of $f_2(\cdot)$ and Lemma \ref{lem:beta01approxsamplecomplexity})}\\
&\leq
2pL\,\sqrt{\sum_{i\,:\, X_i \in A_j} \Big|h(X_i) - h'_j(X_i)\Big|^2 }~ \sqrt{ \sum_{i\,:\, X_i \in A_j} \pi^2[i] } \\
&=
2pL\sqrt{k}\,||h-h'_j||_{2,X, A_j}\,\sqrt{\Sigma}\\
&\leq
2pL\epsilon\sqrt{k\Sigma}\,.
\end{align*}
Plugging back gives
\begin{align}
\sup_{h \in \hyps}\Bigg|\sum_{j=k}^{\muppers} \sigma_j V(h,A_j) \Bigg| 
&\leq
2pL\epsilon\sqrt{k\Sigma}\,\frac{\muppers - k+1}{\big(\mlen\big) \Sigma} \notag\\
&\qquad+ \max_{r\,:\,k \leq r \le \muppers}\,\sup_{h \in \cC_{2,\epsilon}(X,A_r) }\Bigg|\sum_{j=k}^{\muppers} \sigma_j V(h,A_j) \Bigg|\notag\\
&=
\frac{2pL\,\epsilon\sqrt{k}}{\sqrt{\Sigma}}  + \max_{r\,:\,k \leq r \le \frac{np}{2}-k}\,\sup_{h \in \cC_{2,\epsilon}(X,A_r) }\Bigg|\sum_{j=k}^{\frac{np}{2}-k} \sigma_j V(h,A_j) \Bigg|\,,\label{e:this_one}
\end{align}
the equality following from the fact that $S \in \mathcal{B}$ implies $\muppers = \frac{np}{2}-k$.
Note that the size of
\[
{\bigcup}_{r\,:\,k \leq r \le \frac{np}{2}-k} \,\cC_{2,\epsilon}(X,A_r)
\]
is at most
\[
\Bigl(\frac{np}{2}-2k+1\Big)\,\max_{j\,:\,k \leq j \le \frac{np}{2}-k}|\cC_{2,\epsilon}(X,A_j)|\,.
\]
We take expectation w.r.t. the Rademacher variables $\sigma$, apply Massart's finite lemma \citep{ma00}  to the second term of (\ref{e:this_one}), and then take an outer expectation w.r.t. $A$. This yields 
\[
{\E}_{\mu|S}[\Psi(A)] 
\leq
\frac{4pL\,\epsilon\sqrt{k}}{\sqrt{\Sigma}} 
+ 4\,{\E}_A\left[\sqrt{\sum_{j=k}^{\frac{np}{2}-k} V^2_j}\,\sqrt{2\,\log \Bigg( \Bigl(\frac{np}{2}-2k+1\Big)\,\max_{j\,:\,k \leq j \le \frac{np}{2}-k}|\cC_{2,\epsilon}(X,A_j)|\Biggl)}\right]\,.
\]
But from (\ref{e:bound_on_Z}) we have, deterministically,
\[
\sqrt{\sum_{j=k}^{\muppers} V^2_j} \le \frac{8F_2}{n\,\Sigma}\,\sqrt{\frac{np}{2}-2k+1}\,.
\]
We now find an upper bound on the covering number $|\cC_{2,\epsilon}(X,A_j)|$. First, note that, for each $j$, the covering number $|\cC_{2,\epsilon}(X,A_j)|$ cannot be bigger than $(\frac{1}{\epsilon})^k$.  This is because the functions $h \in \hyps$ take values in the interval $[0,1]$ and, for the sake of metric (\ref{e:metric}), they are evaluated only in the $k$ points $\{X_i\}_{i \in A_j}$. Moreover, since the fat-shattering dimension of $\hyps$ at any scale is always upper bounded by the pseudo-dimension $\pseudodim(\hyps)$, we have $|\cC_{2,\epsilon}(X,A_j)| = O\big( (\frac{1}{\epsilon})^{O(\pseudodim(\hyps))} \big)$ by, e.g., Theorem 1 of \citet{mendelson2003entropy}.

This implies, using $k \le \frac{np}8$, and the inequality $\sqrt{a+b} \leq \sqrt{a} + \sqrt{b}$,
\begin{align*}
{\E}_{\mu|S}[\Psi(A)] 
= \inf_{\epsilon \geq 0} \Biggl(\underbrace{O\left(\frac{pL\,\sqrt{k}}{\sqrt{\Sigma}}\right)}_{C_1}\,\epsilon 
&+ 
\underbrace{O\left(\frac{F_2}{\Sigma}\,\sqrt{\frac{p\,\complexitymeasurehere}{n}}\right)}_{C_2}\,\sqrt{\log \frac{1}{\epsilon}} \\
&+ 
\underbrace{O\left(\frac{F_2}{\Sigma}\,\sqrt{\frac{p\log(1+np)}{n}}\right)}_{C_3}
\Biggl)\,,
\end{align*}
where the $O(\cdot)$ only conceals absolute constants. We have then an expression of the form
\[
C_1\epsilon + C_2\sqrt{\log \frac{1}{\epsilon}} + C_3
\]
that we want to optimize over $\epsilon \geq 0$. In particular, we set
\[
\epsilon = \frac{C_2}{2C_1}
\]
to obtain
\begin{align*}
{\E}_{\mu|S}[\Psi(A)] 
&= 
O \Biggl(\frac{F_2}{\Sigma}\,\sqrt{\frac{p\,\complexitymeasurehere}{n}}
\Biggl[1+\sqrt{\log\left(\frac{p L k\Sigma n}{F_2 \complexitymeasurehere}\right)}\Biggl]\\
&\qquad \qquad \qquad+ \frac{F_2}{\Sigma}\,\sqrt{\frac{p\log(1+np)}{n}}
\Biggl)\\
&=
\widetilde O \left( \frac{F_2}{\Sigma}\,\sqrt{\frac{p\,\complexitymeasurehere}{n}}\right)\,,
\end{align*}
provided $S \in \mathcal{B}$.
Combining with (\ref{e:bound_on_term2}) results in
\begin{align*}
{\PP}_{\mu|S} \left(\Psi(A)  > \underbrace{{\E}_{\mu|S}[\Psi(A)] + t}_{\epsilon/4} \,\Big|\,S\right) 
&\leq
{\PP}_{\mu|S} \Big(\Psi(A) > {\E}_{\mu|S}[\Psi(A)] + t\,|\, S \in \mathcal{B}\Big) + \ind\{S \notin \mathcal{B}\}\\
&\le
\exp \left(-\frac{t^2\,n\,\Sigma^2}{192\,p\,F_2^2}\right) + \ind\{S \notin \mathcal{B}\}\,.
\end{align*}
holding for every realization of $S$. Thus, upon setting
$$
\epsilon/4 \geq {\E}_{\mu|S}[\Psi(A)] + t \,,
$$
with ${\E}_{\mu|S}[\Psi(A)]$ bounded as above when $S \in \mathcal{B}$, and noting that $\PP(S \notin \mathcal{B}) \leq e^{-\Omega(np)}$ via standard Chernoff bounds, we can write
\begin{align*}
(\ref{e:term2}) 
&= 
{\E}_S \left[{\PP}_{\mu|S}\left(\Psi(A) \geq \epsilon/4 \right)\right]   
\leq
\exp \left(-\frac{n\,\Sigma^2}{192\,p\,F_2^2}\cdot t^2\right) + e^{-\Omega(np)}\,.
\end{align*}

\paragraph{Finishing the proof:} Combining (\ref{e:term1}) and (\ref{e:term2}) we have therefore obtained, for
\begin{align*}
\epsilon/4 \geq\, t\, +\, \max\Bigg\{ \frac{32L\, R_n(\hyps)}{\Sigma} +     \frac{40 F_2}{\Sigma} e^{-\Omega(np)} ,\, 
\widetilde\Omega \left( \frac{F_2}{\Sigma}\,\sqrt{\frac{p\,\complexitymeasurehere}{n}}\right) \Bigg\}\,,
\end{align*}
we have
\begin{align}\label{e:almost_final_bound}
\PP \left( \regret(\wh) \geq \epsilon \right)    \leq 
e^{-\Omega(np)} + \exp\left( - \frac{n\Sigma^2}{1152 F_2^2} \cdot t^2 \right) + \exp \left(-\frac{n\,\Sigma^2}{192\,p\,F_2^2}\cdot t^2\right) + e^{-\Omega(np)}\,.
\end{align}
Now, the right-hand side of (\ref{e:almost_final_bound}) is smaller than $\delta$ if
\[
t = \Omega\left( \frac{F_2}{\Sigma}\,\sqrt{\frac{\log \frac{1}{\delta}}{n}}\right) \qquad \mbox{and}\qquad np = \Omega\left(\log \frac{1}{\delta}\right)\,.
\]
Moreover, the conditions $\delta \leq 1/2$ and
\(
np = \Omega\left(\log \frac{1}{p}\right)
\)
imply
\[
\frac{40 F_2}{\Sigma} e^{-\Omega(np)} = O\left( \frac{F_2}{\Sigma}\,\sqrt{\frac{\log \frac{1}{\delta}}{n}}\right)\,. 
\]
Hence, under these conditions  (\ref{e:almost_final_bound}) implies
\(
\PP \left( \regret(\wh) \geq \epsilon \right)    \leq \delta 
\)
for
\[
\epsilon = \widetilde\Omega \left( \frac{L\, R_n(\hyps)}{\Sigma} + \frac{F_2}{\Sigma}\,\sqrt{\frac{\log \frac{1}{\delta}}{n}} + \frac{F_2}{\Sigma}\,\sqrt{\frac{p\,\complexitymeasurehere}{n}}\right)\,,
\]
and the proof is concluded.
\end{proof}

\subsection{Robustness to error in prior}\label{subsec:robustnessproofs}
First, we explicitly detail the new definitions of $\widehat\beta_1, \widehat\beta_0, \widehat\ell(h, j)$. Specifically, we now let 
\begin{align*}
\widehat\beta_1(j,i) &:= \frac{\widehat\pi[i]\, B_{n, p, j+k}}{p} + \Biggl( B_{n-1, p, j+k-1} - \frac{1 - p\, B_{n-1, p, j+k-1}}{1-p} \Biggr) \big(1-\widehat\pi[i] \big)\,, \\
\widehat\beta_0(j,i) &:= \frac{1 - p\, B_{n-1, p, j+k-1}}{1-p} \big(1-\widehat\pi[i]\big)\,.
\end{align*}
where $B_{n, p, j}$ and $p$ are defined as in Theorem \ref{prop:unbiasedestimatoroneji}. We now define $\widehat\Sigma = \sum_{i=1}^k \widehat\pi[i]^2 = \| \widehat\pi \|_2^2$, and
\begin{align} 
\widehat \ell(h,j,i) &= \frac{f_2(h(A_j[i]))}{\widehat \beta_1(j, i)} + \frac{\E\Bigl[f_1 ( h(X) )\Bigl]}{B_{n, p, j+k}}  
- \frac{\widehat \beta_0(j, i) \E\Bigl[f_2 ( h(X) )\Bigl]}{\widehat \beta_1(j, i) B_{n, p, j+k}}\,, \notag \\
\widehat\ell(h, j) &= 
\Bigg(\frac1{\widehat\Sigma} \sum_{i=1}^k \frac{ \widehat\pi[i]^2}{\widehat\beta_1(j, i)} f_2(h(A_j[i]))
-  \frac1{\widehat\Sigma} \Biggl( \sum_{i=1}^k \frac{\widehat\pi[i]^2 \cdot \widehat\beta_0(j, i)}{\widehat\beta_1(j, i) B_{n, p, j+k}} \Biggl) \E[f_2 (h(x))] \notag \\
&\qquad \qquad \qquad + \frac1{B_{n, p, j+k}} \E[f_1 (h(x))]\Bigg)\ind\{j \le M-k\} \, \notag.
\end{align}
Now turning to the proof, the only new step we need to prove Theorem \ref{thm:robustallattrsets} is the following:
\begin{lemma}\label{lem:robustnessbias}
For all $h \in \hyps$, we have 
\[ 
\big| {\E}_{\mu}[\widehat\ell] - \mathcal{L}(h) \big| = \begin{cases} O\Bigg( pF_2 \Big( \frac{\|\pi - \widehat \pi\|_1}{\Sigma} + \frac{\|\pi - \widehat \pi\|_2}{\Sigma^{3/2}} \Big) + pk^2 F_2 e^{-\Omega(np)} \Bigg) & {\mbox{if\,\,}} \Sigma \ge 8 \| \pi - \widehat\pi \|_2^2\,, \\
O\Big( \frac{p k F_2 \|\pi - \widehat \pi\|_2}{\Sigma^{1/2}} + pk^2 F_2 e^{-\Omega(np)} \Big) &{\mbox{if\,\,}} \Sigma < 8\| \pi - \widehat\pi \|_2^2\,. \end{cases}
\]
Also, when $\Sigma \ge 8 \| \pi - \widehat\pi \|_1^2$ we have $\widehat\Sigma \in [c_1 \Sigma, c_2 \Sigma]$ for universal constants $c_1, c_2>0$. 
\end{lemma}
Furthermore, as detailed below, it is not possible to avoid explicit $k$-dependence in the last case above using the current strategy of analysis.

\begin{remark}\label{rem:tightnessofrobustnessproof}
Notice that up to multiplicative constants, the above analysis in the second case above (when $\Sigma < 8 \| \pi - \widehat\pi \|^2_2$) is tight. Fix any real parameter $\delta>0$. For any $k \ge 1$, suppose $\pi = \Big( \frac1{k}-\frac{\delta}2, \frac1k, \ldots, \frac1k, \frac1k+\frac{\delta}2 \Big)$ and $\widehat\pi = \Big(\frac1k, \ldots, \frac1k\Big)$. Then $\|\pi - \widehat \pi\|_2 = \frac{\delta}{\sqrt{2}}$, $\Sigma \approx \frac1k + \delta^2$, $\widehat\Sigma = \frac1k$. Hence for $k$ large enough in terms of $\delta$, we have $\Sigma < 8 \| \pi - \widehat\pi \|^2_2$. 
Note in this example, we have $\widehat\Sigma = \frac1k$.
The resulting regret bound from Theorem \ref{thm:robustallattrsets} comes from replacing $\Sigma$ from Theorem \ref{thm:allattributionsetssamplecomplexity} by $\widehat \Sigma$ when $\Sigma \ge 8 \| \pi - \widehat\pi \|^2_2$ and by $\frac1k$ when $\Sigma < 8 \| \pi - \widehat\pi \|^2_2$ -- see the proof of Theorem \ref{thm:robustallattrsets} below. 
Since $\widehat\Sigma = \frac1k$ in this example, one cannot refine this current analysis strategy using a larger lower bound on $\widehat \Sigma$ in the case $\Sigma < 8 \| \pi - \widehat\pi \|^2_2$.
\end{remark}

\begin{proof}[Proof of Theorem \ref{thm:robustallattrsets}, given Lemma \ref{lem:robustnessbias}]
First, we discuss how to establish the result when we have $np = \Omega\left( \log\left(\frac{1}{\delta\,p} \max_{i \in [k]}\frac{1}{\widehat \pi[i]} \right) \right)$. Then following Remark \ref{rem:smallpiindices}, we will explain how to establish the result in its full generality. 

The proof follows the exact same strategy as the proof of Theorem \ref{thm:allattributionsetssamplecomplexity}. The main change is that in the initial steps of that proof, we now decompose 
\begin{align*}
&\ind\Bigl\{\regret(\wh) \geq \epsilon + \big( \bias + pk^2 e^{-\Omega(np)}\big) \Bigl\} \\
&\qquad \leq
\ind\Biggl\{ 2\,\sup_{ h \in \hyps} \Bigl| \widehat \ell(h, S, \cA, \alpha) - \popl(h) \Bigl| \geq \epsilon + \big( \bias + pk^2 e^{-\Omega(np)}\big) \Big\} \\
&\qquad \leq \ind\Biggl\{ 2 \sup_{h \in \hyps} \big| {\E}_{\mu}[\widehat\ell] - \mathcal{L}(h) \big| + 2\,\sup_{ h \in \hyps} \Bigl| \widehat \ell(h, S, \cA, \alpha) - {\E}_{\mu}[\widehat\ell] \Bigl| \geq \epsilon + \big( \bias + pk^2 e^{-\Omega(np)}\big) \Biggl\} \\
&\qquad \leq \ind\Biggl\{ \,\sup_{ h \in \hyps} \Bigl| \widehat \ell(h, S, \cA, \alpha) - {\E}_{\mu}[\widehat\ell] \Bigl| \geq \epsilon/2 \Biggl\}\,.
\end{align*}
Here we used that $\sup_{h \in \hyps} \big| {\E}_{\mu}[\widehat\ell] - \mathcal{L}(h) \big| \le \frac12 \big( \bias + pk^2 e^{-\Omega(np)}\big)$ by Lemma \ref{lem:robustnessbias}, and the fact that
\[ 
\sup_{ h \in \hyps} \Bigl| \widehat \ell(h, S, \cA, \alpha) - \popl(h) \Bigl| \le \sup_{ h \in \hyps} \Bigl| \widehat \ell(h, S, \cA, \alpha) - {\E}_{\mu}[\widehat\ell]| + \sup_{h \in \hyps}\big| {\E}_{\mu}[\widehat\ell] - \mathcal{L}(h) \big|\,. 
\]
The rest of the proof is now identical as that of Theorem \ref{thm:allattributionsetssamplecomplexity}.
In particular, we replace every instantiation of quantities arising in the estimator $\widehat \ell$ that formerly depended on $\pi$, now by the analogous quantities depending on $\widehat\pi$ (e.g. the $r(j, j(i))$ are now defined analogously as before, but in terms of $\widehat \pi$ which defines $\widehat \beta_1, \widehat \beta_0$). Similarly, the $\Sigma$ are now all replaced by $\widehat\Sigma$. The condition $np = \Omega\left( \log\left(\frac{1}{\delta\,p} \max_{i \in [k]}\frac{1}{\widehat \pi[i]} \right) \right)$ enables us to use Lemma \ref{lem:beta01approxsamplecomplexity} to bound $\widehat \beta_1, \widehat \beta_0$.
Note that the attribution sets are constructed as per the adversary's play, which is according to $\pi$. The rationale that for a given $j$, there are at most $k$ indices $i$ such that $x_i \in A_j$ and the corresponding $j(i)$ are all distinct remains exactly the same, so the proof does not change. 

The only situation in the proof where a new bound, that does not arise from replacing all the quantities that formerly depended on $\pi$ by the analogous quantities depending on $\widehat\pi$, is the proof of (\ref{eq:this_one3}). Here, we analogously can derive the exact same bound as follows:
\begin{align}
&{\E}_{\mu | S}\left[\sum_{j\,:\, k \le j \le \muppers \,,\, x_i \in A_j} r(j,j(i))\right] \notag \\
& \qquad = \sum_{t=1}^{r^+} \sum_{i'=b^+(t)+1}^k \pi[i']\, r\Big(j_{b^+(t)},\,i'-b^+(t)\Big) +
\sum_{t=1}^{r^-} \sum_{i'=1}^{k-b^-(t)} \pi[i']\, r\Big(j_{b^-(t)},\, i' + b^-(t)\Big) \notag \\
& \qquad \leq 
2p \sum_{t=1}^{r^+} \sum_{i'=b^+(t)+1}^k \pi[i'] \, \widehat\pi[i'-b^+(t)] +
2p \sum_{t=1}^{r^-} \sum_{i'=1}^{k-b^-(t)} \pi[i']\, \widehat\pi[i'+b^-(t)] \notag \\
&\qquad \leq 2p \sum_{b=0}^{k-1} \sum_{i'=b+1}^k \pi[i'] \widehat\pi[i'-b] + 2p \sum_{b=1}^{k-1} \sum_{i'=1}^{k-b} \pi[i'] \widehat\pi[i'+b]\, \notag \\
&\qquad \leq
2p\,\sum_{i'=1}^{k} \pi[i']\,\sum_{b=1}^{k}  \widehat\pi[b] + 
2p\,\sum_{i'=1}^{k-1} \pi[i]\,\sum_{b=1}^{k}  \widehat\pi[b]\notag\\
&\qquad =
2p\,\sum_{i'=1}^{k} \pi[i'] + 
2p\,\sum_{i'=1}^{k-1} \pi[i']\notag\\
&\qquad \leq 
4p\,.
\end{align}
Hence the same proof of Theorem \ref{thm:allattributionsetssamplecomplexity} goes through as described above. The claimed regret follows from replacing $\widehat\Sigma$ by $\Theta(\Sigma)$ when $\Sigma \ge 8 \| \pi - \widehat\pi \|_1^2$ as per Lemma \ref{lem:robustnessbias}, and using the worst case bound $\widehat\Sigma \ge \frac1k$ otherwise. We finally upper bound $pk^2 F_2 e^{-\Omega(np)} \le \frac{F_2}{\Sigma} \sqrt{\frac1{n}}$ for $np = \Omega\left( \log\left(\frac{1}{\delta\,p} \max_{i \in [k]}\frac{1}{\widehat \pi[i]} \right) \right)$; note this condition implies that $np = \Omega\Big( \max\Big\{ \log\Big( \frac{1}{p}\Big), \log(pk) \Big\}\Big)$. Thus this extra term $pk^2 F_2 e^{-\Omega(np)}$ can be subsumed into the pre-existing terms in the regret.

Finally, to prove the Theorem under the condition $np \ge \Omega\Big(\log\Big( \frac{\sqrt{2k}}{\delta p}\Big)\Big)$, we construct $\widehat \ell$ by restricting to the $i$ such that $\widehat \pi[i] \ge \frac1{\delta p} e^{-np}$, as discussed in Remark \ref{rem:smallpiindices}. We again have \begin{align*}
\sum_{i\,:\,\widehat \pi[i] < \frac1{\delta p} e^{-np}} \widehat \pi[i]^2 \le \frac{k e^{-2np}}{\delta^2 p^2} \le \frac1{2k} \le \frac{\widehat \Sigma}2, \text{ thus }\sum_{i\,:\,\widehat \pi[i] \ge \frac1{\delta p} e^{-np}} \widehat \pi[i]^2 \ge \frac{\widehat \Sigma}2\,,
\end{align*}
and so the rate only changes by a constant factor. Note that $np \ge \Omega\Big(\log\Big( \frac{\sqrt{2k}}{\delta p}\Big)\Big)$ again implies $np = \Omega\Big( \max\Big\{ \log\Big( \frac{1}{p}\Big), \log(pk) \Big\}\Big)$, allowing us to subsume the extra $pk^2 F_2 e^{-\Omega(np)}$ term into the pre-existing terms in the regret.
\end{proof}
\begin{proof}[Proof of Lemma \ref{lem:robustnessbias}]
By Theorem \ref{prop:unbiasedestimatoroneji} and the same work we did prior to stating Theorem \ref{thm:allattributionsetssamplecomplexity} in Section \ref{sec:refined}, we can write 
\begin{align*}
\popl(h) &= {\E}_{\mu}\Bigg[ \frac1{\frac{np}2 - 2k + 1} \sum_{j=k}^{np/2-k} \Bigg(\frac1{\Sigma} \sum_{i=1}^k \frac{ \pi[i]^2}{\beta_1(j, i)} f_2(h(A_j[i]))
-  \frac1{\Sigma} \Biggl( \sum_{i=1}^k \frac{\pi[i]^2 \cdot \beta_0(j, i)}{\beta_1(j, i) B_{n, p, j+k}} \Biggl) \E[f_2 (h(x))] \\
&\qquad \qquad \qquad \qquad \qquad \qquad \qquad + \frac1{B_{n, p, j+k}} \E[f_1 (h(x))]\Bigg)\ind\{j \le M-k\} \Bigg]\,.
\end{align*}
By analogous reasoning, we have
\begin{align*}
{\E}_{\mu}[\widehat\ell] &= {\E}_{\mu}\Bigg[ \frac1{\frac{np}2 - 2k + 1} \sum_{j=k}^{np/2-k} \Bigg(\frac1{\widehat\Sigma} \sum_{i=1}^k \frac{ \widehat\pi[i]^2}{\widehat\beta_1(j, i)} f_2(h(A_j[i]))
-  \frac1{\widehat\Sigma} \Biggl( \sum_{i=1}^k \frac{\widehat\pi[i]^2 \cdot \widehat\beta_0(j, i)}{\widehat\beta_1(j, i) B_{n, p, j+k}} \Biggl) \E[f_2 (h(x))] \\
&\qquad \qquad \qquad \qquad \qquad \qquad \qquad + \frac1{B_{n, p, j+k}} \E[f_1 (h(x))]\Bigg)\ind\{j \le M-k\} \Bigg]\,.
\end{align*}
Thus 
\[ 
\big| {\E}_{\mu}[\widehat\ell] - \mathcal{L}(h) \big| \le \big|\expressionI\big| + \big|\expressionII\big|\,,
\]
where
\begin{align*}
\expressionI &:= \frac1{\big(\mlen\big)\widehat\Sigma} \sum_{j=k}^{np/2-k} \sum_{i=1}^k \frac{\widehat\pi[i]^2}{\widehat\beta_1(j, i)} f_2(h(A_j[i])) \\
&\qquad \qquad \qquad - \frac1{\big(\mlen\big) \Sigma} \sum_{j=k}^{np/2-k} \sum_{i=1}^k \frac{\pi[i]^2}{\beta_1(j, i)} f_2(h(A_j[i]))\,,\\
\expressionII &:= -\frac1{\big(\mlen\big) B_{n, p, j+k} \widehat\Sigma} \sum_{j=k}^{np/2-k} \sum_{i=1}^k \frac{ \widehat\pi[i]^2 \widehat\beta_0(j, i)}{ \widehat\beta_1(j, i)} \E[f_2(h(x))] \\
&\qquad \qquad \qquad + \frac1{\big(\mlen\big)B_{n, p, j+k} \Sigma} \sum_{j=k}^{np/2-k} \sum_{i=1}^k \frac{ \pi[i]^2 \beta_0(j, i)}{ \beta_1(j, i)} \E[f_2(h(x))]\,.
\end{align*}
By Lemma \ref{lem:beta01approxsamplecomplexity}, we have $\Big|\beta_1(j, i) -\frac{\pi[i]}p \Big| \le O\Big(\frac{e^{-\Omega(np)}}p\Big)$, and the exact same proof as of Lemma \ref{lem:beta01approxsamplecomplexity} gives $\Big|\widehat\beta_1(j, i) -\frac{\widehat\pi[i]}p \Big| \le O\Big(\frac{e^{-\Omega(np)}}p\Big)$. Similarly we have $\Big|\beta_0(j, i) - (1-\pi[i])\Big| $, $\Big|\widehat\beta_0(j, i) - (1-\widehat\pi[i])\Big| \le O(pe^{-np})$. Thus letting $\delta = \beta_1(j, i) -\frac{\pi[i]}p$, we obtain
\begin{align*}
\Big|\frac{\pi[i]^2}{\pi[i]/p} - \frac{\pi[i]^2}{\beta_1(j, i)}\Big| = \Big|\frac{\pi[i]^2 \delta}{\frac{\pi[i]}p \cdot \big( \frac{\pi[i]}p + \delta\big)}\Big| = \frac{p^2 \pi[i] \delta}{\pi[i] + p \delta} \le p^2 \delta = O(p e^{-\Omega(np)})\,,
\end{align*}
and similarly $\Big|\frac{\widehat\pi[i]^2}{\widehat\pi[i]/p} - \frac{\widehat\pi[i]^2}{\widehat\beta_1(j, i)}\Big| = O(pe^{-\Omega(np)})$. Now using the above upper bound, and by our condition on $n$ and Lemma \ref{lem:beta01approxsamplecomplexity}, we have
\begin{align*}
&\Big|\frac{\pi[i]^2 \beta_0(j, i)}{\beta_1(j, i)} - \frac{\pi[i]^2(1-\pi[i])}{\pi[i]/p}\Big| \\
&\qquad = \Big| \frac{\pi[i]^2 \beta_0(j, i)}{\beta_1(j, i)} - \frac{\pi[i]^2 \beta_0(j, i)}{\pi[i] / p} + \frac{\pi[i]^2 \beta_0(j, i)}{\pi[i] / p} - \frac{\pi[i]^2(1-\pi[i])}{\pi[i]/p} \Big| \\
&\qquad \le \beta_0(j, i)\Big| \frac{\pi[i]^2}{\beta_1(j, i)} - \frac{\pi[i]^2}{\pi[i] / p} \Big| + p \pi[i] \big|\beta_0(j, i) - (1-\pi[i])\big| \\
&\qquad = O(p e^{-\Omega(np)}) + p \cdot O(pe^{-np})\\
&\qquad = O(p e^{-\Omega(np)})\,,
\end{align*}
and similarly for $\Big|\frac{\widehat\pi[i]^2 \widehat\beta_0(j, i)}{\widehat\beta_1(j, i)} - \frac{\widehat\pi[i]^2(1-\widehat\pi[i])}{\widehat\pi[i]/p}\Big|$. 
Thus applying these bounds, along with $\widehat\Sigma, \Sigma \ge \frac1k$, and $j+k \le \frac{np}2$ (so that $B_{n, p, j+k} \ge \frac12$), we obtain
\begin{align*}
\big| \expressionI \big| &\le \frac{p}{np/2-2k+1} \sum_{j=k}^{np/2-k} \sum_{i=1}^k \big| f_2(h(A_j[i])) \big| \cdot \Big| \frac{\widehat\pi[i]}{\widehat\Sigma} - \frac{\pi[i]}{\Sigma} \Big| \\
&\qquad + \frac{k}{\mlen} \sum_{j=k}^{np/2-k} \sum_{i=1}^k \big|f_2(h(A_j[i]))\big| \Bigg(\Big|\frac{\pi[i]^2}{\pi[i]/p} - \frac{\pi[i]^2}{\beta_1(j, i)}\Big| + \Big|\frac{\widehat\pi[i]^2}{\widehat\pi[i]/p} - \frac{\widehat\pi[i]^2}{\widehat\beta_1(j, i)}\Big| \Bigg)  \\
&\le p F_2 \sum_{i=1}^k \Big| \frac{\widehat\pi[i]}{\widehat\Sigma} - \frac{\pi[i]}{\Sigma} \Big| + O\Big( pk^2 F_2 e^{-\Omega(np)}\Big)\,. \\
\big| \expressionII \big| &\le \frac{p}{\big(\mlen\big) B_{n, p, j+k}} \Big| \E\Big[ f_2(h(x)) \Big] \Big| \cdot \sum_{j=k}^{\mupper} \Bigg| \sum_{i=1}^k -\frac{\widehat\pi[i](1-\widehat\pi[i])}{\widehat\Sigma} + \frac{\pi[i](1-\pi[i])}{\Sigma} \Bigg|  \\
&\qquad + \frac{k}{\big(\mlen\big) B_{n, p, j+k}} \Big| \E\Big[ f_2(h(x)) \Big] \Big| \cdot \sum_{j=k}^{np/2-k} \sum_{i=1}^k \Big|\frac{\pi[i]^2 \beta_0(j, i)}{\beta_1(j, i)} - \frac{\pi[i]^2(1-\pi[i])}{\pi[i]/p}\Big| \\
&\qquad + \frac{k}{\big(\mlen\big) B_{n, p, j+k}} \Big| \E\Big[ f_2(h(x)) \Big] \Big| \cdot \sum_{j=k}^{np/2-k} \sum_{i=1}^k \Big|\frac{\widehat\pi[i]^2 \widehat\beta_0(j, i)}{\widehat\beta_1(j, i)} - \frac{\widehat\pi[i]^2(1-\widehat\pi[i])}{\widehat\pi[i]/p}\Big| \\
&\le 2p F_2 \Big| \sum_{i=1}^k -\frac{\widehat\pi[i](1-\widehat\pi[i])}{\widehat\Sigma} + \frac{\pi[i](1-\pi[i])}{\Sigma} \Big| + O(p k^2 F_2 e^{-\Omega(np)}) \\
&= 2pF_2 \Bigg| \sum_{i=1}^k \Big( -\frac{\widehat\pi[i]}{\widehat\Sigma} + \frac{\pi[i]}{\Sigma} \Big) \Bigg| + O(p k^2 F_2 e^{-\Omega(np)}) \\
&= 2p F_2 \Big| \frac1{\widehat\Sigma}-\frac1{\Sigma}\Big| + O(p k^2 F_2 e^{-\Omega(np)})\,.
\end{align*}
It remains to upper bound $ \sum_{i=1}^k \Big| \frac{\widehat\pi[i]}{\widehat\Sigma} - \frac{\pi[i]}{\Sigma} \Big|$ and $\Big| \frac1{\widehat\Sigma}-\frac1{\Sigma}\Big|$. To this end note
\begin{align*}
\sum_{i=1}^k \Big|\frac{\widehat\pi[i]}{\widehat\Sigma} - \frac{\pi[i]}{\Sigma} \Big| &= \sum_{i=1}^k \Big|\frac{\widehat\pi[i]}{\widehat\Sigma} - \frac{\widehat\pi[i]}{\Sigma} + \frac{\widehat\pi[i]}{\Sigma} - \frac{\pi[i]}{\Sigma} \Big| \\
&\le \frac1{\Sigma} \sum_{i=1}^k \big| \widehat{\pi}[i] - \pi[i] \big| + \Big| \frac1{\widehat\Sigma} - \frac1{\Sigma} \Big| \Big(\sum_{i=1}^k \widehat\pi[i] \Big) \\
&= \frac{\| \pi - \widehat\pi \|_1}{\Sigma} + \Big| \frac1{\widehat\Sigma} - \frac1{\Sigma} \Big|\,.
\end{align*}
Let $\widehat\pi[i] - \pi[i] := \delta_i \in [-1, 1]$, thus $\sum_{i=1}^k \delta_i = 0$, $\sum_{i=1}^k |\delta_i| = \| \pi - \widehat\pi \|_1$. We observe that 
\begin{align}
\big| \widehat\Sigma - \Sigma \big| = \Big| \sum_{i=1}^k \big(\pi[i] + \delta_i\big)^2 - \sum_{i=1}^k \pi[i]^2 \Big| &\le 2\sum_{i=1}^k \big| \pi[i] \big| \cdot \big| \delta_i\big| + \sum_{i=1}^k \delta_i^2 \notag \\
&\le \frac12 \sum_{i=1}^k \pi[i]^2 + 3 \sum_{i=1}^k \big| \delta_i \big|^2 \notag \\
&= \frac{\Sigma}2 + 3\| \pi - \widehat\pi \|_2^2\,.\label{eq:sigmavssigmahat}
\end{align}
Furthermore, we have 
\begin{align*}
\left| \frac{1}{\widehat\Sigma} - \frac{1}{\Sigma} \right| 
&= \frac{\left| \Sigma - \widehat\Sigma \right|}{\widehat\Sigma\,\Sigma} = \frac{\left| \sum_{i=1}^k (\pi[i]^2 - \widehat{\pi}[i]^2) \right|}{\widehat\Sigma\,\Sigma} = \frac{\left| \sum_{i=1}^k (\pi[i] - \widehat{\pi}[i])(\pi[i] + \widehat{\pi}[i]) \right|}{\widehat\Sigma\,\Sigma}\,.
\end{align*}
Applying the Cauchy-Schwarz inequality to the numerator:
\begin{align*}
\left| (\pi - \widehat{\pi})^\top (\pi + \widehat{\pi}) \right| 
&\leq \|\pi - \widehat{\pi}\|_2 \cdot \|\pi + \widehat{\pi}\|_2\,.
\end{align*}
Using the Triangle Inequality for the norm of the sum $\|\pi + \widehat{\pi}\|_2 \leq \|\pi\|_2 + \|\widehat{\pi}\|_2$:
\begin{align*}
\|\pi + \widehat{\pi}\|_2 \leq \sqrt{\Sigma} + \sqrt{\widehat\Sigma}\,.
\end{align*}
Substituting this back into the expression yields the result:
\begin{align}
\left| \frac{1}{\widehat\Sigma} - \frac{1}{\Sigma} \right| 
&\leq \frac{\|\pi - \widehat{\pi}\|_2 (\sqrt{\Sigma} + \sqrt{\widehat\Sigma})}{\widehat\Sigma\,\Sigma}\,.\label{eq:sigmavssigmahatcauchy}
\end{align}
To complete the cases, we break into cases depending on whether $\Sigma \ge 8 \| \pi - \widehat\pi \|_1^2$ or not.

\begin{itemize}
\item \underline{$\Sigma \ge 8 \| \pi - \widehat\pi \|_2^2$.} 
Eq. (\ref{eq:sigmavssigmahat}) yields $\widehat\Sigma \in \big[ \frac18 \Sigma , \frac{15}8\Sigma \big]$, as originally claimed.
Thus by (\ref{eq:sigmavssigmahatcauchy}), 
\begin{align*}
\Big|\frac{1}{\widehat\Sigma} - \frac{1}{\Sigma} \Big| \le O\Big( \frac{\|\pi - \widehat \pi\|_2}{\Sigma^{3/2}}\Big)\,.
\end{align*}
Hence $\big|\expressionI\big| + \big|\expressionII\big| = pF_2 \cdot O\Big( \frac{\|\pi - \widehat \pi\|_1}{\Sigma} + \frac{\|\pi - \widehat \pi\|_2}{\Sigma^{3/2}} \Big)$ in this case.
\item
\underline{$\Sigma < 8 \| \pi - \widehat\pi \|^2_2$.}
Observe that, as each $\pi[i] - \widehat{\pi}[i] \in [-1,1]$, Cauchy-Schwarz inequality gives
\begin{align*}
\big|\widehat\Sigma - \Sigma| &= \Big| \sum_{i=1}^k (\pi[i] - \widehat{\pi}[i]) (\pi[i] + \widehat{\pi}[i]) \Big| \\
&\le \Big( \sum_{i=1}^k (\pi[i] + \widehat{\pi}[i])^2 \Big)^{1/2} \Big( \sum_{i=1}^k  (\pi[i] - \widehat{\pi}[i])^2 \Big)^{1/2} \\
&\le 2^{1/2} (\Sigma+\widehat\Sigma)^{1/2} \Big( \sum_{i=1}^k  \big| \pi[i] - \widehat{\pi}[i] \big| \Big)^{1/2} = O\Big( (\Sigma+\widehat\Sigma)^{1/2} \| \pi - \widehat\pi \|_2 \Big)\,.
\end{align*}
Thus
\begin{align*}
\Big| \frac1{\widehat\Sigma} - \frac1{\Sigma} \Big| 
&= 
O\Big( \frac{ (\Sigma+\widehat\Sigma)^{1/2} \| \pi - \widehat\pi \|_2 }{ \Sigma \widehat\Sigma } \Big)\\ 
&= O\Big( \frac{(1/\Sigma + 1/\widehat\Sigma)^{1/2}}{(\Sigma \widehat\Sigma)^{1/2}} \| \pi - \widehat\pi \|_2 \Big)\\
&= 
O\Big( \frac{k}{\Sigma^{1/2}}\| \pi - \widehat\pi \|_2 \Big)\,.
\end{align*}
Hence 
$$
\big|\expressionI\big| + \big|\expressionII\big| = pF_2 \cdot O\Big( \frac{\|\pi - \widehat \pi\|_1}{\Sigma} + \frac{k \|\pi - \widehat \pi\|_2}{\Sigma^{1/2}} \Big) = O\Big( \frac{k \|\pi - \widehat \pi\|_2}{\Sigma^{1/2}}\Big)
$$ 
in this case. Here in the last step we used the fact that $\frac{\|\pi - \widehat \pi\|_1}{\Sigma} \le \frac{\sqrt{k} \| \pi - \widehat \pi\|_2}{\Sigma} \le \frac{k \|\pi - \widehat \pi\|_2}{\Sigma^{1/2}}$, again by the Cauchy-Schwarz inequality.
\end{itemize}
This concludes the proof.
\end{proof}

\begin{remark}\label{rem:extension}
Suppose the priors differ for each attribution set $A_j$ (note this encompasses varying set sizes as a special case). 
Denoting the prior for attribution set $A_j$ as $\pi_j$, Theorem \ref{prop:unbiasedestimatoroneji} still applies; note the proof of Theorem \ref{prop:unbiasedestimatoroneji} only considered the $j$-th attribution set $A_j$. 
Theorem \ref{prop:unbiasedestimatoroneji} now establishes that $\widehat\ell(h, j, i)$ is unbiased, where $\beta_0(j, i), \beta_1(j, i)$ in the definition of $\widehat\ell(h, j, i)$ (see (\ref{eq:singlejiestimator})) are now defined in terms of $\pi_j$ rather than $\pi$.

Letting $k_j$ denote the size of the support of $\pi_j$ and $\Sigma_j = \sum_{i=1}^{k_j} \pi_j[i]^2$, we now define $\widehat\ell(h,j) = \frac1{\Sigma_j} \sum_{i=1}^{k_j} \pi_j[i]^2 \widehat\ell(h, j, i)$. We again define $\widehat\ell_M(h, S, \cA) = \frac1{\frac{np}2-2k+1}\sum_{j=k}^{\mupper} \widehat\ell(h,j)$ as in (\ref{eq:all_set_estimator}). Defining the ERM estimator $\widehat h$ in terms of $\widehat\ell_M(h, S, \cA)$ as in (\ref{eq:erm_estimator_def}), we again can study $\regret(\widehat h)$ as in Theorem \ref{thm:allattributionsetssamplecomplexity}, following a similar proof as the proof presented above.
\end{remark}

\section{Further details about the experiments}\label{sa:experiments}

\subsection{Datasets}

\paragraph{MNIST:}
The MNIST dataset \citep{lecun2010mnist} is a collection of $28\times28$ grayscale handwritten digits containing 60,000 training and 10,000 test examples. To adapt this for binary classification, we labeled digit ``1" as the positive class and all other digits as negative. Our model architecture consists of a multilayer perceptron (MLP) with three hidden layers of sizes 512, 512, and 128. We employed ReLU activation functions and a dropout rate of 0.2 at each layer. The network outputs a raw logit via a linear final layer.

\paragraph{CIFAR-10:}
The CIFAR-10 dataset \citep{Krizhevsky09Cifar} is a multi-class dataset of 50,000 training and 10,000 test images ($32 \times 32$ color). Each image belongs to one of ten classes: \textsc{Airplane, Automobile, Bird, Cat, Deer, Dog, Frog, Horse, Ship,} or \textsc{Truck}. For our experiments, we applied an Animal-vs-Machine binarization: Positive Class: \textsc{Bird, Cat, Deer, Dog, Frog,} and \textsc{Horse}; Negative Class: \textsc{Airplane, Automobile, Ship,} and \textsc{Truck}. The model is a Convolutional Neural Network (CNN) structured as follows: \begin{itemize}
\item Convolutional Layer: 32 filters with ReLU activation.
\item Max Pooling: $2 \times 2$ window and stride.
\item Convolutional Layer: 64 filters with ReLU activation.
\item Dropout Layer: 0.5 rate.
\item Fully Connected Layer: Single linear output producing a raw logit.
\end{itemize}

\paragraph{Higgs:}
The Higgs dataset \citep{Baldi2014Higgs} is a collection of simulated particle physics data used to distinguish between Higgs boson production processes and background noise. While the original dataset contains 11 million examples, we used a subset of 200,000 to accelerate experimentation, allocating the first 10,000 for test, and the subsequent 190,000 for training. Each example comprises 21 features, including both direct physical measurements and hand-crafted high-level features.
Our model class is a fully connected model with 4 hidden layers each having 300 neurons and ReLU activations followed by a fully connected layer with 1 output and no activation so that it outputs a logit.

\begin{figure*}
    \centering
\includegraphics[width=0.248\textwidth]{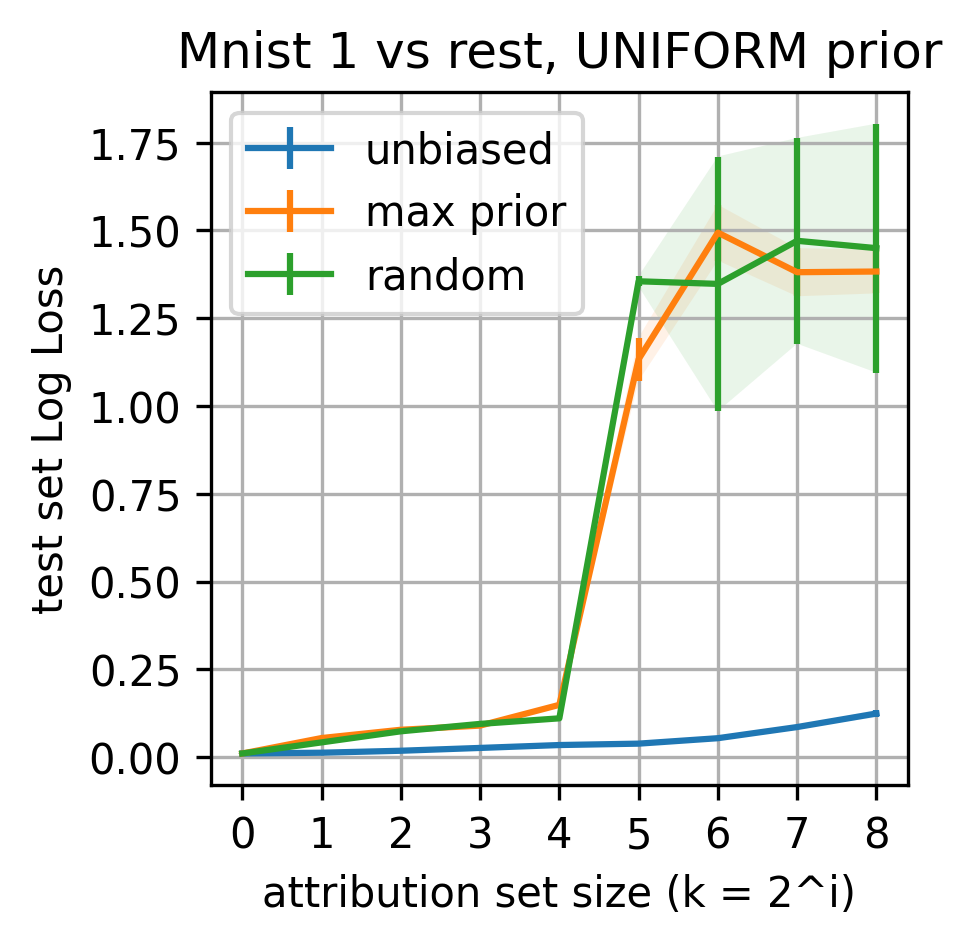}
\includegraphics[width=0.238\textwidth]{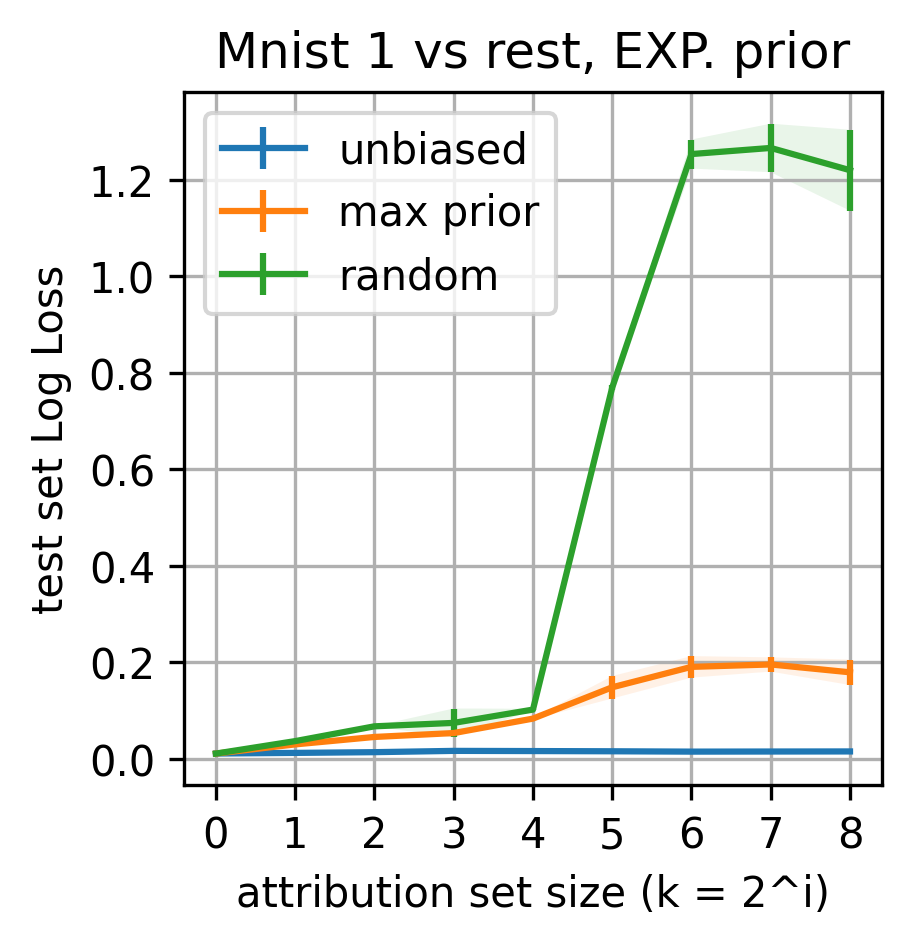}
\includegraphics[width=0.246\textwidth]{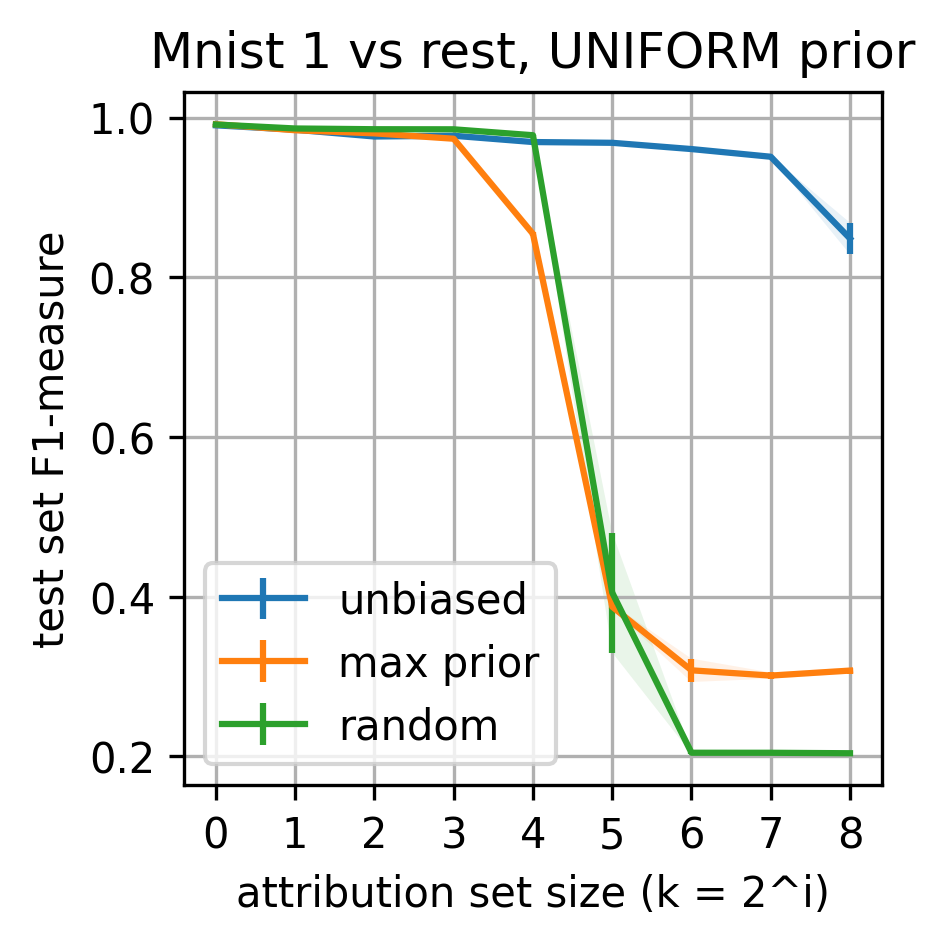}
\includegraphics[width=0.242\textwidth]{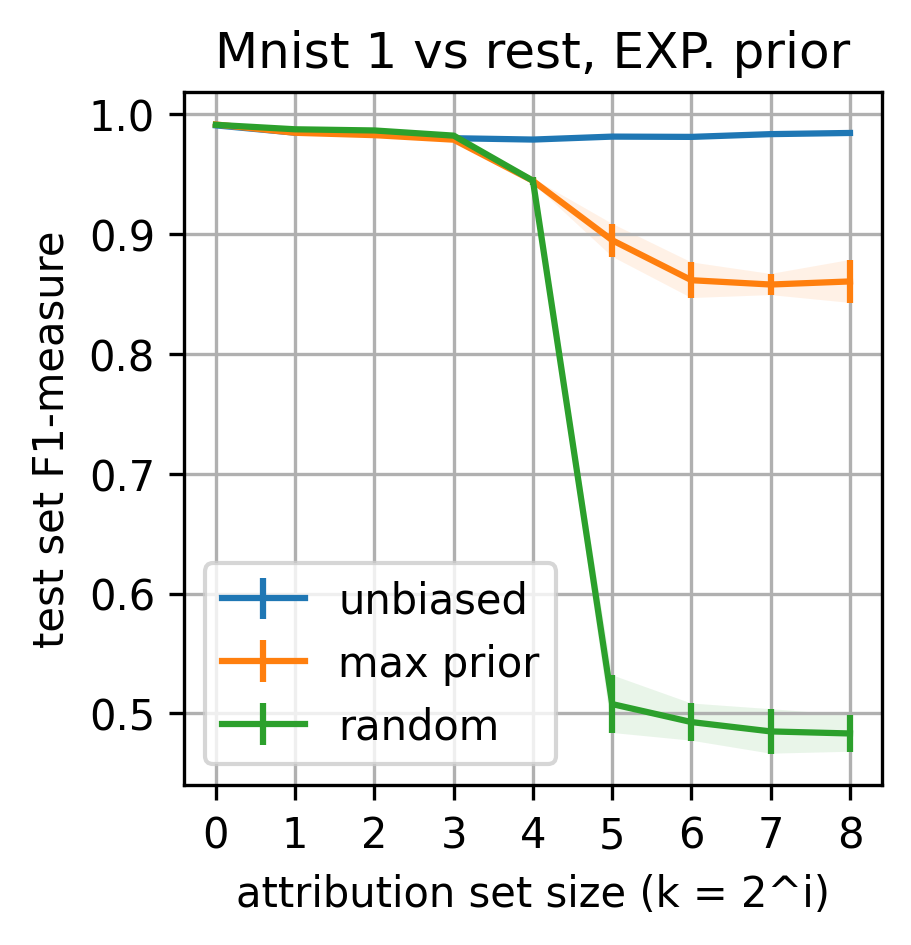}
\vspace{-0.1in}
    \caption{MNIST 1 vs. rest performance measured on the test set via log loss (first two plots from the left) or F1-measure (last two plots) on both the uniform prior and the exponential prior. Standard deviations are also depicted.}
    \label{f:4}
    \vspace{-0.2in}
\end{figure*}

\subsection{Further results}
Figure \ref{f:4} reports average test log loss and F1-measure of the three algorithms on the MNIST dataset.


\newpage


\begin{thebibliography}{43}
\providecommand{\natexlab}[1]{#1}
\providecommand{\url}[1]{\texttt{#1}}
\expandafter\ifx\csname urlstyle\endcsname\relax
  \providecommand{\doi}[1]{doi: #1}\else
  \providecommand{\doi}{doi: \begingroup \urlstyle{rm}\Url}\fi

\bibitem[Applebaum et~al.(2026)Applebaum, Dick, Gentile, Kaplan, and Koren]{applebaum2026optimallearninglabelproportions}
Lorne Applebaum, Travis Dick, Claudio Gentile, Haim Kaplan, and Tomer Koren.
\newblock Optimal learning from label proportions with general loss functions, 2026.
\newblock URL \url{https://arxiv.org/abs/2509.15145}.

\bibitem[Baldi et~al.(2014)Baldi, Sadowski, and Whiteson]{Baldi2014Higgs}
Pierre Baldi, Peter Sadowski, and Daniel Whiteson.
\newblock {Searching for Exotic Particles in High-Energy Physics with Deep Learning}.
\newblock \emph{Nature Commun.}, 5:\penalty0 4308, 2014.
\newblock \doi{10.1038/ncomms5308}.

\bibitem[Borgs et~al.(2007)Borgs, Chayes, Immorlica, Jain, Etesami, and Mahdian]{borgs2007dynamics}
Christian Borgs, Jennifer Chayes, Nicole Immorlica, Kamal Jain, Omid Etesami, and Mohammad Mahdian.
\newblock Dynamics of bid optimization in online advertisement auctions.
\newblock In \emph{Proceedings of the 16th international conference on World Wide Web}, pages 531--540, 2007.

\bibitem[Brahmbhatt et~al.(2023)Brahmbhatt, Saket, and Raghuveer]{brahmbhatt2023pac}
Anand Brahmbhatt, Rishi Saket, and Aravindan Raghuveer.
\newblock Pac learning linear thresholds from label proportions.
\newblock \emph{Advances in Neural Information Processing Systems}, 36:\penalty0 66610--66646, 2023.

\bibitem[Busa-Fekete et~al.(2023)Busa-Fekete, Choi, Dick, Gentile, and Medina]{easyllp}
Robert Busa-Fekete, Heejin Choi, Travis Dick, Claudio Gentile, and Andres~Munoz Medina.
\newblock Easy learning from label proportions.
\newblock In \emph{Proceedings of the 37th International Conference on Neural Information Processing Systems}, NeurIPS 2023. Curran Associates Inc., 2023.

\bibitem[Busa-Fekete et~al.(2025)Busa-Fekete, Dick, Gentile, Kaplan, Koren, and Stemmer]{b+25}
Robert Busa-Fekete, Travis Dick, Claudio Gentile, Haim Kaplan, Tomer Koren, and Uri Stemmer.
\newblock Nearly optimal sample complexity for learning with label proportions.
\newblock In \emph{ICML 2025}, 2025.

\bibitem[Cai et~al.(2017)Cai, Ren, Zhang, Malialis, Wang, Yu, and Guo]{cai2017real}
Han Cai, Kan Ren, Weinan Zhang, Kleanthis Malialis, Jun Wang, Yong Yu, and Defeng Guo.
\newblock Real-time bidding by reinforcement learning in display advertising.
\newblock In \emph{Proceedings of the tenth ACM international conference on web search and data mining}, pages 661--670, 2017.

\bibitem[Chen et~al.(2025)Chen, Chan, Sheng, Zhang, Chen, Hou, Zhu, Xu, and Zheng]{chen2025singleviewmultiattributionlearning}
Sishuo Chen, Zhangming Chan, Xiang-Rong Sheng, Lei Zhang, Sheng Chen, Chenghuan Hou, Han Zhu, Jian Xu, and Bo~Zheng.
\newblock See beyond a single view: Multi-attribution learning leads to better conversion rate prediction, 2025.
\newblock URL \url{https://arxiv.org/abs/2508.15217}.

\bibitem[Chen et~al.(2022)Chen, Jin, Zhao, Wang, Liu, Xu, and Zheng]{chen2022asymptoticallyunbiasedestimationdelayed}
Yu~Chen, Jiaqi Jin, Hui Zhao, Pengjie Wang, Guojun Liu, Jian Xu, and Bo~Zheng.
\newblock Asymptotically unbiased estimation for delayed feedback modeling via label correction.
\newblock In \emph{Proc. WWW 2022}, 2022.

\bibitem[Combes(2024)]{combes2024extension}
Richard Combes.
\newblock An extension of {M}cdiarmid's inequality.
\newblock In \emph{2024 IEEE International Symposium on Information Theory (ISIT)}, pages 79--84. IEEE, 2024.

\bibitem[Crouch and Crawford(2022)]{cc22}
Luke Crouch and Maxx Crawford.
\newblock Over a decade of anti-tracking work at mozilla.
\newblock \url{https://blog.mozilla.org/en/privacy-security/mozilla-antitracking-milestones-timeline/}, 2022.

\bibitem[Dietterich et~al.(1997)Dietterich, Lathrop, and Lozano-P{\'e}rez]{dietterich1997solving}
Thomas~G Dietterich, Richard~H Lathrop, and Tom{\'a}s Lozano-P{\'e}rez.
\newblock Solving the multiple instance problem with axis-parallel rectangles.
\newblock \emph{Artificial intelligence}, 89\penalty0 (1-2):\penalty0 31--71, 1997.

\bibitem[Fan et~al.(2025)Fan, Hu, Wang, Ma, Wang, Ye, and Shen]{fan2025two}
Zhikang Fan, Lan Hu, Ruirui Wang, Zhongrui Ma, Yue Wang, Qi~Ye, and Weiran Shen.
\newblock Two-stage auction design in online advertising.
\newblock In \emph{Proceedings of the ACM on Web Conference 2025}, pages 3571--3585, 2025.

\bibitem[Ilse et~al.(2018)Ilse, Tomczak, and Welling]{ilse2018attention}
Maximilian Ilse, Jakub Tomczak, and Max Welling.
\newblock Attention-based deep multiple instance learning.
\newblock In \emph{International conference on machine learning}, pages 2127--2136. PMLR, 2018.

\bibitem[Jang and Kwon(2024)]{jk24}
Jaeseok Jang and Hyuk-Yoon Kwon.
\newblock Are multiple instance learning algorithms learnable for instances?
\newblock In \emph{Proc. Neurips}, 2024.

\bibitem[Javed et~al.(2022)Javed, Juyal, Padigela, Taylor-Weiner, Yu, and Prakash]{javed2022additive}
Syed~Ashar Javed, Dinkar Juyal, Harshith Padigela, Amaro Taylor-Weiner, Limin Yu, and Aaditya Prakash.
\newblock Additive mil: Intrinsically interpretable multiple instance learning for pathology.
\newblock \emph{Advances in Neural Information Processing Systems}, 35:\penalty0 20689--20702, 2022.

\bibitem[Jin et~al.(2018)Jin, Song, Li, Gai, Wang, and Zhang]{jin2018real}
Junqi Jin, Chengru Song, Han Li, Kun Gai, Jun Wang, and Weinan Zhang.
\newblock Real-time bidding with multi-agent reinforcement learning in display advertising.
\newblock In \emph{Proceedings of the 27th ACM international conference on information and knowledge management}, pages 2193--2201, 2018.

\bibitem[Joulani et~al.(2013)Joulani, Gyorgy, and Szepesvari]{pmlr-v28-joulani13}
Pooria Joulani, Andras Gyorgy, and Csaba Szepesvari.
\newblock Online learning under delayed feedback.
\newblock In \emph{Proceedings of the 30th International Conference on Machine Learning}, volume~28 of \emph{Proceedings of Machine Learning Research}, pages 1453--1461. PMLR, 2013.

\bibitem[Kingma and Ba(2015)]{kingma2015adam}
Diederik~P Kingma and Jimmy Ba.
\newblock Adam: A method for stochastic optimization.
\newblock In \emph{International Conference on Learning Representations (ICLR)}, 2015.

\bibitem[Krizhevsky(2009)]{Krizhevsky09Cifar}
Alex Krizhevsky.
\newblock Learning multiple layers of features from tiny images.
\newblock Technical report, Univ. of Toronto, 2009.

\bibitem[Ktena et~al.(2019)Ktena, Tejani, Theis, Myana, Dilipkumar, Husz{\'a}r, Yoo, and Shi]{ktena2019addressing}
Sofia~Ira Ktena, Alykhan Tejani, Lucas Theis, Pranay~Kumar Myana, Deepak Dilipkumar, Ferenc Husz{\'a}r, Steven Yoo, and Wenzhe Shi.
\newblock Addressing delayed feedback for continuous training with neural networks in ctr prediction.
\newblock In \emph{Proceedings of the 13th ACM conference on recommender systems}, pages 187--195, 2019.

\bibitem[LeCun et~al.(2010)LeCun, Cortes, and Burges]{lecun2010mnist}
Yann LeCun, Corinna Cortes, and CJ~Burges.
\newblock Mnist handwritten digit database.
\newblock \emph{ATT Labs [Online]. Available: http://yann.lecun.com/exdb/mnist}, 2, 2010.

\bibitem[Ledoux and Talagrand(2011)]{lt11}
Michel Ledoux and Michel Talagrand.
\newblock \emph{Probability in Banach spaces. Classics in Mathematics. Isoperimetry and processes, Reprint of the 1991 edition}.
\newblock Springer-Verlag, Berlin, 2011, 2011.

\bibitem[Li et~al.(2024)Li, Chen, Javanmard, and Mirrokni]{l+24}
G.~Li, L.~Chen, A.~Javanmard, and V.~Mirrokni.
\newblock Optimistic rates for learning from label proportions.
\newblock In \emph{Proceedings of Machine Learning Research}, volume 247 of \emph{37th Annual Conference on Learning Theory}, pages 1--38, 2024.

\bibitem[Liu et~al.(2021)Liu, Yu, Zhang, Zheng, Rong, Lv, Huo, Wang, Chen, Xu, et~al.]{liu2021neural}
Xiangyu Liu, Chuan Yu, Zhilin Zhang, Zhenzhe Zheng, Yu~Rong, Hongtao Lv, Da~Huo, Yiqing Wang, Dagui Chen, Jian Xu, et~al.
\newblock Neural auction: End-to-end learning of auction mechanisms for e-commerce advertising.
\newblock In \emph{Proceedings of the 27th ACM SIGKDD Conference on Knowledge Discovery \& Data Mining}, pages 3354--3364, 2021.

\bibitem[Lv et~al.(2023)Lv, Yue, Sun, Luo, Cui, and Zhang]{lv2023unbiased}
Hui Lv, Zhongqi Yue, Qianru Sun, Bin Luo, Zhen Cui, and Hanwang Zhang.
\newblock Unbiased multiple instance learning for weakly supervised video anomaly detection.
\newblock In \emph{Proceedings of the IEEE/CVF conference on computer vision and pattern recognition}, pages 8022--8031, 2023.

\bibitem[Maron and Lozano-P{\'e}rez(1997)]{maron1997framework}
Oded Maron and Tom{\'a}s Lozano-P{\'e}rez.
\newblock A framework for multiple-instance learning.
\newblock In \emph{Advances in neural information processing systems}, volume~10, 1997.

\bibitem[Massart(2000)]{ma00}
Pascal Massart.
\newblock Some applications of concentration inequalities to statistics.
\newblock \emph{Annales de la Facult\'e des Sciences de Toulouse}, IX:\penalty0 245--303, 2000.

\bibitem[Mendelson and Vershynin(2003)]{mendelson2003entropy}
Shahar Mendelson and Roman Vershynin.
\newblock Entropy and the combinatorial dimension.
\newblock \emph{Inventiones Mathematicae}, 152\penalty0 (1):\penalty0 37--55, 2003.

\bibitem[Patrini et~al.(2014)Patrini, Nock, Rivera, and Caetano]{patrini2014almost}
Giorgio Patrini, Richard Nock, Paul Rivera, and Tiberio Caetano.
\newblock (almost) no label no cry.
\newblock \emph{Advances in Neural Information Processing Systems}, 27, 2014.

\bibitem[Quadrianto et~al.(2008)Quadrianto, Smola, Caetano, and Le]{quadrianto2008estimating}
Novi Quadrianto, Alex~J Smola, Tiberio~S Caetano, and Quoc~V Le.
\newblock Estimating labels from label proportions.
\newblock In \emph{Proceedings of the 25th International Conference on Machine learning}, pages 776--783, 2008.

\bibitem[Saket(2021)]{saket2021learnability}
Rishi Saket.
\newblock Learnability of linear thresholds from label proportions.
\newblock \emph{Advances in Neural Information Processing Systems}, 34:\penalty0 6555--6566, 2021.

\bibitem[Saket(2022)]{saket2022algorithms}
Rishi Saket.
\newblock Algorithms and hardness for learning linear thresholds from label proportions.
\newblock \emph{Advances in Neural Information Processing Systems}, 35:\penalty0 1267--1279, 2022.

\bibitem[Scott and Zhang(2020)]{scott2020learning}
Clayton Scott and Jianxin Zhang.
\newblock Learning from label proportions: A mutual contamination framework.
\newblock \emph{Advances in neural information processing systems}, 33:\penalty0 22256--22267, 2020.

\bibitem[Singal et~al.(2019)Singal, Besbes, Desir, Goyal, and Iyengar]{singal2019shapley}
Raghav Singal, Omar Besbes, Antoine Desir, Vineet Goyal, and Garud Iyengar.
\newblock Shapley meets uniform: An axiomatic framework for attribution in online advertising.
\newblock In \emph{The world wide web conference}, pages 1713--1723, 2019.

\bibitem[Tian et~al.(2021)Tian, Pang, Chen, Singh, Verjans, and Carneiro]{tian2021weakly}
Yu~Tian, Guansong Pang, Yuanhong Chen, Rajvinder Singh, Johan~W Verjans, and Gustavo Carneiro.
\newblock Weakly-supervised video anomaly detection with robust temporal feature magnitude learning.
\newblock In \emph{Proceedings of the IEEE/CVF international conference on computer vision}, pages 4975--4986, 2021.

\bibitem[Vernade et~al.(2017)Vernade, Capp\'e, and Perchet]{vernade2017stochasticbanditmodelsdelayed}
Claire Vernade, Olivier Capp\'e, and Vianney Perchet.
\newblock Stochastic bandit models for delayed conversions.
\newblock In \emph{Proc. UAI, 2017}, 2017.

\bibitem[Vernade et~al.(2020)Vernade, Carpentier, Lattimore, Zappella, Ermis, and Brueckner]{lbsdf20}
Claire Vernade, Alexandra Carpentier, Tor Lattimore, Giovanni Zappella, Beyza Ermis, and Michael Brueckner.
\newblock Linear bandits with stochastic delayed feedback.
\newblock In \emph{Proc. ICML, 2020}, 2020.

\bibitem[Wang et~al.(2017)Wang, Zhang, Yuan, et~al.]{wang2017display}
Jun Wang, Weinan Zhang, Shuai Yuan, et~al.
\newblock Display advertising with real-time bidding (rtb) and behavioural targeting.
\newblock \emph{Foundations and Trends{\textregistered} in Information Retrieval}, 11\penalty0 (4-5):\penalty0 297--435, 2017.

\bibitem[Wilander(2019)]{w19}
John Wilander.
\newblock Intelligent tracking prevention 2.3. apple.
\newblock \url{https://webkit.org/blog/9521/intelligent-tracking-prevention-2-3/}, 2019.

\bibitem[Yang et~al.(2019)Yang, Li, Wang, Wu, Tan, Xu, and Gai]{yang2019bid}
Xun Yang, Yasong Li, Hao Wang, Di~Wu, Qing Tan, Jian Xu, and Kun Gai.
\newblock Bid optimization by multivariable control in display advertising.
\newblock In \emph{Proceedings of the 25th ACM SIGKDD international conference on knowledge discovery \& data mining}, pages 1966--1974, 2019.

\bibitem[Zhang et~al.(2022)Zhang, Wang, and Scott]{zhang2022learning}
Jianxin Zhang, Yutong Wang, and Clay Scott.
\newblock Learning from label proportions by learning with label noise.
\newblock \emph{Advances in Neural Information Processing Systems}, 35:\penalty0 26933--26942, 2022.

\bibitem[Zhu et~al.(2017)Zhu, Jin, Tan, Pan, Zeng, Li, and Gai]{zhu2017optimized}
Han Zhu, Junqi Jin, Chang Tan, Fei Pan, Yifan Zeng, Han Li, and Kun Gai.
\newblock Optimized cost per click in taobao display advertising.
\newblock In \emph{Proceedings of the 23rd ACM SIGKDD international conference on knowledge discovery and data mining}, pages 2191--2200, 2017.

\end{thebibliography}
\end{document}